\documentclass{article}

\usepackage[left=1.25in,top=1.25in,right=1.25in,bottom=1.25in,head=1.25in]{geometry}
\usepackage{amsfonts,amsmath,amssymb,amsthm}
\usepackage{verbatim,float,url,bbm,enumitem}
\usepackage{graphicx,psfrag,color}
\usepackage{microtype} 
\usepackage{natbib}
\usepackage{algorithm}
\usepackage{scrextend}
\usepackage[colorlinks,linkcolor=blue,citecolor=blue,urlcolor=blue]{hyperref}

\newtheorem{theorem}{Theorem}
\newtheorem{lemma}{Lemma}
\newtheorem{corollary}{Corollary}
\newtheorem{remark}{Remark}

\newtheorem*{assumption*}{\assumptionnumber}
\providecommand{\assumptionnumber}{}
\makeatletter
\newenvironment{assumption}[3]{
  \renewcommand{\assumptionnumber}{Assumption #1#2 (#3)}
  \begin{assumption*}
  \protected@edef\@currentlabel{#1#2}}
{\end{assumption*}}
\makeatother

\makeatletter
\newcommand*\rel@kern[1]{\kern#1\dimexpr\macc@kerna}
\newcommand*\widebar[1]{%
  \begingroup
  \def\mathaccent##1##2{%
    \rel@kern{0.8}%
    \overline{\rel@kern{-0.8}\macc@nucleus\rel@kern{0.2}}%
    \rel@kern{-0.2}%
  }%
  \macc@depth\@ne
  \let\math@bgroup\@empty \let\math@egroup\macc@set@skewchar
  \mathsurround\z@ \frozen@everymath{\mathgroup\macc@group\relax}%
  \macc@set@skewchar\relax
  \let\mathaccentV\macc@nested@a
  \macc@nested@a\relax111{#1}%
  \endgroup
}
\makeatother

\newcommand*\ccol[1]{\omit\hfil$\displaystyle#1$\hfil\ignorespaces}
\newcommand{\argmin}{\mathop{\mathrm{argmin}}}

\newcommand{\st}{\mathop{\mathrm{subject\,\,to}}}
\newcommand{\esssup}{\mathop{\mathrm{ess\,\,sup}}}
\newcommand{\essinf}{\mathop{\mathrm{ess\,\,inf}}}

\def\E{\mathbb{E}}
\def\P{\mathbb{P}}

\def\Cov{\mathrm{Cov}}
\def\half{\frac{1}{2}}

\def\tr{\mathrm{tr}}
\def\df{\mathrm{df}}
\def\dim{\mathrm{dim}}

\def\nul{\mathrm{null}}

\def\supp{\mathrm{supp}}
\def\diag{\mathrm{diag}}
\def\spa{\mathrm{span}}

\def\hu{\hat{u}}

\def\htheta{\hat{\theta}}

\def\R{\mathbb{R}}

\def\cE{\mathcal{E}}
\def\cF{\mathcal{F}}
\def\cG{\mathcal{G}}
\def\cH{\mathcal{H}}

\def\cM{\mathcal{M}}

\def\cP{\mathcal{P}}
\def\cR{\mathcal{R}}
\def\cS{\mathcal{S}}

\def\cW{\mathcal{W}}

\def\thetaj{\theta_j}
\def\hthetaj{\htheta_j} 
\def\Xj{X_j}
\def\Dj{D^{(\Xj,k+1)}}
\def\one{\mathbbm{1}}
\def\TV{\mathrm{TV}}
\def\hf{\hat{f}}
\def\tf{\tilde{f}}
\def\hDelta{\hat\Delta}
\def\tDelta{\tilde\Delta}
\def\bDelta{\bar\Delta}
\def\hJ{\hat{J}}
\def\tJ{\tilde{J}}
\def\sJ{J^*}

\def\Cov{\mathrm{Cov}}

\def\KL{\mathrm{KL}}
\def\Id{\mathrm{Id}}
\def\TF{\mathrm{TF}}
\def\isim{\overset{\mathrm{i.i.d.}}{\sim}}

\title{Additive Models with Trend Filtering}
\author{Veeranjaneyulu Sadhanala \and Ryan J. Tibshirani}
\date{}

\begin{document}
\maketitle

\begin{abstract}
We study additive models built with trend filtering, i.e., additive
models whose components are each regularized by the (discrete) total
variation of their $k$th (discrete) derivative, for a chosen
integer $k \geq 0$.  This results in $k$th degree piecewise polynomial
components, (e.g., $k=0$ gives piecewise constant components, $k=1$
gives piecewise linear, $k=2$ gives piecewise quadratic, etc.).
Analogous to its advantages in the univariate case, additive trend
filtering has favorable theoretical and computational properties,
thanks in large part to the localized nature of the (discrete) total
variation regularizer that it uses.  On the theory side, we derive
fast error rates for additive trend filtering estimates, and show
these rates are minimax optimal when the underlying function is
additive and has component functions whose derivatives are of bounded
variation.  We also show that these rates are unattainable by additive
smoothing splines (and by additive models built from linear smoothers,
in general).  On the computational side, as per the standard in
additive models, backfitting is an appealing method for optimization,
but it is particularly appealing for additive trend filtering because
we can leverage a few highly efficient univariate trend filtering
solvers.  Going one step further, we describe a new backfitting
algorithm whose iterations can be run in parallel, which (as far as we
know) is the first of its kind.  Lastly, we present experiments
to examine the empirical performance of additive trend filtering. 
\end{abstract}

\section{Introduction}
\label{sec:intro}

As the dimension of the input space grows large, nonparametric regression
turns into a notoriously difficult problem. In this work, we adopt the stance
taken by many others, and consider an {\it additive model} for responses $Y^i
\in \R$, $i=1,\ldots,n$ and corresponding input points $X^i=(X^i_1,\ldots,X^i_d) 
\in \R^d$, $i=1,\ldots,n$, of the form 
$$
Y^i = \mu + \sum_{j=1}^d f_{0j}(X^i_j) + \epsilon^i, \;\;\;
i=1,\ldots,n,
$$
where $\mu \in \R$ is an overall mean parameter, each
$f_{0j}$ is a univariate function with
\smash{$\sum_{i=1}^n f_{0j}(X^i_j)=0$} for
identifiability, $j=1,\ldots,d$, 
and the errors $\epsilon^i$, $i=1,\ldots,n$ are
i.i.d.\ with mean zero. 
A comment on notation: here and throughout, when indexing over the    
$n$ samples we use superscripts, and when indexing over the $d$
dimensions we use subscripts, so that, e.g., \smash{$X^i_j$}
denotes the $j$th component of the $i$th input point.  
(Exceptions will occasionally be made, but the role of the
index should be clear from the context.)  

Additive models are a special case of the more general 
{\it projection pursuit regression} model of \citet{friedman1981}.
Additive models for the Cox regression and logistic regression
settings were studied in \citet{tibshirani1983} and
\citet{hastie1983}, respectively.  Some of the first asymptotic theory
for additive models was developed in
\citet{stone1985additive}. Two algorithms closely related to
(backfitting for) additive models are the {\it alternating least
  squares} and {\it alternating conditional expectations} methods,
from \citet{vanderburg1983} and \citet{breiman1985}, respectively. The
work of \citet{buja1989linear} advocates for the use of additive models in  
combination with linear smoothers, a surprisingly simple combination
that gives rise to flexible and scalable multidimensional regression 
tools. The book by \citet{hastie1990generalized} is the definitive 
practical guide for additive models for exponential family data 
distributions, i.e., generalized additive models.  

More recent work on additive models is focused on high-dimensional  
nonparametric estimation, and here the natural goal is to induce
sparsity in the component functions, so that only a few select
dimensions of the input space are used in the fitted additive model.
Some nice contributions are given in
\citet{lin2006component,ravikumar2009sparse,meier2009high}, all  
primarily focused on fitting splines for component functions and
achieving sparsity through a group lasso type penalty.  In other even
more recent and interesting work sparse additive models,
\citet{lou2016sparse} consider a semiparametric (partially linear)
additive model, and \citet{petersen2016fused} consider a formulation
that uses fused lasso (i.e., total variation) penalization applied to
the component functions.  

The literature on additive models (and by now, sparse additive models)
is vast and the above is far form a complete list of references.
In this paper, we examine a method for estimating additive models
wherein each component is fit in a way that is {\it locally adaptive}
to the underlying smoothness along its associated dimension of the
input space.  The literature on this line of work, as far as we can
tell, is much less extensive. First, we review linear smoothers in 
additive models, motivate our general goal of local adaptivity, and
then describe our specific proposal. 

\subsection{Review: additive models and linear smoothers} 

The influential paper by \citet{buja1989linear} studies additive 
minimization problems of the form
\begin{equation}
\label{eq:add_quad}
\begin{alignedat}{2}
&\ccol{\min_{\theta_1,\ldots,\theta_d \in \R^n}} \quad
&&\bigg\|Y - \bar{Y}\one - \sum_{j=1}^d \thetaj \bigg\|_2^2 +      
\lambda \sum_{j=1}^d \thetaj^T Q_j \thetaj \\
&\ccol{\st}\quad &&\one^T \thetaj =0, \;\;\; j=1,\ldots,d,    
\end{alignedat}
\end{equation}
where $Y=(Y^1,\ldots,Y^n) \in \R^n$ denotes the vector of responses, 
and \smash{$Y-\bar{Y}\one$} is its centered version, with
\smash{$\bar{Y}=\frac{1}{n}\sum_{i=1}^n Y^i$} denoting the sample mean 
of $Y$, and $\one=(1,\ldots,1) \in \R^n$ the vector of all 1s. 
Each vector $\thetaj=(\theta^1_j,\ldots,\theta^n_j) \in \R^n$
represents the evaluations of the $j$th component function $f_j$ in  
our model, i.e., tied together by the relationship
$$
\theta^i_j = f_j(X^i_j), \;\;\; i=1,\ldots,n, \; j=1,\ldots,d.  
$$
In the problem \eqref{eq:add_quad}, $\lambda \geq 0$ is a
regularization parameter and $Q_j$, $j=1,\ldots,d$ are  
penalty matrices.  As a typical example, we might consider $Q_j$ to be
the Reinsch penalty matrix for smoothing splines along the $j$th 
dimension of the input space, for $j=1,\ldots,d$.  Under this choice,
a backfitting (block coordinate descent) routine for
\eqref{eq:add_quad} would repeatedly cycle through the updates 
\begin{equation}
\label{eq:backfit_quad}
\thetaj = (I+\lambda Q_j)^{-1}\bigg(
Y - \bar{Y}\one -\sum_{\ell\not=j} \theta_\ell \bigg),
\;\;\; j=1,\ldots,d, 
\end{equation}
where the $j$th update fits a smoothing spline to the $j$th partial 
residual,
over the $j$th dimension of the input points, denoted by $\Xj
=(X^1_j,X^2_j,\ldots X^n_j) \in \R^n$.  At convergence, we arrive at an additive
smoothing spline estimate, which solves \eqref{eq:add_quad}.   

Modeling the component functions as smoothing splines is arguably the 
most common formulation for additive models, and it is the standard
in several statistical software packages like the R package {\tt gam}.
As \citet{buja1989linear} explain, the backfitting steps in 
\eqref{eq:backfit_quad} suggest that a more algorithmic approach to
additive modeling can be taken.  Instead of starting with a particular
criterion in mind, as in \eqref{eq:backfit_quad}, one can instead
envision repeatedly cycling through updates 
\begin{equation}
\label{eq:backfit_smooth}
\thetaj = S_j \bigg(Y-\bar{Y}\one -\sum_{\ell\not=j} \theta_\ell 
\bigg), \;\;\; j=1,\ldots,d, 
\end{equation}
where each $S_j$ is a particular (user-chosen) {\it linear 
  smoother}, meaning, a linear map that performs a univariate smoothing 
across the $j$th dimension of inputs $\Xj$.  The linear 
smoothers $S_j$, $j=1,\ldots,d$ could correspond to, e.g., smoothing 
splines, regression splines (regression using a spline basis with given
knots), kernel smoothing, local polynomial smoothing, or a
combination of these, across the input dimensions.  The convergence
point of the iterations \eqref{eq:backfit_smooth} solves a problem
of the form \eqref{eq:add_quad} with $\lambda Q_j=S_j^+-I$, where
$S_j^+$ is the Moore-Penrose pseudoinverse of $S_j$, for
$j=1,\ldots,d$.  

The class of linear smoothers is broad enough to offer fairly flexible, 
interesting mechanisms for smoothing, and simple enough to understand 
precisely. 
\citet{buja1989linear} provide a unified analysis of additive models with 
linear smoothers, in which they derive the effective degrees of freedom 
of these estimators and a generalized cross-validation 
routine for tuning; they also study fundamental properties such as 
uniqueness of the component fits, and convergence of the backfitting
steps.   

Much of the work following \citet{buja1989linear} remains in keeping with 
the idea of using linear smoothers in combination with additive
models. Studying high-dimensional additive models,
\citet{lin2006component,ravikumar2009sparse,meier2009high,
koltchinskii2010sparsity,raskutti2012minimax}
all essentially build their methods off of linear smoothers, with
modifications to induce sparsity in the estimated component functions.      
\citet{ravikumar2009sparse} consider a sparsified version of backfitting
in \eqref{eq:backfit_smooth}, while the others consider penalized versions of
the additive criterion in \eqref{eq:add_quad}.  

\subsection{The limitations of linear smoothers}

The beauty of linear smoothers lies in their simplicity.  However,
with this simplicity comes serious limitations, in terms of their
ability to adapt to varying local levels of smoothness.  In the
univariate setting, the seminal theoretical work by
\citet{donoho1998minimax} makes this idea precise.  With $d=1$,
suppose that underlying regression function $f_0$ lies in the univariate
function class 
\begin{equation}
\label{eq:tv_space}
\cF_k(C) = \{f : \TV(f^{(k)}) \leq C \},   
\end{equation}
for a constant $C>0$, where $\TV(\cdot)$ is the total 
variation operator, and \smash{$f^{(k)}$} the $k$th weak 
derivative of $f$. The class in \eqref{eq:tv_space} allows for greater  
fluctuation in the local level of smoothness of $f_0$ than, say, more  
typical function classes like Holder and Sobolev spaces. The
results of \citet{donoho1998minimax} (see also Section 5.1 of 
\citet{tibshirani2014adaptive}) imply that the minimax  
error rate for estimation over $\cF_k(C)$ is
\smash{$n^{-(2k+2)/(2k+3)}$}, but the minimax error rate when we 
consider only linear smoothers (linear transformations of  
$Y$) is \smash{$n^{-(2k+1)/(2k+2)}$}. This difference is highly
nontrivial, e.g., for $k=0$ this
is a difference of $n^{-2/3}$ (optimal) versus $n^{-1/2}$ (optimal 
among linear smoothers) for estimating a function $f_0$ of bounded 
variation. 

It is important to emphasize that this shortcoming is not just a
theoretical one; it is also clearly noticeable in basic practical
examples.
This does not bode well for additive models built from linear
smoothers, when estimating component functions $f_{0j}$,
$j=1,\ldots,d$ that display locally heterogeneous degrees of
smoothness.  Just as linear smoothers will struggle in the univariate 
setting, an additive estimate based on linear smoothers will not be 
able to efficiently track local changes in smoothness, across any of
the input dimensions. This could lead to a loss in accuracy even if  
only some (or one) of the components $f_{0j}$,
$j=1,\ldots,d$ possesses heterogeneous smoothness
across its domain. 

Two well-studied univariate estimators that are locally adaptive, i.e.,
that attain the minimax error rate over the $k$th order total
variation class in \eqref{eq:tv_space}, are wavelet smoothing
and locally adaptive regression splines, as developed by 
\citet{donoho1998minimax} and \citet{mammen1997locally}, respectively.
There is a substantial literature on these methods in the univariate
case (especially for wavelets), but fewer authors have
considered using these locally adaptive estimators in the additive
models context.  Some notable exceptions are
\citet{zhang2003wavelet,sardy2004amlet}, who study additive models built from
wavelets, and \citet{petersen2016fused}, who study sparse additive models with
components given by 0th order locally adaptive regression splines (i.e., the
components are regularized via fused lasso penalties or total variation
penalties).  The latter work is especially related to our focus in this paper. 

\subsection{Additive trend filtering}

We consider additive models that are constructed using {\it trend filtering}  
(instead of linear smoothers, wavelets, or locally adaptive regression splines)
as their componentwise smoother.  Proposed independently by
\citet{steidl2006splines} and \citet{kim2009trend}, trend filtering is a
relatively new approach to univariate nonparametric regression.  As explained in  
\citet{tibshirani2014adaptive}, it can be seen as a discrete-time
analog of the locally adaptive regression spline estimator.  Denoting by
$X=(X^1,\ldots,X^n) \in \R^n$ the vector of univariate input points, where we
assume $X^1 < \ldots < X^n$, the trend filtering estimate of order $k \geq 0$
is defined as the solution of the optimization problem    
\begin{equation}
\label{eq:tf}
\min_{\theta \in \R^n} \; \half \|Y-\theta\|_2^2 + 
\lambda \|D^{(X,k+1)} \theta\|_1,
\end{equation}
where $\lambda \geq 0$ is a tuning parameter, and $D^{(X,k+1)} \in
\R^{(n-k-1)\times n}$ is a $k$th order difference operator, constructed based on
$X$. These difference operators can be defined recursively, as in      
\begin{align}
\label{eq:d1}
D^{(X,1)} &= \left[\begin{array}{rrrrrr}
-1 & 1 & 0 & \ldots & 0 & 0 \\
0 & -1 & 1 & \ldots & 0 & 0 \\
\vdots & & & & & \\
0 & 0 & 0 & \ldots & -1 & 1
\end{array}\right] \in \R^{(n-1) \times n}, \\
\label{eq:dk}
D^{(X,k+1)} &= D^{(X,1)} \cdot
\diag\bigg( \frac{k}{X^k-X^1}, \ldots,
\frac{k}{X^n-X^{n-k+1}} \bigg) \cdot D^{(X,k)} 
\in \R^{(n-k-1)\times n}, \;\;\; k=1,2,3,\ldots.  
\end{align}
(The leading matrix $D^{(X,1)}$ in \eqref{eq:dk} is the
$(n-k-1)\times (n-k)$ version of the difference operator in 
\eqref{eq:d1}.) 
Intuitively, the interpretation is that the problem \eqref{eq:tf}
penalizes the sum of absolute $(k+1)$st order discrete
derivatives of $\theta^1,\ldots,\theta^n$ across the input points  
$X^1,\ldots,X^n$.  Thus, at optimality, the coordinates of the trend filtering   
solution \smash{$\htheta^1,\ldots,\htheta^n$} obey a $k$th order piecewise
polynomial form. 

This intuition is formalized in \citet{tibshirani2014adaptive} and 
\citet{wang2014falling}, where it is shown that the components of 
the $k$th order trend filtering estimate \smash{$\htheta$} are
precisely the evaluations of a fitted $k$th order piecewise polynomial
function across the inputs, and that the trend filtering and locally
adaptive regression spline estimates of the same order $k$ are
asymptotically 
equivalent.  When $k=0$ or $k=1$, in fact, there is no need for
asymptotics, and the equivalence between trend filtering and locally 
adaptive regression spline estimates is exact in finite
samples. It is also worth pointing out that when $k=0$, the trend   
filtering estimate reduces to the 1d fused lasso estimate
\citep{tibshirani2005sparsity}, which is known as 1d total variation
denoising in signal processing \citep{rudin1992nonlinear}. 

Over the $k$th order total variation function class defined in
\eqref{eq:tv_space}, \citet{tibshirani2014adaptive,wang2014falling} 
prove that $k$th order trend filtering achieves 
the minimax optimal \smash{$n^{-(2k+2)/(2k+3)}$} error rate, just like $k$th  
order locally adaptive regression splines. Another important property,  
as developed by
\citet{kim2009trend,tibshirani2014adaptive,ramdas2016fast},     
is that trend filtering estimates are relatively cheap to 
compute---much cheaper than locally adaptive regression 
spline estimates---owing to the bandedness of the difference 
operators in \eqref{eq:d1}, \eqref{eq:dk}, which means that specially
implemented convex programming routines can solve \eqref{eq:tf} in an
efficient manner. 

It is this computational efficiency, along with its capacity for 
local adaptivity, that makes trend filtering a particularly desirable
candidate to extend to the additive model setting.  Specifically, we  
consider the {\it additive trend filtering} estimate of order $k \geq 0$,  
defined as a solution in the problem
\begin{equation}
\label{eq:add_tf} 
\begin{alignedat}{2}
&\ccol{\min_{\theta_1,\ldots,\theta_d \in \R^n}} \quad   
&&\half \bigg\|Y-\bar{Y}\one - \sum_{j=1}^d \thetaj \bigg\|_2^2 +     
\lambda \sum_{j=1}^d \big\|\Dj S_j \thetaj \big\|_1 \\ 
&\ccol{\st}\quad &&\one^T \thetaj =0, \;\;\; j=1,\ldots,d.  
\end{alignedat}
\end{equation}
As before, \smash{$Y-\bar{Y}\one$} is the centered response vector,
and $\lambda \geq 0$ is a regularization parameter. Not to be confused   
with the notation for linear smoothers from a previous subsection,  
$S_j \in \R^{n\times n}$ in \eqref{eq:add_tf} is a permutation matrix 
that sorts the $j$th component of inputs $\Xj=(X^1_j,X^2_j,\ldots 
X^n_j)$ into increasing order, i.e.,
$$
S_j\Xj = (X^{(1)}_j,  X^{(2)}_j, \ldots, X^{(n)}_j), \;\;\; j=1,\ldots,d.
$$
Also, \smash{$\Dj$} in \eqref{eq:add_tf} is the $(k+1)$st order difference
operator, as in \eqref{eq:d1}, \eqref{eq:dk}, but defined over the sorted $j$th 
dimension of inputs $S_j\Xj$, for $j=1,\ldots,d$.  With backfitting (block
coordinate descent), computation of a solution in \eqref{eq:add_tf} is still  
quite efficient, since we can leverage the efficient routines for
univariate trend filtering.

\subsection{A motivating example}
\label{sec:motivating}

Figure \ref{fig:intro} shows a simulated example that compares
the additive trend filtering estimates in \eqref{eq:add_tf} (of quadratic order,
$k=2$), to the additive smoothing spline estimates in \eqref{eq:add_quad} (of
cubic order). In the simulation, we used $n=3000$ and $d=3$.  We drew input
points \smash{$X^i \isim \mathrm{Unif}[0,1]^3$}, $i=1,\ldots,3000$, and 
drew responses \smash{$Y^i \isim N(\sum_{j=1}^3 f_{0j}(X^i_j),\sigma^2)$},  
$i=1,\ldots,3000$, where $\sigma=1.72$ was set to give a signal-to-noise 
ratio of about 1.  The underlying component functions were defined as 
$$
f_{01}(t) = \min(t, 1-t)^{0.2}
\sin \bigg(\frac{2.85 \pi}{0.3 + \min(t,1-t)}\bigg), \;\;\;
f_{02}(t) = e^{3t}  \sin(4\pi t), \;\;\;
f_{03}(t) = - (t-1/2)^2,
$$
so that $f_{01},f_{02},f_{03}$ possess different levels of smoothness ($f_{03}$
being the smoothest, $f_{02}$ less smooth, and $f_{01}$ the least smooth), and
so that $f_{01}$ itself has heteregeneous smoothness across its domain. 

\begin{figure}[p]
\centering
\includegraphics[width=\textwidth]{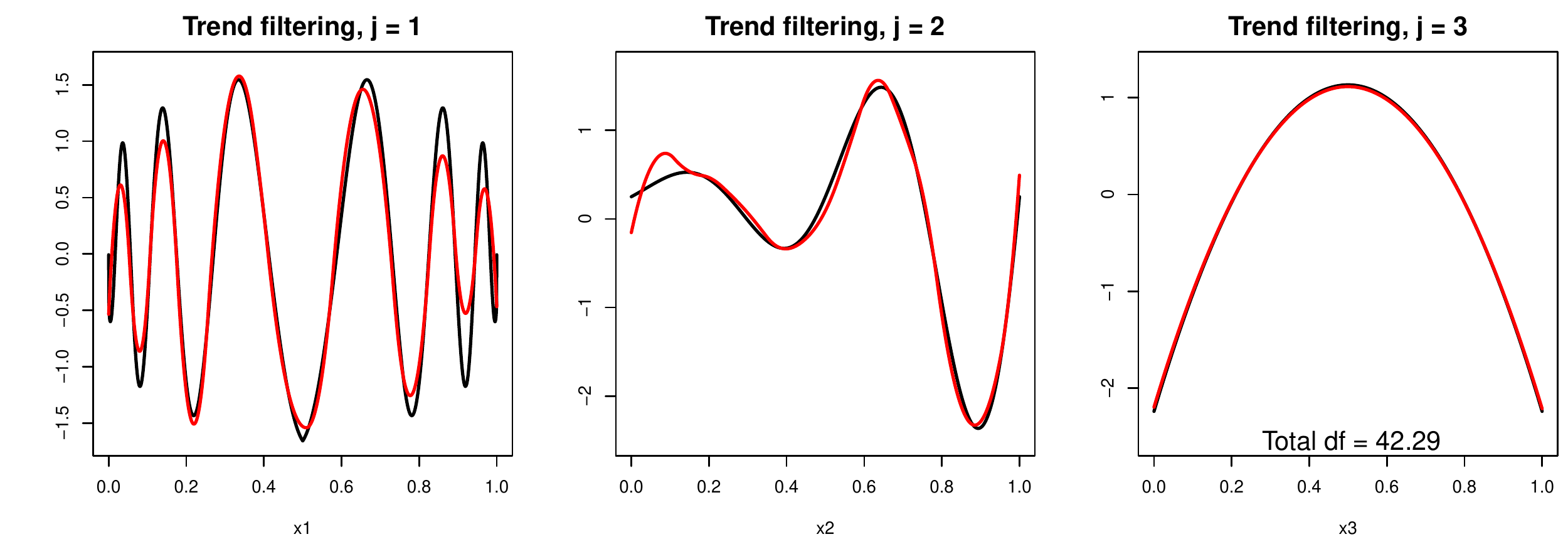} \\
\includegraphics[width=\textwidth]{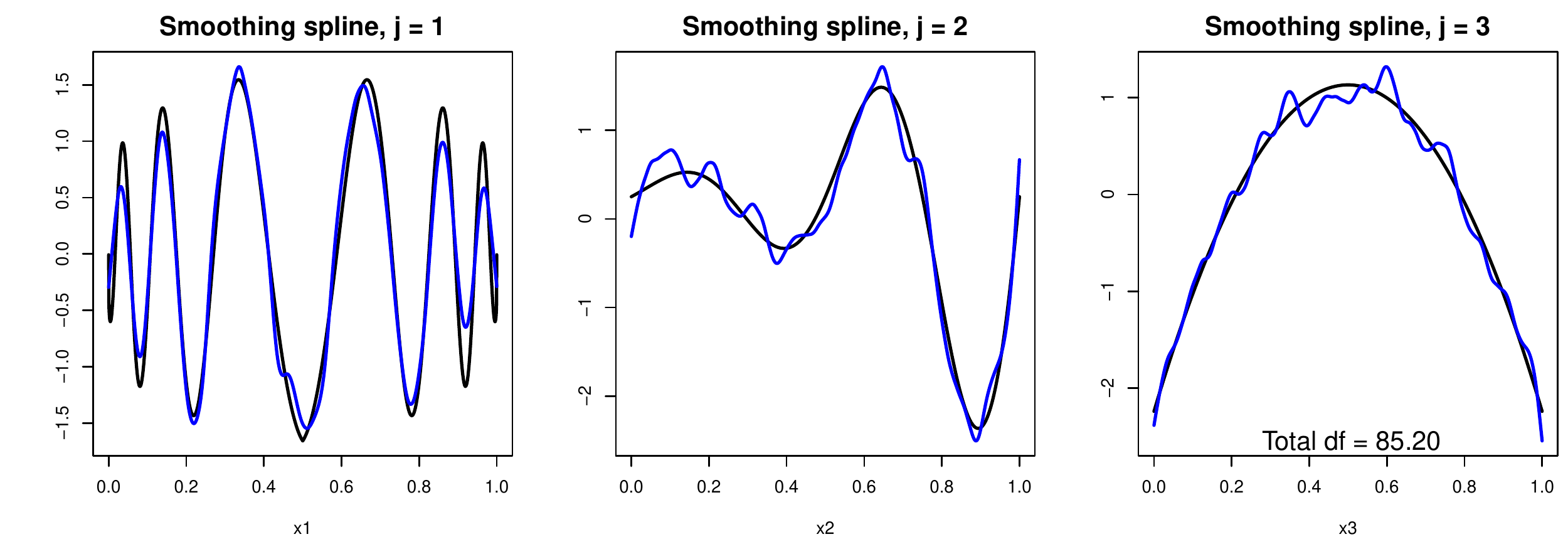} \\
\includegraphics[width=\textwidth]{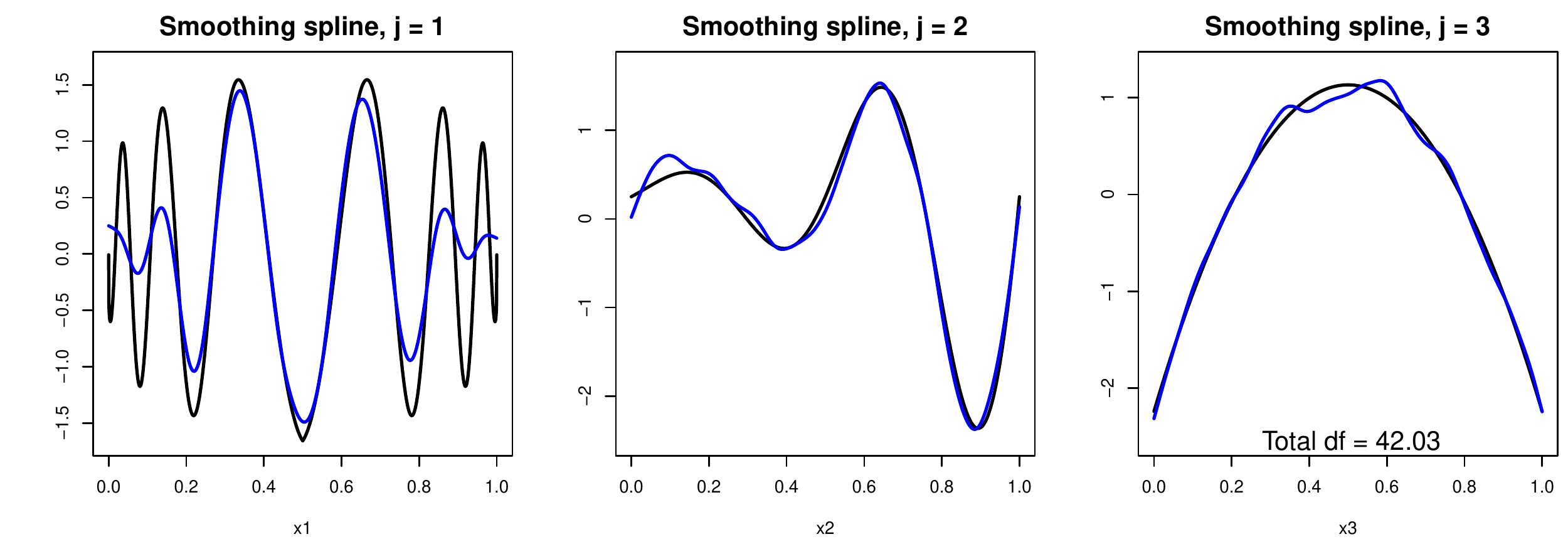} 
\caption{\it\small Comparing estimates from additive trend filtering
  \eqref{eq:add_tf} (of quadratic order) and additive smoothing splines
  \eqref{eq:add_quad} (of cubic order), for a simulation with $n=3000$
  and $d=2$, as described in Section \ref{sec:motivating}.  In each row, the
  underlying component functions are plotted in black.  The first row shows
  the estimated component functions using additive trend filtering, in 
red, at a value of $\lambda$ chosen to minimize mean squared error (MSE),
computed over 20 repetitions.  The second row shows the estimates from
additive smoothing splines, in blue, again at a value of $\lambda$ that
minimizes MSE. The third row shows the estimates from additive smoothing splines
when $\lambda$ is tuned so that the effective degrees of freedom (df) of the fit
roughly matches that of additive trend filtering in the first row.} 
\label{fig:intro}
\end{figure}

\begin{figure}[tb]
\hspace{-50pt}
\begin{minipage}[c]{0.675\textwidth}
\centering
\includegraphics[width=0.675\textwidth]{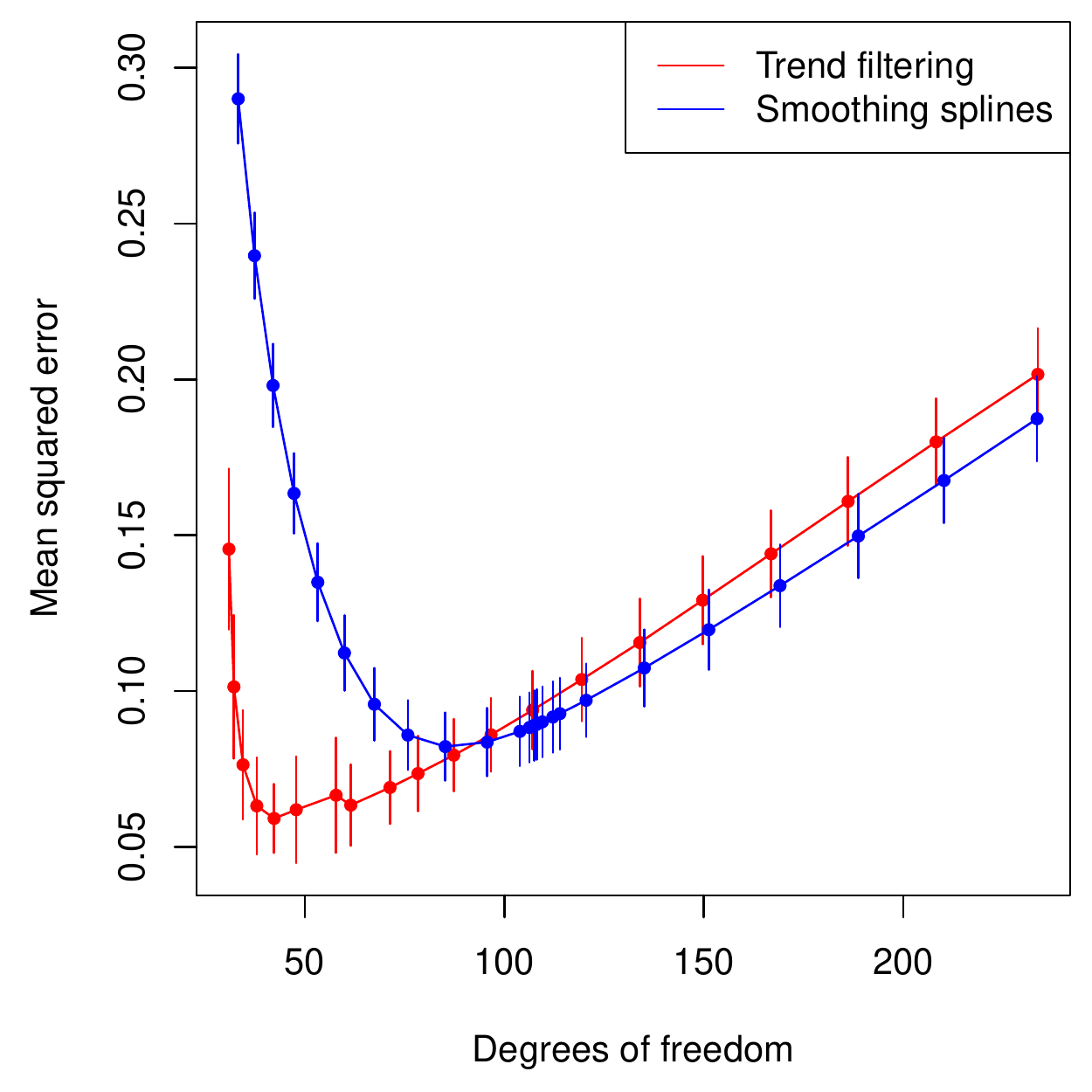}
\end{minipage}
\hspace{-40pt}
\begin{minipage}[c]{0.45\textwidth}
\caption{\it\small MSE curves for additive trend filtering and additive
smoothing splines, computed over 20 repetitions from the same simulation setup
as in Figure~\ref{fig:intro}. Vertical segments denote $\pm 1$ standard
deviations. The MSE curves are parametrized by degrees of freedom 
(computed via standard Monte Carlo methods over the 20 repetitions).}      
\label{fig:intro_mse}
\end{minipage}
\end{figure}
 
The first row of Figure \ref{fig:intro} shows the estimated component functions
from additive trend filtering, at a value of $\lambda$ that minimizes the mean 
squared error (MSE), computed over 20 repetitions. The second row shows the  
estimates from additive smoothing splines, also at a value of $\lambda$ that
minimizes the MSE.  We see that  
the trend filtering fits adapt well to the varying levels of smoothness, but the 
smoothing spline fits are undersmoothed, for the most part.  In terms of
effective degrees of freedom (df), the additive smoothing spline estimate is
much more complex, having about 85 df (computed via Monte Carlo over the 20
repetitions); the additive trend filtering has only about 42 df.  The third row
of the figure shows the estimates from additive smoothing splines, when
$\lambda$ is chosen so that the resulting df is roughly matches that of additive
trend filtering in the first row.  Now we see that the first component fit is
oversmoothed, yet the third is still undersmoothed. 

Figure \ref{fig:intro_mse} displays the MSE curves from additive trend
filtering, as a function of df. We see that trend filtering achieves a lower
MSE, and moreover, its MSE curve is optimized at a lower df (i.e., less complex
model) than that for smoothing splines.  This is analogous to what is typically
seen in the univariate setting \citep{tibshirani2014adaptive}.

\subsection{Multiple tuning parameters}
\label{sec:multi_lambda}

In problems \eqref{eq:add_quad} and \eqref{eq:add_tf}, consider generalizing the 
penalties in the criterions by
\begin{equation}
\label{eq:multi_lambda}
\sum_{j=1}^d \lambda_j \thetaj^T Q_j \thetaj \;\;\; \text{and} \;\;\;
\sum_{j=1}^d \lambda_j \big\|\Dj S_j \thetaj \big\|_1,
\end{equation}
respectively, which means we would now have $d$ tuning parameters $\lambda_j$,
$j=1,\ldots,d$. In practice, selecting multiple tuning parameters is
significantly more challenging than selecting a single one, as is needed in
\eqref{eq:add_quad} and \eqref{eq:add_tf}.  However, as pointed out by a referee
of this article, there has been a considerable amount of work dedicated to this
problem by authors studying additive models built from splines (or other linear
smoothers), e.g.,
\citet{gu1991minimizing,wood2000modelling,fahrmeir2001bayesian,ruppert2003semiparametric,wood2004stable,kim2004smoothing,rue2009approximate,wood2011fast,wood2015generalized,wood2016smoothing}.
Many of these papers use an efficient computational approach based on restricted
maximum likelihood (REML) for selecting $\lambda_j$, $j=1,\ldots,d$; see also
\citet{wood2017generalized} for a nice introduction and description of this
approach.  Unfortunately, as far as we see it, REML does not easily apply to
additive trend filtering.

In this paper, we focus on a single tuning parameter $\lambda$ as in
\eqref{eq:add_quad} and \eqref{eq:add_tf} mainly for simplicity; we are not
suggesting that this is always the preferred formulation in practice.  Many of
our results (the basis formulation, and the uniqueness and degrees of 
freedom results in Section \ref{sec:basic}) carry over immediately to the
multiple tuning parameter case.  Others (the error bounds in 
Section \ref{sec:bounds}) do not, though extending them may certainly be  
possible.  Furthermore, the motivating example of the last subsection is meant
to elucidate the differences in what additive smoothing splines and additive
trend filtering can do with a single tuning parameter; a serious applied
statistician, in just $d=3$ dimensions, would likely use REML or some related
technique to fit a multiple tuning parameter smoothing spline model, which would
bring it closer to additive trend filtering in performance here (it would be
able to adjust to the variable smoothness across the components, though still
not that within $f_{01}$). Of course, as $d$ grows larger, a separate tuning
parameter per component will generally become infeasible, and the single tuning 
parameter comparisons will become more meaningful.  

\subsection{Summary of contributions}

A summary of our contributions, and an outline for the
rest of this paper, are given below.

\begin{itemize}
\item In Section \ref{sec:basic}, we investigate basic 
  properties of the additive trend filtering model: an
  equivalent continuous-time formulation, a condition for uniqueness 
  of component function estimates, and a simple formula for the 
  effective degrees of freedom of the additive fit.  We also introduce two
  estimators related to additive trend filtering, based on splines, that 
  facilitate theoretical analysis (and are perhaps of interest in
  their own right).

\item In Section \ref{sec:bounds}, we derive error bounds   
  for additive trend filtering.  Assuming that the underlying
  regression function is additive, denoted by \smash{$f_0=\sum_{j=1}^d
    f_{0j}$}, and that \smash{$\TV(f_{0j}^{(k)})$} is bounded, for
  $j=1,\ldots,d$, we prove that the $k$th order additive trend filtering
  estimator converges to $f_0$ at the rate \smash{$n^{-(2k+2)/(2k+3)}$} when 
  the dimension $d$ is fixed (under weak assumptions), and at the rate    
  \smash{$d n^{-(2k+2)/(2k+3)}$} when $d$ is growing (under stronger 
  assumptions).  We prove that these rates are optimal in a
  minimax sense, and also show that additive smoothing splines (or more
  generally, additive models built from linear smoothers of any kind) are
  suboptimal over such a class of functions $f_0$.    

\item In Section \ref{sec:comp}, we study the backfitting algorithm 
  for additive trend filtering models, and give a connection
  between backfitting and an alternating projections scheme in the
  additive trend filtering dual problem.  This inspires a new parallelized 
  backfitting algorithm for additive trend filtering.

\item In Section \ref{sec:experiments}, we present empirical
  experiments and comparisons, and we also investigate the use of
  cross-validation for tuning parameter selection, and multiple tuning parameter
  models as in \eqref{eq:multi_lambda}.  In Section \ref{sec:discuss}, we give a
  brief discussion. 
\end{itemize}

\section{Basic properties}
\label{sec:basic}

In this section, we derive a number of basic properties of additive
trend filtering estimates, starting with a representation for the
estimates as continuous functions over $\R^d$ (rather than simply 
discrete fitted values at the input points).

\subsection{Falling factorial representation}
\label{sec:fall_fact}

We may describe additive trend filtering in \eqref{eq:add_tf}
as an estimation problem written in {\it analysis form}.  
The components are modeled directly by the parameters  
$\thetaj$, $j=1,\ldots,d$, and the desired structure is established 
by regularizing the discrete derivatives of these parameters,
through the penalty terms \smash{$\|\Dj S_j  
  \thetaj\|_1$}, $j=1,\ldots,d$.  Here, we present an
alternative representation for \eqref{eq:add_tf} in 
{\it basis form}, where each component is expressed as a linear
combination of basis functions, and regularization is applied to the
coefficients in this expansion.     

Before we derive the basis formulation that underlies additive trend
filtering, we first recall the {\it falling factorial basis} 
\citep{tibshirani2014adaptive,wang2014falling}.  
Given knot points $t^1 < \ldots < t^n \in \R$, the $k$th order  
falling factorial basis functions $h_1,\ldots,h_n$ are defined by
\begin{equation}
\label{eq:hbasis}
\begin{gathered}
h_i(t) = \prod_{\ell=1}^{i-1} (t-t^\ell), 
\;\;\; i=1,\ldots,k+1,\\ 
h_{i+k+1}(t) = \prod_{\ell=1}^k (t-t^{i+\ell}) \cdot  
1\{t > t^{i+k}\}, 
\;\;\;  i=1,\ldots,n-k-1.
\end{gathered}
\end{equation}
We denote $1\{t > a\}=1$ when $t > a$, and 0 otherwise. 
(Also, our convention is to define the empty product to be 1, so that 
$h_1(t)=1$.) The functions $h_1,\ldots,h_n$ are piecewise 
polynomial functions of order $k$, and appear very similar
in form to the $k$th order truncated power basis functions.  
In fact, when $k=0$ or $k=1$, the two bases are exactly equivalent
(meaning that they have the same span). Similar 
to an expansion in the truncated power basis, an expansion in the falling 
factorial basis, 
$$
g = \sum_{i=1}^n \alpha^i h_i
$$
is a continuous piecewise polynomial function, having a global
polynomial structure determined by $\alpha^1,\ldots,\alpha^{k+1}$,
and exhibiting a knot---i.e., a change in its $k$th derivative---at
the location $t^{i+k}$ when $\alpha^{i+k+1} \not= 0$.  
But, unlike the truncated power functions, the falling factorial
functions in \eqref{eq:hbasis} are not splines, and when $g$ (as
defined above) has a knot at a particular location, it displays a change not
only in its $k$th derivative at this location, but also in all lower order
derivatives (i.e., all derivatives of orders $1,\ldots,k-1$).  

\citet{tibshirani2014adaptive,wang2014falling} establish a connection 
between univariate trend filtering and the falling 
factorial functions, and show that the trend filtering problem can be  
interpreted as a sparse basis regression problem using these 
functions. As we show next, the analogous result holds for additive
trend filtering. 

\begin{lemma}[\textbf{Falling factorial representation}]
\label{lem:fall_fact}
For $j=1,\ldots,d$, let \smash{$h_1^{(\Xj)}, \ldots, h_n^{(\Xj)}$} be
the falling factorial basis in \eqref{eq:hbasis} with knots 
$(t^1,\ldots,t^n)=S_j \Xj$, the $j$th dimension of the input points, 
properly sorted. Then the additive trend filtering problem \eqref{eq:add_tf} is
equivalent to the problem  
\begin{equation}
\label{eq:add_tf_func}
\begin{alignedat}{2}
&\ccol{\min_{\alpha_1,\ldots,\alpha_d \in \R^n}} \quad
&&\half \sum_{i=1}^n \bigg(Y^i - \bar{Y} -
\sum_{j=1}^d \sum_{\ell=1}^n 
\alpha^\ell_j h^{(\Xj)}_\ell (X^i_j) \bigg)^2 + 
\lambda k! \sum_{j=1}^d \sum_{\ell=k+2}^n |\alpha^\ell_j| \\  
&\ccol{\st}\quad &&\sum_{i=1}^n \sum_{\ell=1}^n \alpha^\ell_ j
h^{(\Xj)}_\ell (X^i_j) =0,  \;\;\; j=1,\ldots,d,
\end{alignedat}
\end{equation}
in that, at any solutions in \eqref{eq:add_tf}, \eqref{eq:add_tf_func}, we have  
$$
\htheta^i_j = \sum_{\ell=1}^n \hat\alpha^\ell_j
h^{(\Xj)}_\ell (X^i_j), 
\;\;\; i=1,\ldots,n, \; j=1,\ldots,d.
$$
An alternative way of expressing problem \eqref{eq:add_tf_func} is   
\begin{equation}
\label{eq:add_tf_func_tv}
\begin{alignedat}{2}
&\ccol{\min_{f_j \in \cH_j, \, j=1,\ldots,d}} \quad   
&&\half \sum_{i=1}^n \bigg(Y^i - \bar{Y} -  
\sum_{j=1}^d f_j(X^i_j) \bigg)^2 +  
\lambda \sum_{j=1}^d \TV(f_j^{(k)}) \\
&\ccol{\st} \quad &&\sum_{i=1}^n f_j(X^i_j) = 0, \;\;\; j=1,\ldots,d,       
\end{alignedat}
\end{equation}
where 
\smash{$\cH_j=\spa\{h^{(\Xj)}_1,\ldots,h^{(\Xj)}_n\}$} is the 
span of the falling factorial basis over the $j$th dimension, and  
\smash{$f_j^{(k)}$} is the $k$th weak derivative of $f_j$, 
$j=1,\ldots,d$.  In this form, at any solutions in \eqref{eq:add_tf}, 
\eqref{eq:add_tf_func_tv},  
$$
\htheta^i_j = \hf_j(X^i_j), 
\;\;\; i=1,\ldots,n, \;  j=1,\ldots,d.
$$
\end{lemma}

\begin{proof}
For $j=1,\ldots,d$, define the $k$th order falling factorial
basis matrix \smash{$H^{(\Xj,k)} \in \R^{n\times n}$} by
\begin{equation}
\label{eq:hmat}
H_{i\ell}^{(\Xj,k)} = h^{(\Xj)}_\ell(X^i_j), 
\;\;\; i=1,\ldots,n, \;  \ell=1,\ldots,n.
\end{equation}
Note that the columns of \smash{$H^{(\Xj,k)}$} follow the order of the
sorted inputs $S_j\Xj$, but the rows do not; however, for 
\smash{$S_j H^{(\Xj,k)}$}, both its rows and columns of follow the
order of $S_j\Xj$.  From \citet{wang2014falling}, we know that 
$$
(S_j H^{(\Xj,k)})^{-1} = \left[\begin{array}{c}
C^{(\Xj,k+1)} \\ \frac{1}{k!} \Dj
\end{array}\right],
$$
for some matrix \smash{$C^{(\Xj,k+1)} \in \R^{(k+1)\times n}$}, i.e., 
\begin{equation}
\label{eq:hinv}
(H^{(\Xj,k)})^{-1} = \left[\begin{array}{c}
C^{(\Xj,k+1)} \\ \frac{1}{k!} \Dj
\end{array}\right] S_j.
\end{equation}
Problem \eqref{eq:add_tf_func} is given by reparameterizing 
\eqref{eq:add_tf} according to \smash{$\thetaj =
  H^{(\Xj,k)}\alpha_j$}, for $j=1,\ldots,d$.
As for \eqref{eq:add_tf_func_tv}, the equivalence
between this and \eqref{eq:add_tf_func} follows by noting  
that for \smash{$f_j=\sum_{\ell=1}^n \alpha^\ell_j h_\ell^{(\Xj)}$},
we have 
$$
f_j^{(k)}(t) = k! + k!\sum_{\ell=k+2}^n 
\alpha^\ell_j \cdot 1\{t > X^{\ell-1}_j\},
$$
and so
\smash{$\TV(f_j^{(k)}) = k! \sum_{\ell=k+2}^n |\alpha^\ell_j|$}, for
each $j=1,\ldots,d$.
\end{proof}

This lemma not only provides an interesting
reformulation for additive trend filtering, it is also practically
useful in that it allows us to perform 
interpolation or extrapolation using the additive trend filtering
model.  That is, from a solution
\smash{$\htheta=(\htheta_1,\ldots,\htheta_d)$} in \eqref{eq:add_tf},
we can extend each component fit \smash{$\hthetaj$} to the real line,
by forming an appropriate linear combination of falling factorial
functions: 
\begin{equation}
\label{eq:fit_j}
\hf_j(x_j) =
\sum_{\ell=1}^n \hat\alpha^\ell_j h_\ell ^{(\Xj)} (x_j), \;\;\;
x_j \in \R.
\end{equation} 
The coefficients above are determined by the 
relationship \smash{$\hat\alpha_j = (H^{(\Xj,k)})^{-1} \hthetaj$},  
and  are easily computable given the highly
structured form of \smash{$(H^{(\Xj,k)})^{-1}$}, as revealed in 
\eqref{eq:hinv}.  
Writing the coefficients in block form, as in 
$\hat\alpha_j=(\hat{a}_j,\hat{b}_j) \in 
\R^{(k+1)} \times \R^{(n-k-1)}$, we have
\begin{gather}
\label{eq:a_j}
\hat{a}_j = C^{(\Xj,k+1)} S_j \hthetaj, \\
\label{eq:b_j}
\hat{b}_j = \frac{1}{k!} \Dj S_j \hthetaj.
\end{gather}
The first $k+1$ coefficients 
\smash{$\hat{a}_j$} index the pure polynomial functions 
\smash{$h^{(\Xj)}_1,\ldots,h^{(\Xj)}_{k+1}$}. These coefficients will
be generically dense 
(the form of \smash{$C^{(\Xj,k+1)}$} is not important here, so we omit
it for simplicity, but details are given in Appendix \ref{app:fastext}).
The last $n-k-1$ coefficients 
\smash{$\hat{b}_j$} index the knot-producing functions
\smash{$h^{(\Xj)}_{k+2}, \ldots,h^{(\Xj)}_n$}, and when
\smash{$(\hat{b}_j)_\ell = 
\frac{1}{k!}(\Dj S_j \hthetaj)_\ell \not= 0$}, the fitted  
function \smash{$\hf_j$} exhibits a knot at the $(\ell+k)$th sorted
input point among $S_j\Xj$, i.e., at 
\smash{$\Xj^{(\ell+k)}$}.  Figure \ref{fig:3dplot} gives an example.  

\begin{figure}[p]
\centering
\hspace{-7pt} 
\includegraphics[width=0.325\textwidth]{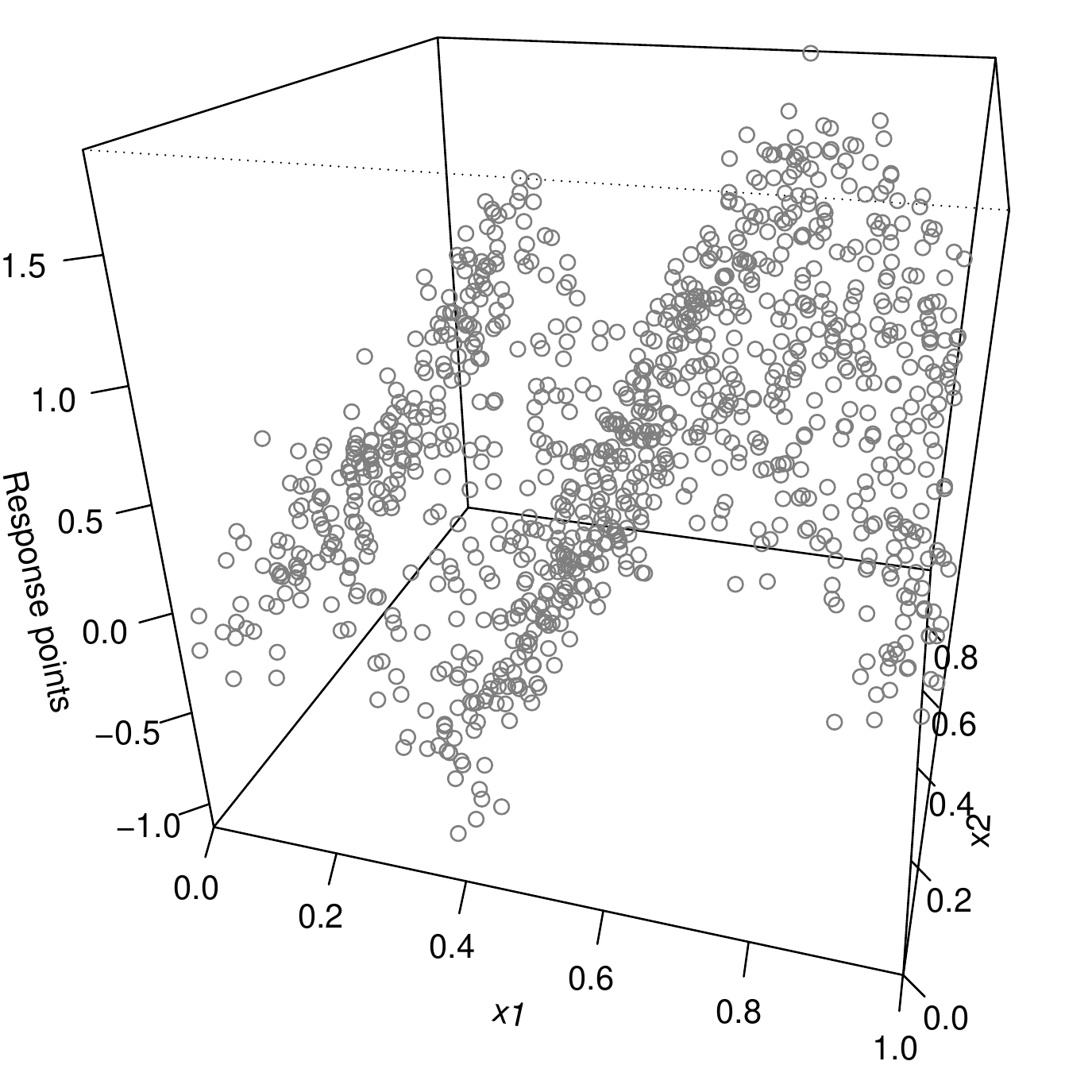}
\hspace{-7pt} 
\includegraphics[width=0.325\textwidth]{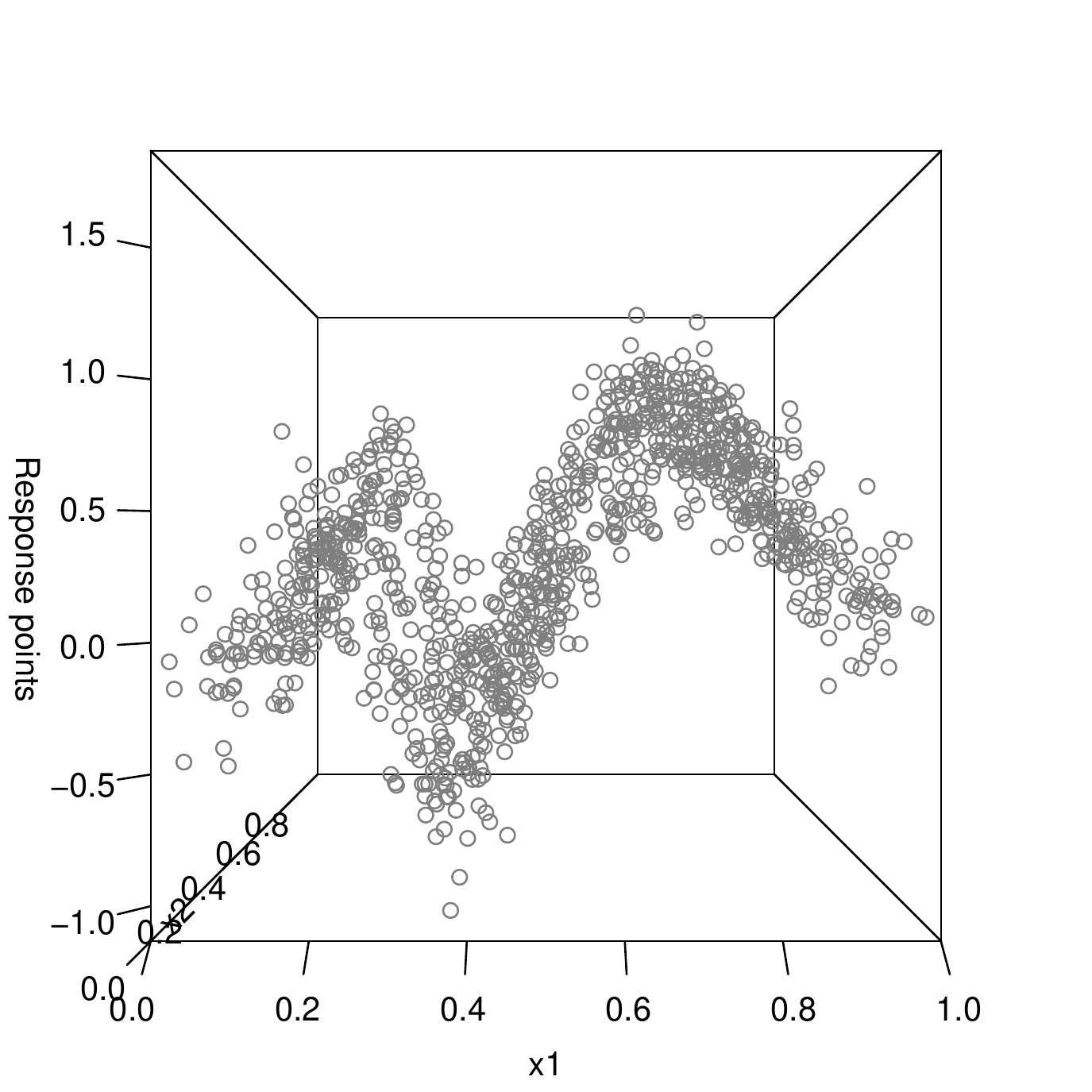}
\hspace{-7pt} 
\includegraphics[width=0.325\textwidth]{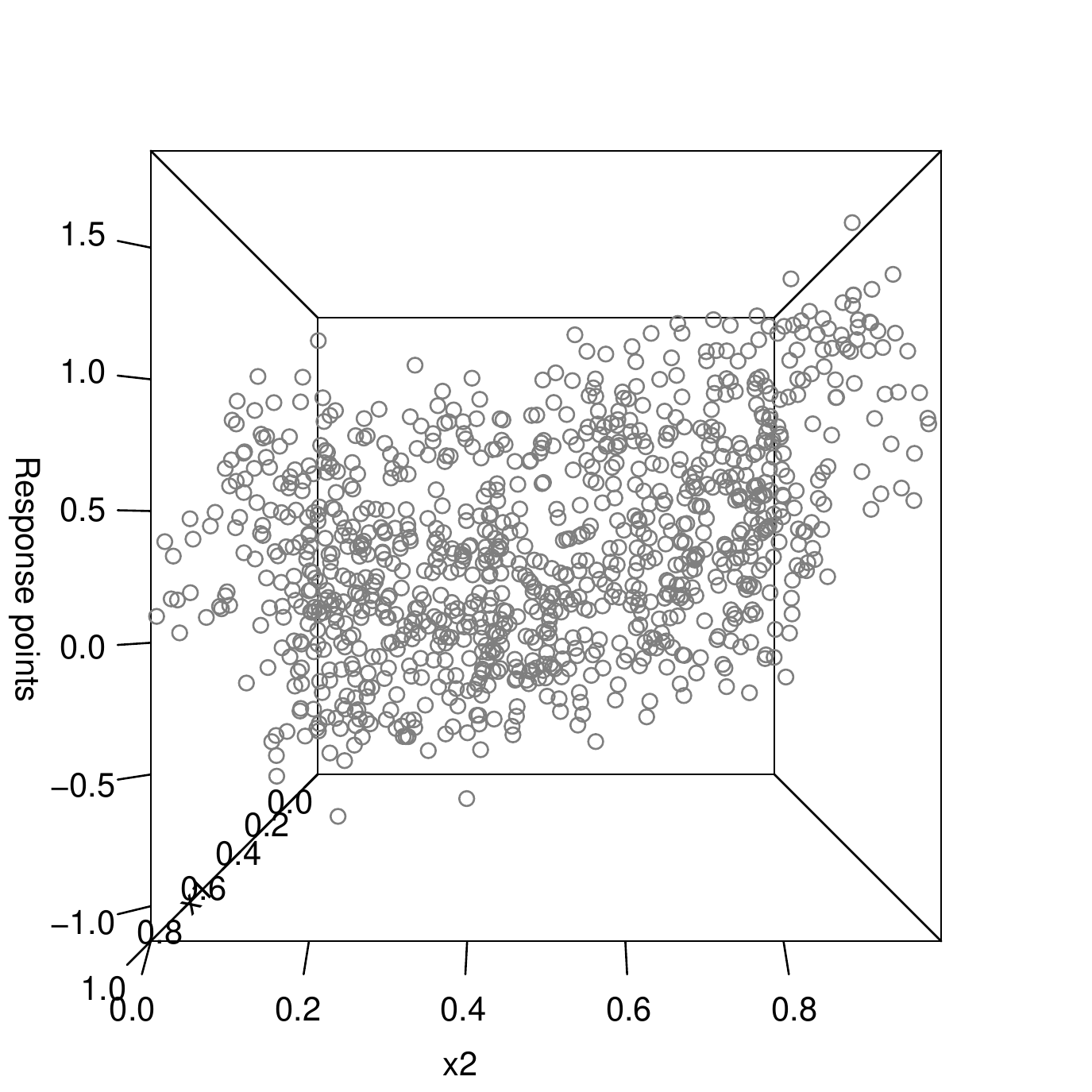} 
\bigskip \\
\hspace{-10pt} 
\includegraphics[width=0.5\textwidth]{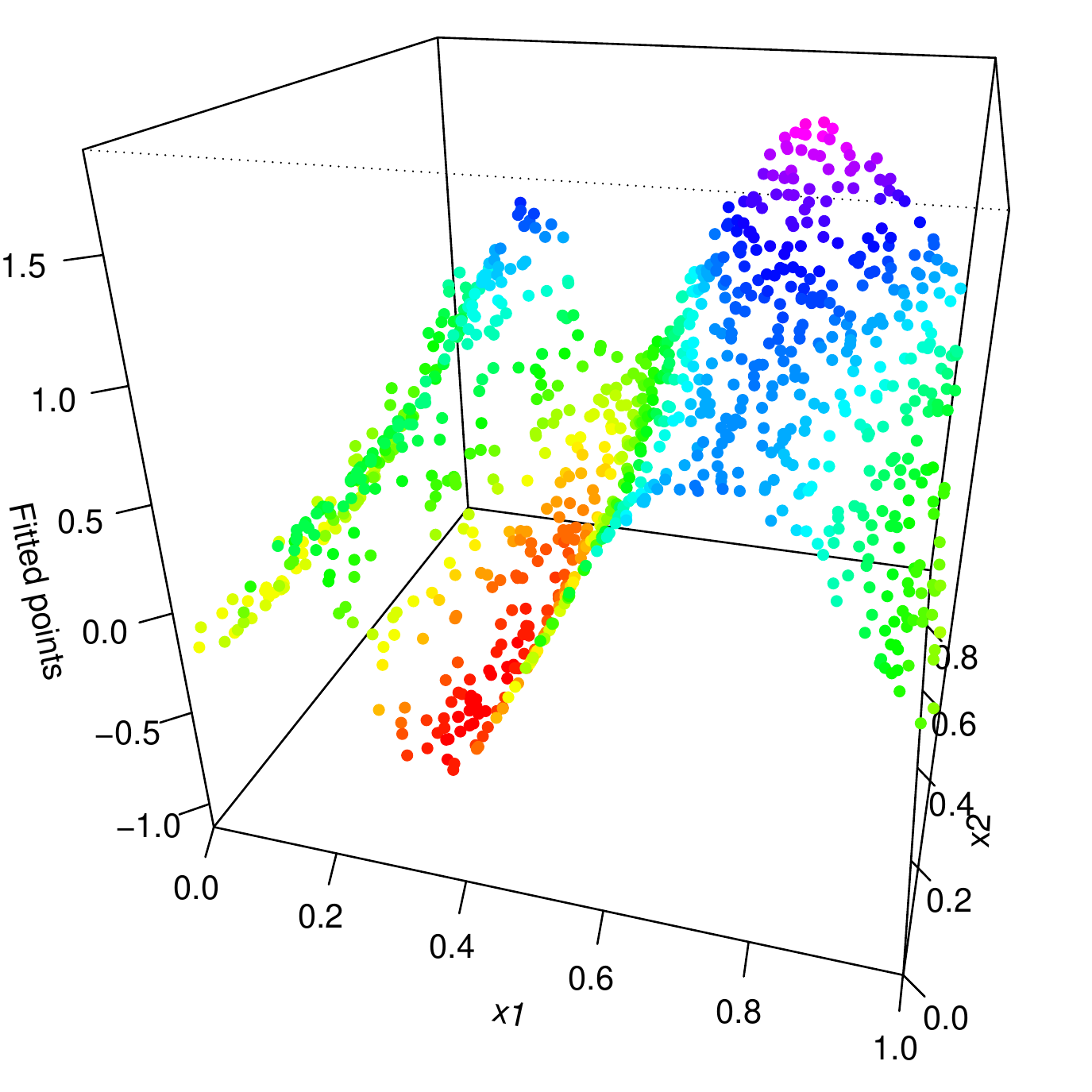}
\hspace{-10pt} 
\includegraphics[width=0.5\textwidth]{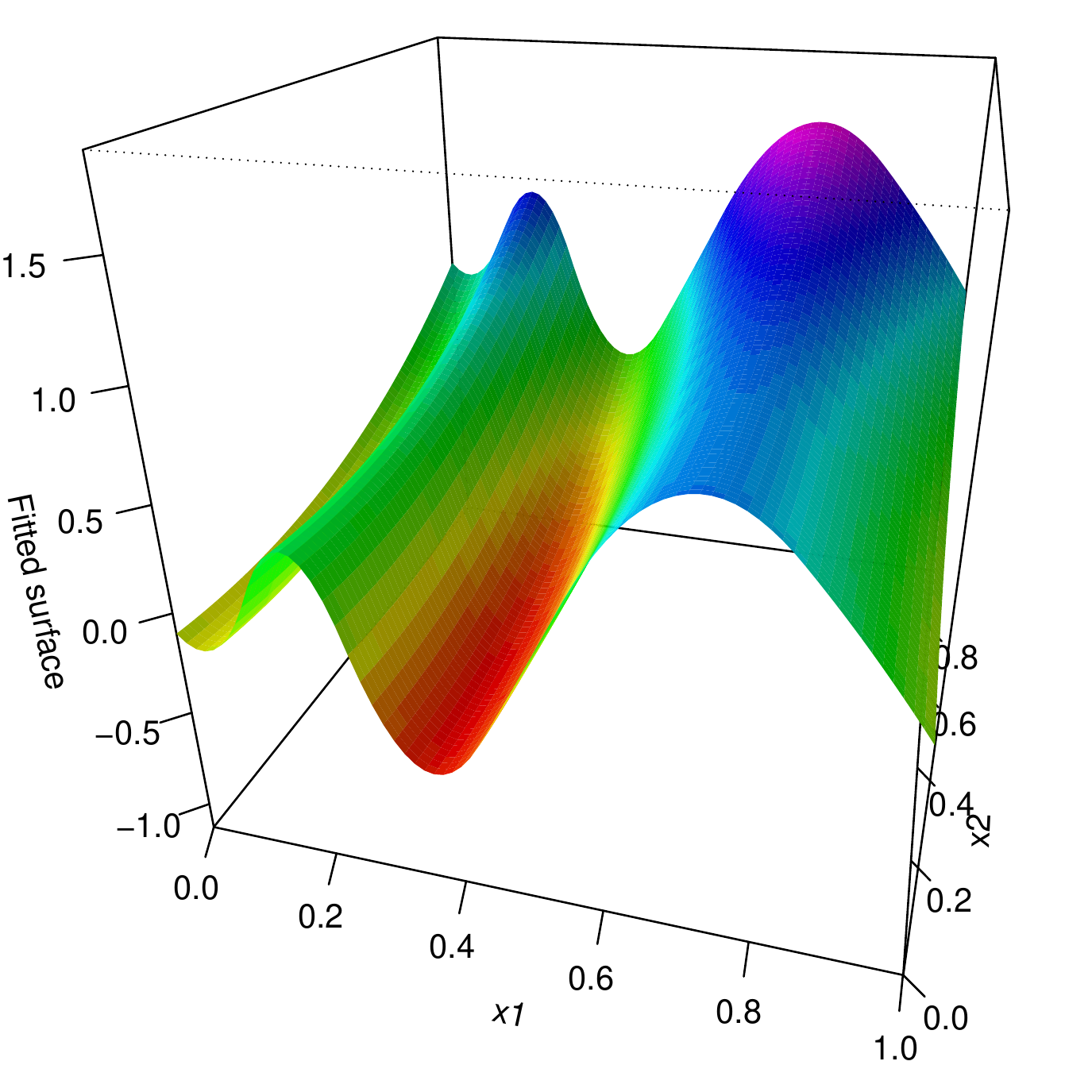}
\bigskip \\ 
\caption{\it\small An example of extrapolating the fitted additive
  trend filtering model, 
  where $n=1000$ and $d=2$.  We generated input points 
  \smash{$X^i \isim \mathrm{Unif}[0,1]^2$},  
  $i=1,\ldots,1000$, and responses
  \smash{$Y^i \isim N(\sum_{j=1}^2 f_{0j}(X^i_j),\sigma^2)$},
  $i=1,\ldots,1000$, where we 
  \smash{$f_{01}(x_1)=\sqrt{x_1}\sin(3\pi/(x_1+1/2))$} and 
  \smash{$f_{02}(x_2)=x_2(x_2-1/3)$}, and $\sigma=0.36$.
  The top row shows three perspectives of the data.  The bottom left
  panel shows the fitted values from additive trend filtering 
  \eqref{eq:add_tf} (with $k=2$ and $\lambda=0.004$), where points
  are colored by their depth for visualization purposes. The 
  bottom right panel shows the 2d surface associated with the trend 
  filtering estimate, \smash{$\hf_1(x_1)+\hf_2(x_2)$} over $(x_1,x_2)
  \in [0,1]^2$, with each component function extrapolated as in
  \eqref{eq:fit_j}.} 
\label{fig:3dplot}
\end{figure}

We note that the coefficients
\smash{$\hat\alpha_j=(\hat{a}_j,\hat{b}_j)$} in
\eqref{eq:a_j}, \eqref{eq:b_j} can be computed in $O(n)$
operations and $O(1)$ memory.  This makes extrapolation of the $j$th
fitted function \smash{$\hf_j$} in \eqref{eq:fit_j} highly efficient.
Details are given in Appendix \ref{app:fastext}. 

\subsection{Uniqueness of component fits}
\label{sec:unique}

It is easy to see that, for the problem \eqref{eq:add_tf}, the
additive fit \smash{$\sum_{j=1}^d \hthetaj$} is always uniquely
determined: denoting 
\smash{$\sum_{j=1}^d \thetaj = T\theta$} for a linear operator $T$
and $\theta=(\theta_1,\ldots,\theta_d) \in \R^{nd}$, 
the loss term $\|y-T\theta\|_2^2$ is strictly convex in the variable 
$T\theta$, and this, along with the convexity of the problem 
\eqref{eq:add_tf}, implies a unique additive fit \smash{$T\htheta$}, no matter
the choice of solution \smash{$\htheta=(\htheta_1,\ldots,\htheta_d) \in  
  \R^{nd}$}. 

On the other hand, when $d>1$, the criterion in \eqref{eq:add_tf} is
not strictly convex in $\theta$, and hence there need not be a unique  
solution \smash{$\htheta$}, i.e., the individual components fits 
\smash{$\hthetaj$}, $j=1,\ldots,d$ need not be uniquely determined. 
We show next that uniqueness of the component fits can be 
guaranteed under some conditions on the input matrix
$X = [X_1 \, \cdots \, X_d] \in \R^{n\times d}$.  We will rely on the 
falling factorial representation for additive trend filtering,
introduced in the previous subsection, and on the notion of 
{\it general position}: a matrix $A \in \R^{m \times p}$ is said to  
have columns in general position provided that, for any 
$\ell < \min\{m,p\}$, subset of $\ell+1$ columns denoted 
\smash{$A_{i_1},\ldots,A_{i_{\ell+1}}$}, and signs
 $s_1,\ldots,s_{\ell+1} \in \{-1,1\}$, the affine span of  
\smash{$\{s_1A_{i_1},\ldots,s_{\ell+1}A_{i_{\ell+1}}\}$}
does not contain any element of 
\smash{$\{\pm A_i : i \not= i_1,\ldots,i_{\ell+1}\}$}.  
Informally, if the columns of $A$ are not in general position,
then there must be some small subset of columns that are affinely
dependent. 

\begin{lemma}[\textbf{Uniqueness}]
\label{lem:unique}
For $j=1,\ldots,d$, let
\smash{$H^{(\Xj,k)} \in \R^{n\times n}$} be
the falling factorial basis matrix constructed over
the sorted $j$th dimension of inputs $S_j\Xj \in \R^n$, as in 
\eqref{eq:hmat}. Decompose \smash{$H^{(\Xj,k)}$} into its first 
$k+1$ columns \smash{$P^{(\Xj,k)} \in \R^{n\times(k+1)}$}, and its
last $n-k-1$ columns \smash{$K^{(\Xj,k)} \in \R^{n\times(n-k-1)}$}.
The former contains evaluations of the pure polynomials  
\smash{$h^{(\Xj)}_1,\ldots,h^{(\Xj)}_{k+1}$}; the latter
contains evaluations of the knot-producing functions 
\smash{$h^{(\Xj)}_{k+2},\ldots,h^{(\Xj)}_n$}. Also, 
let \smash{$\tilde{P}^{(\Xj,k)}$} denote
the matrix \smash{$P^{(\Xj,k)}$} with its first column removed, for
$j=1,\ldots,d$, and $M=I-\one\one^T/n$.  Define  
\begin{equation}
\label{eq:ptil}
\tilde{P} = M \big[\, \tilde{P}^{(X_1,k)} \;\; 
\ldots \;\; \tilde{P}^{(X_d,k)} \, \big] 
\in \R^{n\times dk},  
\end{equation}
the product of $M$ and the columnwise concatenation of
\smash{$\tilde{P}^{(\Xj,k)}$}, $j=1,\ldots,d$. Let 
$UU^T$ denote the projection operator onto the space orthogonal to
the column span of \smash{$\tilde{P}$}, where $U \in
\R^{n\times(n-kd-1)}$ has orthonormal columns, and define 
\begin{equation}
\label{eq:ktil}
\tilde{K} = U^T M \big[\, K^{(X_1,k)} \;\; 
\ldots \;\; K^{(X_d,k)} \, \big] 
\in \R^{(n-kd-1)\times (n-k-1)d},  
\end{equation}
the product of $U^T M$ and the columnwise concatenation of  
\smash{$K^{(\Xj,k)}$}, $j=1,\ldots,d$.  A sufficient
condition for uniqueness of the additive trend filtering solution
in \eqref{eq:add_tf} can now be given in two parts.  

\begin{enumerate}
\item If \smash{$\tilde{K}$} has columns in general position,
  then the knot-producing parts of all component
  fits are uniquely determined, i.e., for each $j=1,\ldots,d$, the
  projection of \smash{$\hthetaj$} onto the column
  space of \smash{$K^{(\Xj,k)}$} is unique.

\item If in addition \smash{$\tilde{P}$} has full column rank,
  then the polynomial parts of component fits are uniquely
  determined, i.e., for each $j=1,\ldots,d$, the projection of
  \smash{$\hthetaj$} 
  onto the column space of \smash{$P^{(\Xj,k)}$} is unique, and thus
  the component fits \smash{$\hthetaj$}, $j=1,\ldots,d$ are all
  unique.   
\end{enumerate}
\end{lemma}

The proof is deferred to Appendix \ref{app:unique}.  To rephrase, the
above lemma decomposes each component of the additive trend filtering 
solution according to  
$$
\hthetaj =
\hthetaj^{\mathrm{poly}} + \hthetaj^{\mathrm{knot}}, \;\;\; 
j=1,\ldots,d, 
$$
where \smash{$\hthetaj^{\mathrm{poly}}$} exhibits a purely
polynomial trend over $S_j\Xj$, and  
\smash{$\hthetaj^{\mathrm{knot}}$} exhibits a piecewise polynomial  
trend over $S_j\Xj$, and hence determines the knot locations, for  
$j=1,\ldots,d$.  The lemma shows that the knot-producing parts 
\smash{$\hthetaj^{\mathrm{knot}}$}, $j=1,\ldots,d$ are uniquely
determined when the columns of \smash{$\tilde{K}$} are in general
position, and the polynomial parts
\smash{$\hthetaj^{\mathrm{knot}}$}, $j=1,\ldots,d$ are unique when the 
columns of \smash{$\tilde{K}$} are in general
position, and the columns of \smash{$\tilde{P}$} are linearly
independent. 

The conditions placed on
\smash{$\tilde{P},\tilde{K}$} in Lemma \ref{lem:unique} are not
strong.  When $n>kd$, and the elements of input matrix $X$ are drawn 
from a density over $\R^{nd}$, it is not hard to show that
\smash{$\tilde{P}$} has full column rank with probability 1.  We
conjecture that, under the same conditions, 
\smash{$\tilde{K}$} will also have columns in general position with
probability 1, but do not pursue a proof.    

\begin{remark}[\textbf{Relationship to concurvity}]
It is interesting to draw a connection to \citet{buja1989linear}. 
In the language used by these authors, when \smash{$\tilde{P}$} has  
linearly dependent columns, we say that the predictor variables    
display {\it concurvity}, i.e., linear dependence
after nonlinear (here, polynomial) transformations are applied. 
\citet{buja1989linear} establish that the components in the  
additive model \eqref{eq:add_quad}, built with quadratic penalties,  
are unique provided there is no concurvity between variables.  In
comparison, Lemma \ref{lem:unique} establishes uniqueness of the  
additive trend filtering components when there is no 
concurvity between variables, and additionally, the columns of
\smash{$\tilde{K}$} are in general position. The latter two
conditions together can be seen as requiring no
{\it generalized concurvity}---if \smash{$\tilde{K}$} were to fail
the general position assumption, then there would be a small subset  
of the variables that are linearly dependent after nonlinear (here, 
piecewise polynomial) transformations are applied.   
\end{remark}

\subsection{Dual problem}
\label{sec:dual}

Let us abbreviate \smash{$D_j=\Dj$}, $j=1,\ldots,d$ for the penalty
matrices in the additive trend filtering problem \eqref{eq:add_tf}.
Basic arguments in convex analysis, deferred to Appendix
\ref{app:dual}, show that the dual of problem \eqref{eq:add_tf} can be
expressed as: 
\begin{equation}
\label{eq:add_tf_dual}
\begin{gathered}
\min_{u \in \R^n} \; \|Y - \bar{Y}\one - u\|_2^2
\;\; \st \;\; u \in U = U_1 \cap \cdots \cap U_d, \\ 
\text{where} \;\;\; U_j = \{S_j D_j^T v_j : \|v_j\|_\infty \leq
\lambda\}, \;\;\; j=1,\ldots,d,
\end{gathered}
\end{equation}
and that primal and dual solutions in \eqref{eq:add_tf}, \eqref{eq:add_tf_dual}
are related by: 
\begin{equation}
\label{eq:add_tf_pd}
\sum_{j=1}^d \hthetaj = Y - \bar{Y}\one - \hu.
\end{equation}
From the form of \eqref{eq:add_tf_dual}, it is clear that we can
write the (unique) dual solution as \smash{$\hu =  
  \Pi_U(Y-\bar{Y}\one)$}, where $\Pi_U$ is the (Euclidean)
projection operator onto $U$. Moreover, using \eqref{eq:add_tf_pd},  
we can express the additive fit as \smash{$\sum_{j=1}^d
  \hthetaj = (\Id-\Pi_U)(Y-\bar{Y}\one)$}, where   
$\Id-\Pi_U$ is the operator that gives the residual from projecting
onto $U$.  These relationships will be revisited in Section
\ref{sec:comp}, where we return to the dual perspective, and argue 
that the backfitting algorithm for the additive trend filtering
problem \eqref{eq:add_tf} can be seen as a type of alternating
projections algorithm for its dual problem \eqref{eq:add_tf_dual}.  

\subsection{Degrees of freedom}

In general, given data $Y \in \R^n$ with $\E(Y)=\eta$, 
$\Cov(Y)=\sigma^2 I$, and an estimator \smash{$\hat\eta$} of        
$\eta$, recall that we define the {\it effective degrees of freedom}
of \smash{$\hat\eta$} as
\citep{efron1986biased,hastie1990generalized}:   
$$
\df(\hat\eta) = \frac{1}{\sigma^2}\sum_{i=1}^n 
\Cov\big(\hat\eta^i(Y),Y^i\big),
$$
where 
\smash{$\hat\eta(Y)=(\hat\eta^1(y),\ldots,\hat\eta^n(Y))$}.  Roughly 
speaking, the above definition sums the influence of the $i$th
component $Y^i$ on its corresponding fitted value
\smash{$\hat\eta^i(Y)$}, across $i=1,\ldots,n$.  A precise
understanding of degrees of freedom is  
useful for model comparisons (recall the x-axis in Figure
\ref{fig:intro_mse}), and other reasons. For linear 
smoothers, in which \smash{$\hat\eta(Y)=SY$} for some $S \in
\R^{n\times n}$, it is clear that \smash{$\df(\hat\eta)=\tr(S)$},  
the trace of $S$.  (This also covers additive models whose 
components are built from univariate linear smoothers, because in 
total these are still just linear smoothers: the additive fit is still
just a linear function of $Y$.)     

Of course, additive trend filtering is a not a linear smoother;
however, it is a particular type of generalized lasso estimator, and
degrees of freedom for such a class of estimators is well-understood  
\citep{tibshirani2011solution,tibshirani2012degrees}. 
The next result is an consequence of existing generalized lasso
theory, proved in Appendix \ref{app:df}. 

\begin{lemma}[\textbf{Degrees of freedom}]
\label{lem:df}
Assume the conditions of Lemma \ref{lem:unique}, i.e., that the matrix 
\smash{$\tilde{P}$} in \eqref{eq:ptil} has full column rank, and the
matrix \smash{$\tilde{K}$} in \eqref{eq:ktil} is in general position.
Assume also that the response is Gaussian,
$Y \sim N(\eta,\sigma^2 I)$, and treat the input points
$X^i \in \R^d$, $i=1,\ldots,n$ as fixed and
arbitrary, as well as the tuning parameter value $\lambda \geq 0$.
Then the additive trend filtering fit from \eqref{eq:add_tf} has
degrees of freedom 
$$
\df\bigg(\sum_{j=1}^d\hthetaj \bigg) = 
\E \bigg(\sum_{j=1}^d (\text{\rm number of knots in $\hthetaj$}) 
\bigg) + kd. 
$$
\end{lemma}

\begin{remark}[\textbf{The effect of shrinkage}]
Lemma \ref{lem:df} says that for an unbiased estimate of the
degrees of freedom of the additive trend filtering fit, we count the 
number of knots in each component fit \smash{$\hthetaj$} (recall that 
this is the number of nonzeros in \smash{$\Dj\hthetaj$}, i.e., the
number of changes in the discrete $(k+1)$st derivative), add them up
over $j=1,\ldots,d$, and add $kd$.  This may seem surprising,
as these knot locations are chosen adaptively based on the  
data $Y$. But, such adaptivity is counterbalanced by the shrinkage 
induced by the $\ell_1$ penalty in \eqref{eq:add_tf} (i.e., for each
component fit \smash{$\hthetaj$}, there is shrinkage in the
differences between the attained $k$th derivatives on either side of a
selected knot).  See \citet{tibshirani2015degrees} for a study of this
phenomenon.    
\end{remark}

\begin{remark}[\textbf{Easy unbiased degrees of freedom estimation}] 
It is worth emphasizing that an unbiased estimate from Lemma
\ref{lem:df} for the degrees of freedom of the total fit in additive
trend filtering is very easy to calculate: we scan the
individual component fits and add up the number of knots that appear
in each one. The same cannot be said for additive smoothing splines,
or additive models built from
univariate linear smoothers, in general.     
Although computing the fit itself is typically cheaper with 
additive linear smoothers than with additive trend filtering,
computing the degrees of freedom is more 
challenging. For example, for the additive model in
\eqref{eq:add_quad} built with quadratic penalties, we have  
$$
\df\bigg(\sum_{j=1}^d \hthetaj \bigg) = 
\tr\Big(F^T F (F^T F + \lambda Q)^+ \Big),
$$
where $F \in \R^{n\times nd}$ has $d$ copies of the centering matrix
$M=I-\one\one^T/n \in \R^{n\times n}$ stacked across its columns, 
$Q \in \R^{nd   \times nd}$ is a block diagonal matrix with blocks
$MQ_jM$, $j=1,\ldots,d$, and $A^+$ denotes the Moore-Penrose
pseudoinverse of a matrix $A$.  The above formula does not
obviously decompose into a sum of quantities across components, and is 
nontrivial to compute post optimization of \eqref{eq:add_quad},
specifically when a backfitting algorithm as in
\eqref{eq:backfit_quad} has been used to compute a solution.  
\end{remark}

\subsection{Two related additive spline estimators}
\label{sec:splines}

From its equivalent formulation in \eqref{eq:add_tf_func_tv}, additive trend
filtering is seen to be closely related to two other additive spline estimators,
which we introduce here. Consider, for univariate function classes $\cS_j$,
$j=1,\ldots,d$, the problem   
\begin{equation}
\label{eq:add_func_tv}
\begin{alignedat}{2}
&\ccol{\min_{f_j \in \cS_j\, j=1,\ldots_d}} \quad 
&&\half \sum_{i=1}^n \bigg(Y^i - \bar{Y} - \sum_{j=1}^d f_j(X^i_j)   
\bigg)^2 + \lambda \sum_{j=1}^d \TV(f_j^{(k)}) \\
&\ccol{\st} \quad &&\sum_{i=1}^n f_j(X^i_j) = 0, 
\;\;\; j=1,\ldots,d.  
\end{alignedat}
\end{equation}
When each $\cS_j$, $j=1,\ldots,d$ is the set of $k$ times weakly
differentiable functions, we call the solution in \eqref{eq:add_func_tv} the 
{\it additive locally adaptive regression spline} of order $k \geq 0$, as it
is the natural extension of the univariate estimator considered in
\citet{mammen1997locally}.  Denote by \smash{$\hf_j$}, $j=1,\ldots,d$ this  
solution; the representation arguments used by these authors apply immediately   
to the additive setting, and imply that each \smash{$\hf_j$}, $j=1,\ldots,d$ is 
indeed a spline of degree $k$ (justifying the choice of name). The same
arguments show that, for $k=0$ or $k=1$, the knots of the spline \smash{$\hf_j$}  
lie among the $j$th dimension of the input points $X^1_j,\ldots,X^n_j$, for 
$j=1,\ldots,d$, but for $k \geq 2$, this need not be true, and in general the
components will be splines with knots at locations other than the inputs.

We can facilitate computation by taking $\cS_j=\cG_j$, where $\cG_j$ is the
set of splines of degree $k$ with knots lying among the $j$th dimension of
inputs $X^1_j,\ldots,X^n_j$, for $j=1,\ldots,d$.  We call the resulting
solution the {\it restricted additive locally adaptive regression spline} of
order $k \geq 0$. More precisely, we require that the splines in $\cG_j$ have
knots in a set $T_j$, which, writing $t_j = S_j\Xj$ for the sorted inputs along
the $j$th dimension, is defined by    
\begin{equation}
\label{eq:tj}
T_j = \begin{cases}
\big\{t^{k/2+2}_j,\ldots,t^{n-k/2}_j\big\} & 
\text{if $k$ is even}, \\
\big\{t^{(k+1)/2+1}_j,\ldots,t^{n-(k+1)/2}_j\big\} & 
\text{if $k$ is odd},  
\end{cases}
\end{equation}
i.e., defined by removing $k+1$ input points at the boundaries, for
$j=1,\ldots,d$. Setting $\cS_j=\cG_j$, $j=1,\ldots,d$ makes
\eqref{eq:add_func_tv} a finite-dimensional problem, just like
\eqref{eq:add_tf_func_tv}.  When $k=0$ or $k=1$, as claimed in Section
\ref{sec:fall_fact} (and shown in \citet{tibshirani2014adaptive}), the falling
factorial functions are simply splines, which means that $\cH_j=\cG_j$ for
$j=1,\ldots,d$, hence additive trend filtering and restricted additive locally
adaptive regression splines are the same estimator.  When $k \geq 2$, this is no
longer true, and they are not the same.  Additive trend filtering will be much
easier to compute, since \smash{$\TV(g^{(k)})$} does not admit a nice
representation in terms of discrete derivatives for a $k$th order spline (and
yet it does for a $k$th order falling factorial function, as seen in
\eqref{eq:add_tf}).

To summarize, additive locally adaptive splines, restricted additive locally
adaptive splines, and additive trend filtering all solve a problem of the form
\eqref{eq:add_func_tv} for different choices of function classes $\cS_j$,
$j=1,\ldots,d$.  For $k=0$ or $k=1$, these three estimators are equivalent.  For 
$k \geq 2$, they will be generically different, though our intuition tells us
that their differences should not be too large: the unrestricted problem admits
a solution that is a spline in each component; the restricted problem simply
forces these splines to have knots at the input points; and the trend filtering 
problem swaps splines for falling factorial functions, which are
highly similar in form.  Next, we give theory that confirms this intuition, in
large samples.

\section{Error bounds} 
\label{sec:bounds}

We derive error bounds for additive trend filtering and additive locally
adaptive regression splines (both the unrestricted and restricted variants),
when the underlying regression function is additive, and has components whose
derivatives are of bounded variation.  These results are actually special cases
of a more general result we prove in this section, on a generic
roughness-regularized additive estimator, where we assume a certain decay for
the entropy of the unit ball in the roughness operator.  We treat separately the 
settings in which the dimension $d$ of the input space is fixed and growing.  We
also complement our error rates with minimax lower bounds.  We start by
introducing helpful notation. 

\subsection{Notation} 
\label{sec:notation}

Given a distribution $Q$ supported on a set $D$, and i.i.d.\ samples $X^i$,   
$i=1,\ldots,n$ from $Q$, denote by $Q_n$ the associated empirical
distribution.   We define the $L_2(Q)$ and  
$L_2(Q_n)$ inner products, denoted \smash{$\langle \cdot, \cdot  
  \rangle_{L_2(Q)}$} and \smash{$\langle \cdot, \cdot \rangle_{L_2(Q_n)}$},
respectively, over functions $m,r : D \to \R$ 
$$
\langle m,r \rangle_{L_2(Q)} =  
\int_D m(x) r(x) \, dQ(x), \;\;\;\text{and}\;\;\;
\langle m,r \rangle_{L_2(Q_n)} =  
\frac{1}{n} \sum_{i=1}^n m(X^i) r(X^i).
$$
Definitions for the corresponding $L_2(Q)$ and $L_2(Q_n)$ norms, denoted 
\smash{$\|\cdot\|_{L_2(Q)}$} and  \smash{$\|\cdot\|_{L_2(Q_n)}$},
respectively, arise naturally from these inner products, defined by 
$$
\|m\|_2^2 = \langle m,m \rangle_2 = 
\int_D m(x)^2 \, dQ(x), \;\;\;\text{and}\;\;\; 
\|m\|_n^2 = \langle m,m \rangle_n = 
\frac{1}{n} \sum_{i=1}^n m(X^i)^2.
$$
Henceforth, we will abbreviate subscripts when using these norms and
inner products, writing $\|\cdot\|_2$ and $\|\cdot\|_n$ for the
$L_2(Q)$ and $L_2(Q_n)$ norms, respectively, and similarly for the
inner products.  This abbreviated notation omits the underlying  
distribution $Q$; thus, unless explicitly stated otherwise, the underlying 
distribution should always be interpreted as the distribution of the input 
points.  We will often call $\|\cdot\|_2$ the $L_2$ norm and $\|\cdot\|_n$ the 
empirical norm, and similarly for inner products.   

In what follows, of particular interest will be the case when $D=[0,1]^d$, and  
$m : [0,1]^d \to \R$ is an additive function, of the form
$$
m=\sum_{j=1}^d m_j,
$$ 
which we write to mean \smash{$m(x) = \sum_{j=1}^d m_j(x_j)$}.
In a slight abuse of notation (overload of notation), for each $j=1,\ldots,d$,
we will abbreviate the $L_2(Q_j)$ norm by $\|\cdot\|_2$, where $Q_j$ is the
$j$th marginal of $Q$, and will also abbreviate $L_2(Q_{jn})$ norm by
$\|\cdot\|_n$, where $Q_{jn}$ is the empirical distribution of \smash{$X^i_j$},
$i=1,\ldots,n$. We will use similar abbreviations for the inner products.    

A few more general definitions are in order. We denote the $L_\infty$ 
norm, also called the sup norm, of a function $f : D \to \R$ by 
\smash{$\|f\|_\infty =  \esssup_{z \in  D} |f(z)|$}.
For a functional $\nu$, acting on functions from $D$ to $\R$, 
we write $B_\nu(\delta)$ for the $\nu$-ball of radius $\delta > 0$, 
i.e., \smash{$B_\nu(\delta) = \{ f : \nu(f) \leq \delta\}$}. We abbreviate
$B_n(\delta)$ for the $\|\cdot\|_n$-ball of radius $\delta$, $B_2(\delta)$ for
the $\|\cdot\|_2$-ball of radius $\delta$, and $B_\infty(\delta)$ for the
$\|\cdot\|_\infty$-ball of radius $\delta$.  We will use these concepts fluidly,
without explicit reference to the domain $D$ (or its dimensionality), as the
meaning should be clear from the context.     

Lastly, for a set $S$ and norm $\|\cdot\|$, we define the covering number 
$N(\delta,\|\cdot\|,S)$ to be the smallest number of $\|\cdot\|$-balls  
of radius $\delta$ to cover $S$, and the packing number
$M(\delta,\|\cdot\|,S)$ to be the largest number of disjoint
$\|\cdot\|$-balls of radius $\delta$ that are contained in $S$.  
We call $\log N(\delta,\|\cdot\|,S)$ the entropy number.     

\subsection{Error bounds for a fixed dimension $d$}
\label{sec:error_bd}

We consider error bounds for the generic roughness-penalized estimator defined 
as a solution of 
\begin{equation}
\label{eq:add_func_j}
\begin{alignedat}{2}
&\ccol{\min_{f_j \in \cS_j, \, j=1,\ldots,d}}  \quad   
&&\half \sum_{i=1}^n \bigg(Y^i - \bar{Y} - 
\sum_{j=1}^d f_j(X^i_j) \bigg)^2 + \lambda \sum_{j=1}^d J(f_j) \\    
&\ccol{\st} \quad && \sum_{i=1}^n f_j(X^i_j) = 0, 
\;\;\;  j=1,\ldots,d, 
\end{alignedat}
\end{equation}
where $\cS_j$, $j=1,\ldots,d$ are univariate function spaces, and 
$J$ is a regularizer that acts on univariate functions.  We assume in this
subsection that the dimension $d$ of the input space is fixed, i.e., it does not
grow with $n$. Before stating our main result in this setting, we list and
briefly discuss our other assumptions.  First, we give our assumptions on the
data generation process.    

\begin{assumption}{A}{1}{\textbf{Continuous input distribution}}  
\label{as:x_dist}
The input points $X^i$, $i=1,\ldots,n$ are i.i.d.\ from a continuous
distribution $Q$ supported on $[0,1]^d$. 
\end{assumption}

\begin{assumption}{B}{1}{{\bf Generic regression model, with sub-Gaussian 
      errors}}    
\label{as:y_dist}
The responses $Y^i$, $i=1,\ldots,n$ follow the model 
$$
Y^i = \mu + f_0(X^i) + \epsilon^i, \;\;\; i=1,\ldots,n, 
$$
with overall mean $\mu \in \R$, where \smash{$\sum_{i=1}^n
  f_0(X^i)=0$} for identifiability. The errors $\epsilon^i$,
$i=1,\ldots,n$ are uniformly sub-Gaussian and have mean zero, i.e.,  
$$
\E(\epsilon)=0, \;\;\;\text{and}\;\;\;
\E[\exp(v^T \epsilon)] \leq \exp(\sigma^2 \|v\|_2^2/2) 
\;\, \text{for all $v \in \R^n$},
$$
for a constant $\sigma>0$. The errors and input points are independent.   
\end{assumption}

Assumptions \ref{as:x_dist} and \ref{as:y_dist} are very weak.  We do
not place any specific smoothness or additivity conditions on the underlying
regression function $f_0$, as our error bound in Theorem \ref{thm:error_bd_j}
will involve the error of the closest additive function to $f_0$, whose
components lie in the given function spaces \smash{$\cS_j$}, $j=1,\ldots,d$.  

Next, we present our assumptions on the regularizer $J$.
We write \smash{$\|\cdot\|_{Z_n}$} for the empirical
norm defined over a set of univariate 
points $Z_n=\{z^1,\ldots,z^n\} \subseteq [0,1]$, i.e., 
\smash{$\|g\|_{Z_n}^2 = \frac{1}{n}\sum_{i=1}^n g^2(z^i)$}.

\begin{assumption}{C}{1}{{\bf Seminorm regularizer, null space of
      polynomials}}  
\label{as:poly_null}
The regularizer $J$ is a seminorm, and its domain is contained in
the space of $k$ times weakly differentiable functions, for an integer $k \geq 
0$. Furthermore, its null space contains all $k$th order polynomials, i.e., 
$J(g)=0$ for all $g(t)=t^\ell$, $\ell=0,\ldots,k$.     
\end{assumption}

\begin{assumption}{C}{2}{{\bf Relative boundedness of derivatives}}  
\label{as:deriv_bd}
There is a constant $L>0$ such that
\smash{$\esssup_{t \in [0,1]} g^{(k)}(t) - \essinf_{t \in [0,1]}
  g^{(k)}(t) \leq L$} for $g \in
B_J(1)$ (with $g^{(k)}$ the $k$th weak derivative of $g$).   
\end{assumption}

\begin{assumption}{C}{3}{{\bf Entropy bound}} 
\label{as:entropy_bd}
There are constants $0<w<2$ and $K>0$ such that          
$$
\sup_{Z_n = \{z^1,\ldots,z^n\} \subseteq [0,1]} \,
\log N\big(\delta, \|\cdot\|_{Z_n}, B_J(1) \cap B_\infty(1)\big)  
\leq K \delta^{-w}.
$$
\end{assumption}

Assumptions \ref{as:poly_null}, \ref{as:deriv_bd}, \ref{as:entropy_bd} on the
regularizer $J$ are not strong, and are satisfied by various common
regularizers, e.g., \smash{$J(g)=[\int_0^1 (g^{(k+1)}(t))^2\, dt]^{1/2}$} or    
\smash{$J(g)=\TV(g^{(k)})$} (the latter studied shortly).

We now state our main result in the fixed $d$ case, which is proved in
Appendix \ref{app:error_bd_j_prelim}, \ref{app:error_bd_j}.  

\begin{theorem}
\label{thm:error_bd_j}
Assume \ref{as:x_dist}, \ref{as:y_dist} on the data distribution, 
and assume \ref{as:poly_null}, \ref{as:deriv_bd}, \ref{as:entropy_bd}
on the seminorm $J$.  
Also, assume that the dimension $d$ of the input space is fixed. Let 
$C_n \geq 1$ be an arbitrary sequence.  
There exist constants $c_1,c_2,c_3,n_0>0$, that depend 
only on $d,\sigma,k,L,K,w$, such that for all $c \geq c_1$, $n \geq n_0$,
and tuning parameter values \smash{$\lambda \geq c n^{w/(2+w)}
  C_n^{-(2-w)/(2+w)}$}, any solution in \eqref{eq:add_func_j} 
satisfies   
\begin{equation}
\label{eq:error_bd_j}
\bigg\| \sum_{j=1}^d \hf_j - f_0 \bigg\|_n^2 \leq 
\bigg\| \sum_{j=1}^d \tf_j - f_0 \bigg\|_n^2 + \frac{6\lambda}{n}
  \max\bigg\{ C_n,\sum_{j=1}^d J(\tf_j) \bigg\},
\end{equation}
with probability at least \smash{$1-\exp(-c_2 c)-\exp(-c_3\sqrt{n})$},
simultaneously over all \smash{$\tf=\sum_{j=1}^d \tf_j$}, feasible for the
problem \eqref{eq:add_func_j}, such that \smash{$\|\tf-f_0\|_n \leq \max\{C_n,
  \sum_{j=1}^d J(\tf_j)\}$}.  
\end{theorem}

\begin{remark}[\textbf{Error bound for additive, $J$-smooth $f_0$}]    
Assume \smash{$f_0=\sum_{j=1}^d f_{0j}$}, where 
\smash{$f_{0j} \in \cS_j$}, $j=1,\ldots,d$, and 
\smash{$\sum_{j=1}^d J(f_{0j}) \leq C_n$}.
Letting \smash{$\tf=f_0$}, the approximation error term in \eqref{eq:error_bd_j} 
(the first term on the right-hand side) is zero, and for \smash{$\lambda = c
  n^{w/(2+w)} C_n^{-(2-w)/(2+w)}$}, the result in the theorem reads  
\begin{equation}
\label{eq:error_bd_j_add}
\bigg\| \sum_{j=1}^d \hf_j - \sum_{j=1}^d f_{0j} \bigg\|_n^2 \leq 
6 c n^{-2/(2+w)} C_n^{2w/(2+w)},  
\end{equation}
with probability at least \smash{$1-\exp(-c_2 c)-\exp(-c_3\sqrt{n})$}. 
As we will see in the minimax lower bound in Theorem \ref{thm:lower_bd_j}
(plugging in $c_n=C_n/d$, and taking $d$ to be a constant), the rate
\smash{$n^{-2/(2+w)} C_n^{2w/(2+w)}$} is optimal for such a class of functions.   
\end{remark}

\begin{remark}[\textbf{Distance to best additive, $J$-smooth approximation of
    $f_0$}]   
The arguments used to establish the oracle-type inequality 
\eqref{eq:error_bd_j} also imply a result on the empirical
norm error between \smash{$\hf$} and the best additive
approximation of $f_0$.  To be precise, let
\smash{$(f^{\mathrm{best}}_1,\ldots,f^{\mathrm{best}}_d)$} denote a solution  
in the population-level problem 
\begin{equation}
\label{eq:add_func_j_best}
\begin{alignedat}{2}
&\ccol{\min_{f_j \in \cS_j, \, j=1,\ldots,d}}  \quad   
&&\half \sum_{i=1}^n \bigg(f_0(X^i) -
\sum_{j=1}^d f_j(X^i_j) \bigg)^2 + \frac{\lambda}{2} \sum_{j=1}^d J(f_j) \\    
&\ccol{\st} \quad && \sum_{i=1}^n f_j(X^i_j) = 0, 
\;\;\;  j=1,\ldots,d.
\end{alignedat}
\end{equation}
We note that \eqref{eq:add_func_j_best} has ``half'' of the regularization of
problem \eqref{eq:add_func_j}, as it uses a penalty parameter of $\lambda/2$
versus $\lambda$.
We can think of \smash{$f^{\mathrm{best}}=\sum_{j=1}^d 
  f^{\mathrm{best}}_j$} as the best additive, $J$-smooth approximation of 
$f_0$, where $\lambda$ as usual controls the level of smoothness.
The following is a consequence of the proof of Theorem \ref{thm:error_bd_j},
verified in Appendix \ref{app:error_bd_j_best}: assume that
\smash{$\|f^{\mathrm{best}}-f_0\|_n \leq
  \max\{C_n,\sum_{j=1}^d J(f^{\mathrm{best}}_j)\}$} almost surely (with
respect to $Q$), for sufficiently large $\lambda$; then any 
solution in \eqref{eq:add_func_j} 
satisfies for all $c \geq c_1$, $n \geq n_0$, and \smash{$\lambda
  \geq c n^{w/(2+w)} C_n^{-(2-w)/(2+w)}$},   
\begin{equation}
\label{eq:error_bd_j_best}
\bigg\| \sum_{j=1}^d \hf_j - \sum_{j=1}^d f^{\mathrm{best}}_j 
\bigg\|_n^2 \leq \frac{6\lambda}{n}
  \max\bigg\{ C_n,\sum_{j=1}^d J(f^{\mathrm{best}}_j) \bigg\},  
\end{equation}
with probability at least \smash{$1-\exp(-c_2 c)-\exp(-c_3\sqrt{n})$}, where 
as before $c_1,c_2,c_3,n_0>0$ are constants that depend  
only on $d,\sigma,k,L,K,w$.  Notably, the right-hand side in the bound
\eqref{eq:error_bd_j_best} does not depend on the approximation error; in
particular, we do not even require \smash{$\|f^{\mathrm{best}}-f_0\|_n$} to
converge to zero.  This is analogous to classical results from
\citet{stone1985additive}. 
\end{remark}

We examine a special case of the generic problem \eqref{eq:add_func_j} when the  
regularizer is \smash{$J(g)=\TV(g^{(k)})$}, and derive 
implications of the above Theorem \ref{thm:error_bd_j} for additive locally
regression adaptive splines and additive trend filtering, corresponding to   
different choices of the function classes $\cS_j$, $j=1,\ldots,d$ in
\eqref{eq:add_func_j}.  We must introduce an additional (weak) assumption on the
input distribution, for the results on restricted locally adaptive regression 
splines and trend filtering.

\begin{assumption}{A}{2}{\textbf{Bounded input density}}  
\label{as:x_dens_bd}
The density of the input distribution $Q$ is bounded below by a constant
$b_0>0$. 
\end{assumption}

Here is our result for additive locally adaptive splines and additive trend
filtering. The proof is given in Appendix \ref{app:error_bd_tv_prelim}, 
\ref{app:error_bd_tv}.  

\begin{corollary}
\label{cor:error_bd_tv}
Assume \ref{as:x_dist}, \ref{as:y_dist} on the data distribution.  Also, assume
that the dimension $d$ of the input space is fixed, and that the underlying
regression function is additive, \smash{$f_0=\sum_{j=1}^d f_{0j}$}, where the
components $f_{0j}$, $j=1,\ldots,d$ are $k$ times weakly differentiable, such
that \smash{$\sum_{j=1}^d \TV(f_{0j}^{(k)}) \leq C_n$} for a sequence $C_n \geq
1$.  For \smash{$J(g)=\TV(g^{(k)})$}, Assumptions \ref{as:poly_null},
\ref{as:deriv_bd}, \ref{as:entropy_bd} hold with $L=1$ and $w=1/(k+1)$.
Furthermore, the following is true of the estimator defined by problem
\eqref{eq:add_func_j}.

\begin{itemize}
\item[(a)]
Let $\cS_j$ be the set of all $k$ times weakly
differentiable functions, for each $j=1,\ldots,d$.  There are
constants $c_1,c_2,c_3,n_0 > 0$, depending only on
$d,\sigma,k$, such that for all $c \geq c_1$ and $n \geq n_0$,
any solution in the additive locally adaptive regression spline
problem \eqref{eq:add_func_j}, with tuning parameter value  
\smash{$\lambda = c n^{1/(2k+3)} C_n^{-(2k+1)/(2k+3)}$}, satisfies     
\begin{equation}
\label{eq:error_bd_tv}
\bigg\| \sum_{j=1}^d \hf_j - \sum_{j=1}^d f_{0j} \bigg\|_n^2 
\leq  c n^{-(2k+2)/(2k+3)} C_n^{2/(2k+3)},  
\end{equation}
with probability at least \smash{$1-\exp(-c_2 c) -\exp(-c_3\sqrt{n})$}. 

\item[(b)] Let $\cS_j=\cG_j$, the set of $k$th degree splines with knots in
the set $T_j$ in \eqref{eq:tj}, for $j=1,\ldots,d$, and assume
\ref{as:x_dens_bd} on the input density. Then there are constants
$c_1,c_2,c_3,n_0>0$, that depend only on $d,b_0,\sigma,k$, such that for all $c  
\geq c_1$ and \smash{$n(\log n)^{-(1+1/k)} \geq n_0 C_n^{(2k+2)/(2k^2+2k-1)}$},
any solution in the restricted additive locally adaptive spline
problem \eqref{eq:add_func_j}, with \smash{$\lambda = c 
  n^{1/(2k+3)} C_n^{-(2k+1)/(2k+3)}$}, satisfies the same result in
\eqref{eq:error_bd_tv}, with probability at least $1-\exp(-c_2 c) - c_3/n$.

\item[(c)] Let $\cS_j=\cH_j$, the set of $k$th degree falling factorial
functions defined over $\Xj$ (the $j$th dimension of inputs), for   
$j=1,\ldots,d$, and assume \ref{as:x_dens_bd}. Then there exist constants
$c_1,c_2,c_3,n_0>0$, that depend only on $d,b_0,\sigma,k$, such that for all $c 
\geq c_1$ and \smash{$n(\log n)^{-(2k+3)} \geq n_0 C_n^{4k+4}$}, any solution in
the additive trend filtering problem \eqref{eq:add_func_j}, with \smash{$\lambda
  = c n^{1/(2k+3)} C_n^{-(2k+1)/(2k+3)}$}, satisfies \eqref{eq:error_bd_tv},
with probability at least $1-\exp(-c_2 c) - c_3/n$.
\end{itemize}
\end{corollary}

\begin{remark}[\textbf{Spline and falling factorial approximants}]  
\label{rem:spline_approx} 
For part (a) of the corollary, the approximation
error (the first term on the right-hand side) in
\eqref{eq:error_bd_j_add} is zero by definition, and we need only   
verify Assumptions \ref{as:poly_null}, \ref{as:deriv_bd},
\ref{as:entropy_bd} for the regularizer \smash{$J(g)=\TV(g^{(k)})$}.
Parts (b) and (c) require control over the approximation error,
because the underlying regression function \smash{$f_0=\sum_{j=1}^d
  f_{0j}$} need not have components that lie in the chosen function spaces 
$\cS_j$, $j=1,\ldots,d$.  To be clear: for $k=0$ or $k=1$, as 
discussed in Section \ref{sec:splines}, all three problems considered in parts
(a), (b), (c) are equivalent; hence parts (b) and (c) really only concern the
case $k \geq 2$.   For both of these parts, we control the approximation error
by controlling the univariate approximation error and then applying the triangle   
inequality.  For part (b), we use a special spline quasi-interpolant from
Proposition 7 in \citet{mammen1997locally} (who in turn construct this
using results from \citet{deboor1978practical}); for part (c), we develop a new
falling factorial approximant that may be of independent interest.
\end{remark}

\subsection{Error bounds for a growing dimension $d$}
\label{sec:error_bd_hd}

In this subsection, we allow the input dimension $d$ to grow with the sample
size $n$.  To keep our analysis as clean as possible, we consider a constrained
version of the problem \eqref{eq:add_func_j}, namely   
\begin{equation}
\label{eq:add_func_j_bd}
\begin{alignedat}{2}
&\ccol{\min_{f_j \in \cS_j, \, j=1,\ldots,d}}  \quad   
&&\half \sum_{i=1}^n \bigg(Y^i - \bar{Y} - 
\sum_{j=1}^d f_j(X^i_j) \bigg)^2 \\    
&\ccol{\st} \quad && \sum_{i=1}^n f_j(X^i_j) = 0, \;
J(f_j) \leq \delta, \;\;\;  j=1,\ldots,d, 
\end{alignedat}
\end{equation}
for a tuning parameter $\delta>0$.  (The penalized problem \eqref{eq:add_func_j} 
can also be analyzed in the setting of growing $d$, but we find that the analysis 
is messier and requires more assumptions in order to obtain the same results.) 
Instead of \ref{as:x_dist}, we now use the following assumption in the input
distribution.   

\begin{assumption}{A}{3}{\textbf{Product input distribution}}
\label{as:x_dist_prod}
The input points $X^i$, $i=1,\ldots,n$ are i.i.d.\ from a continuous
distribution $Q$ supported on $[0,1]^d$, that decomposes as 
$Q=Q_1 \times \cdots \times Q_d$, where the density of each $Q_j$ is lower and 
upper bounded by constants $b_1,b_2>0$, for $j=1,\ldots,d$.
\end{assumption}

Assumption \ref{as:x_dist_prod} is fairly restrictive, since it requires the
input distribution $Q$ to be independent across dimensions of the input
space. The reason we use this assumption: when 
$Q=Q_1 \times \cdots \times Q_d$, 
additive functions enjoy a key decomposability
property in terms of the (squared) $L_2$ norm defined with respect to $Q$. 
In particular, if \smash{$m=\sum_{j=1}^d m_j$} has components with
$L_2$ mean zero, denoted by \smash{$\bar{m}_j = \int_0^1 m_j(x_j) \,
  dQ_j(x_j)=0$}, $j=1,\ldots,d$, then we have 
\begin{equation}
\label{eq:L2_decomp}
\bigg\|\sum_{j=1}^d m_j \bigg\|_2^2 = 
\sum_{j=1}^d \|m_j\|_2^2. 
\end{equation}
This is explained by the fact that each pair of components $m_j$,
$m_\ell$ with $j \not= \ell$ are orthogonal with respect to the $L_2$
inner product, since  
$$
\langle m_j, m_\ell \rangle_2
= \int_{[0,1]^2} m_j(x_j) m_\ell(x_\ell)\, dQ_j(x_j) \, dQ_\ell(x_\ell)
= \bar{m}_j \bar{m}_\ell = 0.   
$$
The above orthogonality, and thus the decomposability property in
\eqref{eq:L2_decomp}, is only true because of the product form 
$Q=Q_1 \times \cdots \times Q_d$. Such decomposability is not
generally possible with the empirical norm (the inner products 
between components do not vanish even if all empirical means are zero).
In the proof of Theorem \ref{thm:error_bd_j_hd}, we move from considering the 
empirical norm of the error vector to the $L_2$ norm, in order to leverage the
property in \eqref{eq:L2_decomp}, which eventually leads to an error rate that
has a linear dependence on the dimension $d$.  In the absence of $L_2$
decomposability, the same error rate can be achieved with a weaker
incoherence bound, as in \eqref{eq:L2_incoherence}; see Remark
\ref{rem:L2_decomp} after the theorem.   

We now state our main result in the growing $d$ case, whose proof
is in Appendix \ref{app:error_bd_j_hd_prelim},
\ref{app:error_bd_j_hd}.  

\begin{theorem}
\label{thm:error_bd_j_hd}
Assume \ref{as:x_dist_prod}, \ref{as:y_dist} on the data
distribution, and assume \ref{as:poly_null}, \ref{as:deriv_bd},
\ref{as:entropy_bd} on the seminorm $J$.  Let $\delta \geq 1$ be arbitrary. 
There are constants $c_1,c_2,c_3,n_0>0$, that depend
only on $b_1,b_2,\sigma,k,L,K,w$, such that for all $c \geq c_1$ and
\smash{$n \geq n_0 (d\delta)^{1+w/2}$}, any solution in \eqref{eq:add_func_j_bd}
satisfies both     
\begin{gather}
\label{eq:error_bd_j_hd_empirical}
\bigg\| \sum_{j=1}^d \hf_j - f_0 \bigg\|_n^2 \leq 
\bigg\| \sum_{j=1}^d \tf_j - f_0 \bigg\|_n^2 + 
c d n^{-2/(2+w)} \delta, \\
\label{eq:error_bd_j_hd_L2}
\bigg\| \sum_{j=1}^d \hf_j - f_0 \bigg\|_2^2 \leq 
2\bigg\| \sum_{j=1}^d \tf_j - f_0 \bigg\|_2^2 +
24 \bigg\| \sum_{j=1}^d \tf_j - f_0 \bigg\|_n^2 +  
c d n^{-2/(2+w)} \delta^2,
\end{gather}
with probability at least \smash{$1-\exp(-c_2 c) - c_3/n$},
simultaneously over all functions \smash{$\tf=\sum_{j=1}^d \tf_j$}, feasible for
the problem \eqref{eq:add_func_j_bd}.
\end{theorem}

\begin{remark}[\textbf{Error bound for additive, $J$-smooth $f_0$}]   
\label{rem:error_bd_j_hd_add} 
Assume \smash{$f_0=\sum_{j=1}^d f_{0j}$}, where 
\smash{$f_{0j} \in \cS_j$} and 
\smash{$J(f_{0j}) \leq c_n$}, $j=1,\ldots,d$, for a sequence $c_n \geq 1$. 
Letting \smash{$\tf=f_0$}, and $\delta=c_n$, the results in  
\eqref{eq:error_bd_j_hd_empirical}, \eqref{eq:error_bd_j_hd_L2} translate to 
\begin{equation}
\label{eq:error_bd_j_hd_add}
\bigg\| \sum_{j=1}^d \hf_j - \sum_{j=1}^d f_{0j} \bigg\|_n^2 \leq
c d n^{-2/(2+w)} c_n, \;\;\;\text{and}\;\;\;
\bigg\| \sum_{j=1}^d \hf_j - \sum_{j=1}^d f_{0j} \bigg\|_2^2 \leq 
c d n^{-2/(2+w)} c_n^2,
\end{equation}
with probability at least \smash{$1-\exp(-c_2 c)-c_3/n$},
provided that \smash{$n \geq n_0 (d c_n)^{1+w/2}$}. From the minimax
lower bound in Theorem \ref{thm:lower_bd_j}, we can see that the optimal rate 
for such a class of functions is in fact \smash{$d n^{-2/(2+w)} c_n^{2w/(2+w)}$},
which reveals that the rates in \eqref{eq:error_bd_j_hd_add} are tight when
$c_n$ is a constant, but not when $c_n$ grows with $n$.  It is worth noting that
the dependence of the bounds on $c_n$ in Theorem \ref{thm:error_bd_j_hd} (and  
hence in \eqref{eq:error_bd_j_hd_add}) can be improved to have the optimal
scaling of \smash{$c_n^{2w/(2+w)}$} by assuming that $f_0$ is sup norm bounded,
and additionally placing a sup norm bound on the components in
\eqref{eq:add_func_j_bd}.  This feels like an unnecessary restriction, so we
prefer to present results without it, as in Theorem \ref{thm:error_bd_j_hd} 
(and \eqref{eq:error_bd_j_hd_add}).
\end{remark}

\begin{remark}[\textbf{Distance to best additive, $J$-smooth approximation of
    $f_0$}]  A consequence of the proof of
  \eqref{eq:error_bd_j_hd_empirical} is a bound on the empirical
norm error between \smash{$\hf$} and the best additive
approximation of $f_0$.  To be precise, let
\smash{$f^{\mathrm{best}}=\sum_{j=1}^d f^{\mathrm{best}}_j$} minimize 
\smash{$\|\sum_{j=1}^d \tf_j - f_0\|_n^2$} over all
additive functions \smash{$\tf=\sum_{j=1}^d \tf_j$} feasible for problem
\eqref{eq:add_func_j_bd}.  Then following directly from \eqref{eq:basic_ineq_9}
in the proof of Theorem \ref{thm:error_bd_j_hd}, we have for all $c \geq c_1$
and \smash{$n \geq n_0 (d \delta)^{1+w/2}$}, 
\begin{equation}
\label{eq:error_bd_j_hd_best}
\bigg\| \sum_{j=1}^d \hf_j - \sum_{j=1}^d f^{\mathrm{best}}_j 
\bigg\|_n^2 \leq c d n^{-2/(2+w)} \delta,
\end{equation}
with probability at least \smash{$1-\exp(-c_2 c)-c_3/n$},
where again $c_1,c_2,c_3,n_0>0$ are constants that depend on 
$b_1,b_2,\sigma,k,L,K,w$.   
Just as we saw in fixed $d$ case, the right-hand side in
\eqref{eq:error_bd_j_hd_best} does not depend on the approximation error
\smash{$\|f^{\mathrm{best}} - f_0\|_n$}, which is analogous to classical results
from \citet{stone1985additive}. 
\end{remark}

\begin{remark}[\textbf{$L_2$ decomposability and incoherence}] 
\label{rem:L2_decomp}
The decomposability property in \eqref{eq:L2_decomp} is critical in 
obtaining the sharp (linear) dependence on $d$ in the error
rates \eqref{eq:error_bd_j_hd_empirical}, \eqref{eq:error_bd_j_hd_L2}. However,
it is worth noting that all that is needed in the proof is in fact a lower bound
of the form  
\begin{equation}
\label{eq:L2_incoherence}
\bigg\|\sum_{j=1}^d m_j \bigg\|_2^2 \geq 
\phi_0 \sum_{j=1}^d \|m_j\|_2^2,
\end{equation}
for a constant $\phi_0>0$, rather than an exact equality, as in
\eqref{eq:L2_decomp}.  The condition 
\eqref{eq:L2_incoherence} is an incoherence condition
that can hold for nonproduct distributions $Q$, over an appropriate
class of functions (additive functions with smooth
components), provided that the correlations between components of $Q$
are not too large.  See \citet{meier2009high,vandegeer2014uniform} 
for similar incoherence conditions.  
\end{remark}

Next we present our results for additive locally adaptive regression splines
(both unresricted and restricted variants) and additive trend filtering.
The proof is in Appendix \ref{app:error_bd_tv_hd}.  

\begin{corollary}
\label{cor:error_bd_tv_hd}
Assume \ref{as:x_dist_prod}, \ref{as:y_dist} on the
data distribution.  Also, assume that the underlying regression
function is additive, \smash{$f_0=\sum_{j=1}^d f_{0j}$}, where the
components $f_{0j}$, $j=1,\ldots,d$ are $k$ times weakly differentiable, such 
that \smash{$\TV(f_{0j}^{(k)}) \leq c_n$}, $j=1,\ldots,d$, for a sequence $c_n
\geq 1$. Then for \smash{$J(g)=\TV(g^{(k)})$}, the following is true of the
estimator defined by problem \eqref{eq:add_func_j_bd}.      

\begin{itemize}
\item[(a)]
Let $\cS_j$ be the space of all $k$ times weakly
differentiable functions, for each $j=1,\ldots,d$. There exist
constants $c_1,c_2,c_3,n_0 > 0$, that depend only on 
$b_1,b_2,\sigma,k$, such that for all $c \geq c_1$ and \smash{$n \geq  
  n_0 (dc_n)^{(2k+3)/(2k+2)}$}, any solution in the constrained-form additive
locally adaptive spline problem \eqref{eq:add_func_j_bd}, with tuning 
parameter $\delta=c_n$, satisfies    
\begin{equation}
\label{eq:error_bd_tv_hd}
\bigg\| \sum_{j=1}^d \hf_j - \sum_{j=1}^d f_{0j} \bigg\|_n^2 \leq
c d n^{-(2k+2)/(2k+3)} c_n, \;\;\;\text{and}\;\;\;
\bigg\| \sum_{j=1}^d \hf_j - \sum_{j=1}^d f_{0j} \bigg\|_2^2 \leq 
c d n^{-(2k+2)/(2k+3)} c_n^2,
\end{equation}
with probability at least \smash{$1-\exp(-c_2 c) -c_3/n$}.

\item[(b)] Let $\cS_j=\cG_j$, the set of $k$th degree splines with knots in
the set $T_j$ in \eqref{eq:tj}, for $j=1,\ldots,d$. There exist constants
$c_1,c_2,c_3,n_0>0$, that depend only on $b_1,b_2,\sigma,k$, such that for  
$c \geq c_1$ and \smash{$n \geq (dc_n)^{(2k+3)/(2k+2)}$}, any solution in the 
constrained-form restricted additive locally adaptive spline
problem \eqref{eq:add_func_j_bd}, with tuning parameter 
$\delta=a_k c_n$, where $a_k \geq 1$ is a constant that depends only on $k$, 
satisfies \eqref{eq:error_bd_tv_hd}, with probability at least
\smash{$1-\exp(-c_2 c) -c_3d/n$}.   

\item[(c)] Let $\cS_j=\cH_j$, the set of $k$th degree falling factorial
functions defined over $\Xj$ (the $j$th dimension of input points), for    
$j=1,\ldots,d$. Then there are constants $c_1,c_2,c_3,n_0>0$, depending only 
on $b_1,b_2,\sigma,k$, such that for all $c \geq c_1$ and \smash{$n \geq n_0 
  (dc_n)^{(2k+3)/(2k+2)}$}, any solution in the constrained-form additive
trend filtering problem \eqref{eq:add_func_j_bd}, with tuning parameter
$\delta=a_k c_n$, where $a_k \geq 1$ is a constant depending only on $k$,
satisfies \eqref{eq:error_bd_tv_hd}, with probability at least
\smash{$1-\exp(-c_2 c) -c_3d/n$}.    
\end{itemize}
\end{corollary}

\subsection{Minimax lower bounds}
\label{sec:lower_bd}

We consider minimax lower bounds for estimation over the class of
additive functions whose components are smooth with respect to the
seminorm $J$.  We allow the dimension $d$ to grow with $n$.  As for the data 
distribution, we will use the following assumptions in place of \ref{as:x_dist},
\ref{as:x_dens_bd}, \ref{as:x_dist_prod}, \ref{as:y_dist}.  

\begin{assumption}{A}{4}{{\bf Uniform input distribution}}
\label{as:x_dist_unif}
The inputs $X^i$, $i=1,\ldots,n$ are i.i.d.\ from the uniform distribution on 
$[0,1]^d$.   
\end{assumption}

\begin{assumption}{B}{2}{{\bf Additive model, Gaussian errors}}  
\label{as:y_dist_gauss}
The responses $Y^i$, $i=1,\ldots,n$ follow 
$$
Y^i = \mu + \sum_{j=1}^d f_{0j} (X^i_j) + \epsilon^i, \;\;\;
i=1,\ldots,n,  
$$
with mean $\mu \in \R$, where \smash{$\int_{[0,1]^d} f_0(x) \,
  dx = 0$} for identifiability.
The errors $\epsilon^i$, $i=1,\ldots,n$ are i.i.d.\ $N(0,\sigma^2)$,
for some constant $\sigma>0$.   The errors and input points are
independent. 
\end{assumption}

For the regularizer $J$, assumed to satisfy Assumptions
\ref{as:poly_null}, \ref{as:deriv_bd}, we will replace
Assumption \ref{as:entropy_bd} by the following assumption, on the
log packing and log covering (entropy) numbers. 

\begin{assumption}{C}{4}{{\bf Matching packing and 
      covering number bounds}}    
\label{as:packing_bd}
There exist constants $0<w<2$ and $K_1,K_2>0$ such that 
\begin{align*}
\log M \big(\delta, \|\cdot\|_2, B_J(1) \cap B_\infty(1)\big) 
&\geq K_1 \delta^{-w}, \\
\log N \big(\delta, \|\cdot\|_2, B_J(1) \cap B_\infty(1)\big) 
&\leq K_2 \delta^{-w}.
\end{align*} 
(To be clear, here $\|\cdot\|_2$ is the $L_2$ norm defined with 
respect to the uniform distribution on $[0,1]$.)
\end{assumption}

Let us introduce the notation
$$
B_J^d(\delta) = \bigg \{ \sum_{j=1}^d f_j : J(f_j) \leq \delta, \; 
 j=1,\ldots,d \bigg\}, 
$$
Now we state our main minimax lower bound.  The proof is given in Appendix 
\ref{app:lower_bd_j_prelim}, \ref{app:lower_bd_j}.   

\begin{theorem}
\label{thm:lower_bd_j}
Assume \ref{as:x_dist_unif}, \ref{as:y_dist_gauss} on the data distribution,
and \ref{as:poly_null}, \ref{as:deriv_bd}, \ref{as:packing_bd} on the seminorm
$J$.  Then there exist constants $c_0,n_0 > 0$, that depend only on 
$\sigma, k, L, K_1, K_2, w$, such that for all $c_n \geq 1$ and 
\smash{$n \geq n_0 d^{1+w/2} c_n^w$}, we have  
\begin{equation}
\label{eq:lower_bd_j}
\inf_{\hf} \, \sup_{f_0 \in B_J^d(c_n)} \,
\E \| \hf - f_0 \|_2^2 \geq c_0 d n^{-2/(2+w)} c_n^{2w/(2+w)}.
\end{equation}
\end{theorem}

When we choose \smash{$J(g)=\TV(g^{(k)})$} as our regularizer, the
additive function class \smash{$B_J^d(\delta)$} becomes
$$
\cF_k^d(\delta) = \bigg \{ \sum_{j=1}^d f_j : \TV(f^{(k)}_j) \leq \delta, \;
j=1,\ldots,d \bigg\}, 
$$
and Theorem \ref{thm:lower_bd_j} implies the following result, whose
proof is in Appendix \ref{app:lower_bd_tv}.

\begin{corollary}
\label{cor:lower_bd_tv}
Assume \ref{as:x_dist_unif}, \ref{as:y_dist_gauss} on the data distribution.
Assume further that $f_{0j}$, $j=1,\ldots,d$ are $k$ 
times weakly differentiable.  Then there are constants $c_0,n_0 > 0$, that
depend only on $\sigma, k$, such that for all $c_n \geq 1$ and
and \smash{$n \geq n_0 d^{(2k+3)/(2k+2)} c_n^{1/(k+1)}$},
\begin{equation}
\label{eq:lower_bd_tv}
\inf_{\hf} \, \sup_{f_0 \in \cF_k^d(c_n)} \, 
\E \| \hf - f_0 \|_2^2 \geq c_0 d n^{-(2k+2)/(2k+3)} c_n^{2/(2k+3)}.
\end{equation}
\end{corollary}

\begin{remark}[{\bf Optimality for a fixed dimension $d$}] For a fixed 
$d$, the estimator defined by \eqref{eq:add_func_j} is minimax rate optimal over
the class of additive functions $f_0$ such that \smash{$\sum_{j=1}^d
  J(f_{0j}) \leq C_n$}.  To see this, note that such a class of functions
contains \smash{$B_J^d(C_n/d)$}, therefore plugging $c_n=C_n/d$ into
the right-hand side in \eqref{eq:lower_bd_j} yields a lower bound rate of  
\smash{$n^{-2/(2+w)} C_n^{2w/(2+w)}$}, which matches the upper bound rate in
\eqref{eq:error_bd_j_add}. 

Furthermore, when \smash{$J(g)=\TV(g^{(k)})$}, the lower bound rate given by
plugging $c_n=C_n/d$ into the right-hand side in \eqref{eq:lower_bd_tv} is  
\smash{$n^{-(2k+2)/(2k+3)} C_n^{2/(2k+3)}$}, matching the upper 
bound rate in \eqref{eq:error_bd_tv}.  Hence additive
locally adaptive regression splines, restricted additive locally adaptive
regression splines, and additive trend filtering all achieve the minimax rate  
over the space of additive functions $f_0$ such that \smash{$\sum_{j=1}^d
  \TV(f_{0j}^{(k)}) \leq C_n$}.        
\end{remark}

\begin{remark}[{\bf Optimality for a growing dimension $d$}]
For growing $d$, the estimator defined by \eqref{eq:add_func_j_bd} 
is minimax rate optimal over the class of additive functions
$f_0$ such that \smash{$J(f_{0j}) \leq c$}, $j=1,\ldots,d$, where $c>0$ is a
constant.  This is verified by noting that the lower bound rate of
\smash{$dn^{-2/(2+w)}$} in \eqref{eq:lower_bd_j} matches the upper bound rates
in \eqref{eq:error_bd_j_hd_empirical}, \eqref{eq:error_bd_j_hd_L2}. 

When \smash{$J(g)=\TV(g^{(k)})$}, and again, $c_n=c$ (a constant), 
the lower bound rate of \smash{$d n^{-(2k+2)/(2k+3)}$} in \eqref{eq:lower_bd_tv}  
matches the upper bound rates in \eqref{eq:error_bd_tv_hd}.
 Thus additive locally adaptive regression splines,
restricted additive locally adaptive regression splines, and additive trend
filtering all attain the minimax rate over the space of additive 
functions $f_0$ with \smash{$\TV(f_{0j}^{(k)}) \leq c$}, $j=1,\ldots,d$.

For growing $c_n$, we note that the upper bounds in \eqref{eq:error_bd_j_hd_add}
and \eqref{eq:error_bd_tv_hd} have an inflated dependence on $c_n$, compared to
\eqref{eq:lower_bd_j} and \eqref{eq:lower_bd_tv}.  It turns out that the latter 
(lower bounds) are tight, and the former (upper bounds) are loose.  The upper
bounds can be tightened under further boundedness assumptions (see Remark 
\ref{rem:error_bd_j_hd_add}). 
\end{remark}

\begin{remark}[{\bf Suboptimality of additive linear smoothers}]
Seminal theory from \citet{donoho1998minimax} on minimax linear
rates over Besov spaces shows that, under Assumption
\ref{as:y_dist_gauss}, and with the inputs $X^i$, $i=1,\ldots,n$ being now 
nonrandom and occurring over the regular $d$-dimensional lattice  
$\{1/N,2/N,\ldots,1\}^d \subseteq [0,1]^d$ with $N=n^{1/d}$, we have 
\begin{equation}
\label{eq:lower_bd_tv_linear}
\inf_{\hf \, \text{\rm additive linear}} \, \sup_{f_0 \in
  \cF_k^d(c_n)} \, \E \| \hf - f_0 \|_2^2 \geq  
c_0 d n^{-(2k+1)/(2k+2)} c_n^{2/(2k+2)} , 
\end{equation}
for all $n \geq n_0$, where $c_0,n_0>0$ are constants, depending only on
$\sigma,k$.  On the left-hand side in \eqref{eq:lower_bd_tv_linear} the infimum
is taken over all additive linear smoothers, i.e., estimators
\smash{$\hf=\sum_{j=1}^d \hf_j$} such that each component \smash{$\hf_j$} is a
linear smoother, for $j=1,\ldots,d$.  The additive linear smoother lower bound
\eqref{eq:lower_bd_tv_linear} is verified in Appendix \ref{app:lower_bd_tv_linear}.  

For a fixed $d$, we can see that all additive linear smoothers---e.g., additive 
smoothing splines, additive kernel smoothing estimators, additive RKHS 
estimators, etc.---are suboptimal over the class of additive functions
$f_0$ with \smash{$\sum_{j=1}^d \TV(f_{0j}^{(k)}) \leq C_n$}, as the
optimal linear rate in \eqref{eq:lower_bd_tv_linear} (set $c_n=C_n/d$) is
\smash{$n^{-(2k+1)/(2k+2)} C_n^{2/(2k+2)}$}, slower than the optimal rate
\smash{$n^{-(2k+2)/(2k+3)} C_n^{2/(2k+2)}$} of additive locally adaptive
regression splines and additive trend filtering in \eqref{eq:error_bd_tv}.  

For growing $d$, and $c_n=c$ (a constant), we also see that additive linear
smoothers are suboptimal over the class of additive functions $f_0$ such that
\smash{$\TV(f_{0j}^{(k)}) \leq c$}, $j=1,\ldots,d$, as the optimal linear rate
in \eqref{eq:lower_bd_tv_linear} is \smash{$d n^{-(2k+1)/(2k+2)}$}, slower than
the optimal rate \smash{$d n^{-(2k+2)/(2k+3)}$} of additive locally adaptive
regression splines and additive trend filtering in \eqref{eq:error_bd_tv_hd}.

\end{remark}

\section{Backfitting and the dual}
\label{sec:comp}

We now examine computational approaches for the additive
trend filtering problem \eqref{eq:add_tf}.  This is a
convex optimization problem, and many standard approaches can be 
applied.  For its simplicity and its ubiquity in additive modeling, 
we focus on the backfitting algorithm in particular.

\subsection{Backfitting}

The backfitting approach for problem \eqref{eq:add_tf} is
described in Algorithm \ref{alg:backfit}. 
We write $\TF_\lambda(r, \Xj)$ for the univariate trend 
filtering fit, with a tuning parameter $\lambda>0$, 
to a response vector $r=(r^1,\ldots,r^n) \in \R^n$ over an
input vector $\Xj=(X^1_j,\ldots,X^n_j) \in \R^n$.  In words, the
algorithm cycles over $j=1,\ldots,d$, and at each step updates the
estimate for component $j$ by applying 
univariate trend filtering to the $j$th partial residual (i.e., the
current residual excluding component $j$). Centering in Step 2b part
(ii) is optional, because the fit $\TF_\lambda(r,\Xj)$ will have mean
zero whenever $r$ has mean zero, but centering can still be performed
for numerical stability. In general, the efficiency of backfitting 
hinges on the efficiency of the univariate 
smoother employed; to implement Algorithm
\ref{alg:backfit} in practice we can use fast interior
point methods \citep{kim2009trend} or fast operator splitting 
methods \citep{ramdas2016fast} for univariate trend filtering, 
both of which result in efficient empirical performance. 

\begin{algorithm}[tb]
\caption{Backfitting for additive trend filtering}
\label{alg:backfit}
Given responses $Y^i \in \R$ and input points $X^i \in \R^d$, $i=1,\ldots,n$.  
\begin{enumerate}
\item Set $t=0$ and initialize $\thetaj^{(0)}=0$, $j=1,\ldots,d$.
\item For $t=1,2,3,\ldots$ (until convergence): 
\begin{enumerate}
\item[a.] For $j=1,\ldots,d$:
\begin{enumerate}
\item[(i)] $\displaystyle \thetaj^{(t)} = \TF_\lambda \bigg(
Y - \bar{Y}\one - \sum_{\ell<j} \thetaj^{(t)} - 
\sum_{\ell>j} \thetaj^{(t-1)}, \; \Xj \bigg)$
\item[(ii)] (Optional) $\thetaj^{(t)}=\thetaj^{(t)} - \frac{1}{n}
  \one^T \thetaj^{(t)}$
\end{enumerate}
\end{enumerate}
\item Return $\hthetaj$, $j=1,\ldots,d$ (parameters
  $\thetaj^{(t)}$, $j=1,\ldots,d$ at convergence). 
\end{enumerate}
\end{algorithm}

Algorithm \ref{alg:backfit} is equivalent to block
coordinate descent (BCD), also called exact blockwise minimization, 
applied to problem \eqref{eq:add_tf} over the coordinate blocks 
$\thetaj$, $j=1,\ldots,d$.  A general treatment of BCD
is given in \citet{tseng2001convergence}, who shows that for a convex  
criterion that decomposes into smooth plus separable terms, as does 
that in \eqref{eq:add_tf}, all limit points of the sequence of iterates
produced by BCD are optimal solutions.  A recent wave of work from 
the optimization community gives refined convergence analyses for 
coordinate descent (or its variants) in particular settings.  We do
not pursue the implications of this work for our problem; our interest here is
primarily in developing a connection between BCD for problem \eqref{eq:add_tf}
and alternating projections in its dual problem \eqref{eq:add_tf_dual}, 
which is the topic of the next subsection. 

\subsection{Dual alternating projections}

Using the additive trend filtering problem \eqref{eq:add_tf} and its dual
\eqref{eq:add_tf_dual}, related by the transformation \eqref{eq:add_tf_pd}, we
see that for any dimension 
$j=1,\ldots,d$, the univariate trend filtering fit with response vector
$r=(r^1,\ldots,r^n)$ and input vector $\Xj=(X^1_j,\ldots,X^n_j)$ can be written
as     
\begin{equation}
\label{eq:tf_pd}
\TF_\lambda(r,\Xj) = (\Id-\Pi_{U_j})(r),  
\end{equation}
where $U_j=\{S_jD_j^T v_j : \|u\|_\infty \leq \lambda\}$, and recall, we 
abbreviate \smash{$D_j=\Dj$}.  (This follows from \eqref{eq:tf_pd}
specialized to the univariate trend filtering problem.) The backfitting 
approach in Algorithm \ref{alg:backfit} can be viewed (ignoring 
the optional centering step) as performing the updates, for $t=1,2,3,\ldots$, 
\begin{equation}
\label{eq:backfit_tf}
\thetaj^{(t)} = (\Id-\Pi_{U_j})
\bigg(Y - \bar{Y}\one - \sum_{\ell<j} \theta_\ell^{(t)} 
- \sum_{\ell>j} \theta_\ell^{(t-1)}\bigg), \;\;\; j=1,\ldots,d,
\end{equation}
or, reparametrized in terms of the primal-dual relationship
\smash{$u=Y-\bar{Y}\one-\sum_{j=1}^d \thetaj$} in
\eqref{eq:add_tf_pd}, 
\begin{equation}
\label{eq:admm_seq}
\begin{aligned}
u_0^{(t)} &= Y - \bar{Y}\one - \sum_{j=1}^d \thetaj^{(t-1)}, \\ 
u_j^{(t)} &= \Pi_{U_j} \big(u_{j-1}^{(t)} + \thetaj^{(t-1)}\big), 
\;\;\; j=1,\ldots,d, \\
\thetaj^{(t)} &= \thetaj^{(t-1)} + u_{j-1}^{(t)} - u_j^{(t)},
\;\;\;  j=1,\ldots,d.
\end{aligned}
\end{equation}
Thus the backfitting algorithm for \eqref{eq:add_tf},
as expressed above in \eqref{eq:admm_seq}, is seen to be a 
particular type of {\it alternating projections} method applied to the
dual problem \eqref{eq:add_tf_dual}, cycling 
through projections onto $U_j$, $j=1,\ldots,d$.
Interestingly, as opposed to the classical alternating
projections approach, which would repeatedly project the
current iterate \smash{$u_{j-1}^{(t)}$} onto $U_j$,
$j=1,\ldots,d$, the steps in \eqref{eq:admm_seq} repeatedly
project an ``offset'' version \smash{$u_{j-1}^{(t)}+\thetaj^{(t-1)}$}
of the current iterate, for $j=1,\ldots,d$ 
(this corresponds to running univariate trend filtering on the current 
residual, in the iterations \eqref{eq:backfit_tf}). 

There is a considerable literature on alternating projections in 
optimization, see, e.g., \citet{bauschke1996projection} for a review.   
Many alternating projections algorithms can be derived from the
perspective of an operator splitting technique, e.g., 
the alternative direction method of multipliers (ADMM). 
Indeed, the steps in \eqref{eq:admm_seq}
appear very similar to those from an ADMM algorithm applied to the
dual \eqref{eq:add_tf_dual}, if we think of the ``offset''
variables $\thetaj$, $j=1\ldots,d$ in the iterations
\eqref{eq:admm_seq} as {\it dual} variables in the dual problem
\eqref{eq:add_tf_dual} (i.e., if we think of the primal variables
$\thetaj$, $j=1,\ldots,d$ as dual variables in the dual problem
\eqref{eq:add_tf_dual}).  This connection inspires a new
parallel version of backfitting, presented in the next subsection.

\subsection{Parallelized backfitting}

We have seen that backfitting is a special type of alternating
projections algorithm, applied to the dual problem
\eqref{eq:add_tf_dual}. For set intersection
problems (where we seek a point in the intersection of given closed, convex
sets), the optimization literature offers a variety of {\it parallel projections}
methods (in contrast to alternating projections methods) that are provably
convergent.  One such method can be derived using ADMM (e.g., see Section 5.1 of 
\citet{boyd2011distributed}), and a similar construction can be used for the
dual problem \eqref{eq:add_tf_dual}. We first rewrite this problem as
\begin{equation}
\label{eq:add_tf_dual_re}
\begin{alignedat}{2}
&\ccol{\min_{u_0,u_1,\ldots,u_d \in \R^n}} \quad 
&&\half \|Y - \bar{Y}\one - u_0\|_2^2 + 
\sum_{j=1}^d I_{U_j}(u_j) \\
&\ccol{\st} \quad && u_0=u_1,\; u_0=u_2, \;\ldots\; u_0=u_d,   
\end{alignedat}
\end{equation}
where we write $I_S$ for the indicator function of a set $S$  
(equal to 0 on $S$, and $\infty$ otherwise). Then we
define the augmented Lagrangian, for an arbitrary $\rho>0$, as  
$$
L_\rho(u_0,u_1,\ldots,u_d,\gamma_1,\ldots,\gamma_d) =  
\half \|Y - \bar{Y}\one - u_0\|_2^2 + \sum_{j=1}^d
\bigg( I_{U_j}(u_j) + \frac{\rho}{2} \|u_0 - u_j + \gamma_j \|_2^2 
- \frac{\rho}{2} \|\gamma_j\|_2^2 \bigg),
$$
The ADMM steps for problem \eqref{eq:add_tf_dual_re} are now given by repeating,
for $t=1,2,3,\ldots$, 
\begin{equation}
\label{eq:admm_par}
\begin{aligned}
u_0^{(t)} &= 
\frac{1}{\rho d+1}\bigg(Y-\bar{Y}\one + 
\rho \sum_{j=1}^d (u_j^{(t-1)}-\gamma_j^{(t-1)})\bigg) \\ 
u_j^{(t)} &= \Pi_{U_j}\big(u_0^{(t)}+\gamma_j^{(t-1)}\big), 
\;\;\; j=1,\ldots,d\\
\gamma_j^{(t)} &= \gamma_j^{(t-1)} + u_0^{(t)} - u_j^{(t)},   
\;\;\; j=1,\ldots,d.
\end{aligned}
\end{equation}
Now compare \eqref{eq:admm_par} to
\eqref{eq:admm_seq}---the key difference is that in
\eqref{eq:admm_par}, the updates to $u_j$, $j=1,\ldots,d$, i.e., the
projections onto 
$U_j$, $j=1,\ldots,d$, completely decouple and can hence be performed   
{\it in parallel}.  Run properly, this could provide
a large speedup over the sequential projections in
\eqref{eq:admm_seq}.  

Of course, for our current study, the dual problem 
\eqref{eq:add_tf_dual_re} is really only interesting insofar as
it is connected to the additive trend filtering problem
\eqref{eq:add_tf}.  Fortunately, the parallel projections algorithm
\eqref{eq:admm_par} maintains a very useful connection to the primal
problem \eqref{eq:add_tf}: \smash{$\rho\hat\gamma_j$}, $j=1,\ldots,d$
(i.e., the scaled iterates \smash{$\rho\gamma_j^{(t)}$},
$j=1,\ldots,d$ at convergence) are optimal for the additive trend
filtering problem \eqref{eq:add_tf}.  This is simply because the dual
of the dual problem \eqref{eq:add_tf_dual_re} is the additive
trend filtering problem \eqref{eq:add_tf} (therefore $\rho\gamma_j$,
$j=1,\ldots,d$, which are dual to the constraints in \eqref{eq:add_tf_dual_re},
are equivalent 
to the primal parameters $\thetaj$, $j=1,\dots d$ in \eqref{eq:add_tf}). 
We state this next as a theorem, and transcribe the iterations in
\eqref{eq:admm_par} into an equivalent primal form, in Algorithm
\ref{alg:backfit_par}.  For details, see Appendix \ref{app:backfit_par}.    

\begin{theorem}
\label{thm:backfit_par}
Initialized arbitrarily, the ADMM steps \eqref{eq:admm_par}
produce parameters \smash{$\hat\gamma_j$}, $j=1,\ldots,d$  
(i.e., the iterates \smash{$\gamma_j^{(t)}$}, $j=1,\ldots,d$ at convergence)
such that the scaled parameters \smash{$\rho \hat\gamma_j$}, $j=1,\ldots,d$
solve additive trend filtering \eqref{eq:add_tf}. Further, the outputs 
\smash{$\hthetaj$}, $j=1,\ldots,d$ of Algorithm \ref{alg:backfit_par} 
solve additive trend filtering \eqref{eq:add_tf}.   
\end{theorem}

\begin{algorithm}[tb]
\caption{Parallel backfitting for additive trend filtering}
\label{alg:backfit_par}
Given responses $Y^i \in \R$, input points $X^i \in \R^d$,
$i=1,\ldots,n$, and $\rho>0$.
\begin{enumerate}
\item Initialize \smash{$u_0^{(0)}=0$},
  \smash{$\thetaj^{(0)}=0$} and \smash{$\thetaj^{(-1)}=0$} for $j=1,\ldots,d$. 
\item For $t=1,2,3,\ldots$ (until convergence): 
\begin{enumerate}
\item[a.] $\displaystyle u_0^{(t)} = 
\frac{1}{\rho d+1} \bigg(Y - \bar{Y}\one -  
\sum_{j=1}^d \thetaj^{(t-1)}\bigg) +
\frac{\rho d}{\rho d+1} \bigg( u_0^{(t-1)} + 
\frac{1}{\rho d} \sum_{j=1}^d
\big(\thetaj^{(t-2)}-\thetaj^{(t-1)} \big)\bigg)$  
\end{enumerate}
\begin{enumerate}
\item[b.] For $j=1,\ldots,d$ (in parallel):
\begin{enumerate}
\item[(i)] $\thetaj^{(t)} = \rho \cdot \TF_\lambda \big( 
u_0^{(t)} + \thetaj^{(t-1)}/\rho, \Xj \big)$
\item[(ii)] (Optional) $\thetaj^{(t)}=\thetaj^{(t)} - \frac{1}{n} 
  \one^T \thetaj^{(t)}$
\end{enumerate}
\end{enumerate}
\item Return $\hthetaj$, $j=1,\ldots,d$ (parameters
  $\thetaj^{(t)}$, $j=1,\ldots,d$ at convergence). 
\end{enumerate}
\end{algorithm}

Written in primal form, we see that the
the parallel backfitting approach in Algorithm \ref{alg:backfit_par}
differs from what may be considered the ``naive'' approach to
parallelizing the usual backfitting iterations in Algorithm
\ref{alg:backfit}.  Consider $\rho=1$.  If we were to replace Step 2a in
Algorithm \ref{alg:backfit_par} with \smash{$u_0^{(t)}=r^{(t-1)}$}, the full
residual  
$$
r^{(t-1)} = Y-\bar{Y}\one-\sum_{j=1}^d \thetaj^{(t-1)},
$$
then the update steps for 
\smash{$\thetaj^{(t)}$, $j=1,\ldots,d$} that follow would be just
given by applying univariate trend filtering to each partial
residual (without sequentially updating the partial residuals
between trend filtering runs).  This naive parallel
method has no convergence guarantees, and can fail even  
in simple practical examples to produce optimal solutions.  
Importantly, Algorithm \ref{alg:backfit_par} does not take
\smash{$u_0^{(t)}$} to  
be the full residual, but as Step 2a shows, uses a less greedy
choice: it basically takes \smash{$u_0^{(t)}$} to be a convex
combination of the residual \smash{$r^{(t-1)}$} and its
previous value \smash{$u_0^{(t-1)}$}, with higher 
weight on the latter. The subsequent parallel updates for   
\smash{$\thetaj^{(t)}$, $j=1,\ldots,d$} are still given
by univariate trend filtering fits, and though these steps do not
exactly use partial residuals (since \smash{$u_0^{(t)}$} is not
exactly the full residual), they are guaranteed to produce additive
trend filtering solutions upon convergence (as per Theorem 
\ref{thm:backfit_par}).   An example of cyclic versus parallelized 
backfitting is given in Appendix \ref{app:experiment_backfitting}.

\section{Experiments}
\label{sec:experiments}

Through empirical experiments, we examine the performance of  
additive trend filtering relative to additive smoothing splines.  We also
examine the efficacy of cross-validation for choosing the tuning  
parameter $\lambda$, as well as the use of multiple tuning parameters    
(i.e., a separate parameter $\lambda_j$ for each component $j=1,\ldots,d$).  
All experiments were performed in R. For the univariate trend filtering solver,
we used the {\tt trendfilter} function in the {\tt glmgen} package, which is an
implementation of the fast ADMM algorithm given in \citet{ramdas2016fast}.  For
the univariate smoothing spline solver, we used the {\tt smooth.spline} function
in base R.  

\subsection{Simulated heterogeneously-smooth data}  
\label{sec:experiment_large_sim}

We sampled $n=2500$ input points in $d=10$ dimensions, 
by assigning the inputs along each dimension
\smash{$\Xj=(X^1_j,\ldots,X^n_j)$} to be a different permutation of
the equally spaced points $(1/n,2/n,\ldots,1)$, for $j=1,\ldots,10$.
For the componentwise trends, we examined sinusoids with
Doppler-like spatially-varying frequencies:
$$
g_{0j}(x_j) = \sin\bigg(\frac{2\pi}{(x_j+0.1)^{j/10}}\bigg), \;\;\;
j=1,\ldots,10. 
$$
We then defined the component functions as
\smash{$f_{0j}=a_jg_{0j}-b_j$}, $j=1,\ldots,d$, where $a_j,b_j$ were
chosen so that \smash{$f_{0j}$} had empirical mean zero and empirical  
norm \smash{$\|f_{0j}\|_n=1$}, for $j=1,\ldots,d$. 
The responses were generated according to
\smash{$Y^i \isim N(\sum_{j=1}^d f_{0j}(X^i_j),\sigma^2)$},
$i=1,\ldots,2500$.  By construction, in this setup, there
is considerable heterogeneity in the levels of smoothness both within
and between the component functions.

The left panel of Figure \ref{fig:large_sim} shows a comparison of the
MSE curves from additive trend filtering in \eqref{eq:add_tf} (of
quadratic order, $k=2$) and additive smoothing splines in
\eqref{eq:add_quad} (of cubic order).  We set $\sigma^2$ in the
generation of the responses so that the signal-to-noise ratio (SNR)
was \smash{$\|f_0\|_n^2/\sigma^2=4$}, where \smash{$f_0=\sum_{j=1}^d  
f_{0j}$}.  The two methods (additive trend filtering and additive
smoothing splines) were each allowed their own sequence of tuning 
parameter values, and results were averaged over 10 repetitions
from the simulation setup described above.  As we can see, additive 
trend filtering achieves a better minimum MSE along its regularization 
path, and does so at a less complex model (lower df).

\begin{figure}[tb]
\centering
\includegraphics[width=0.45\textwidth]{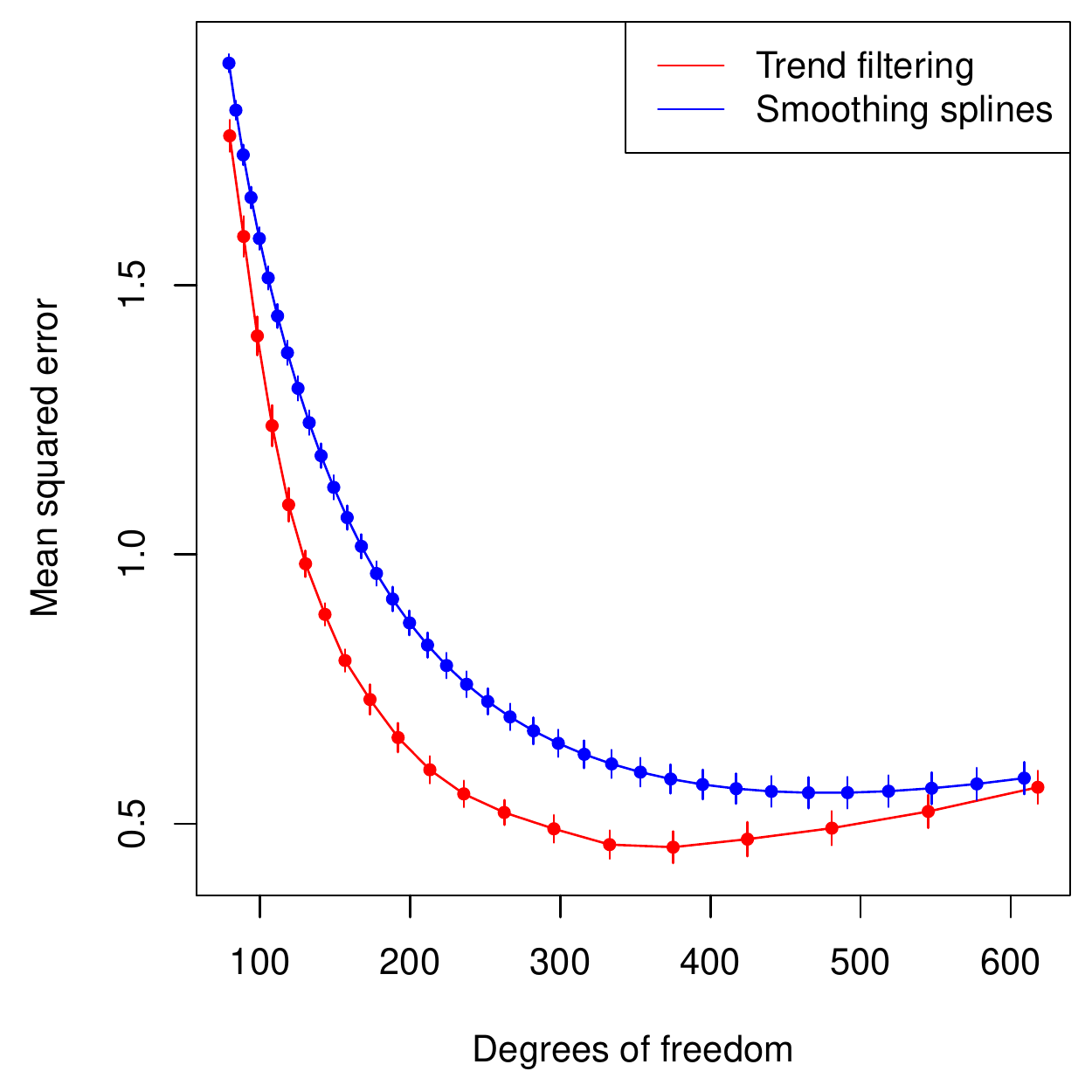} 
\includegraphics[width=0.45\textwidth]{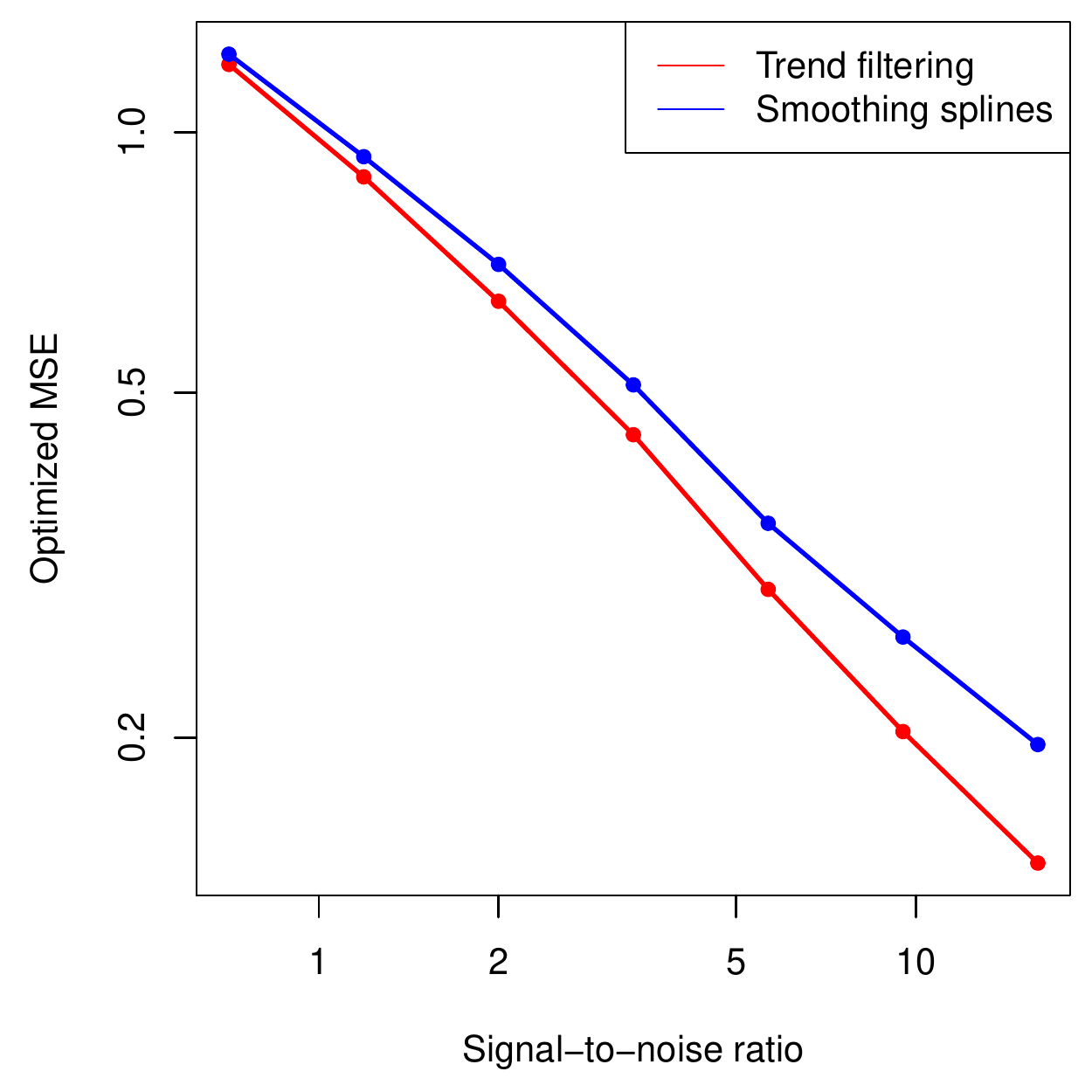} 
\caption{\it\small The left panel shows the MSE curves for 
  additive trend filtering \eqref{eq:add_tf} (of quadratic order) and additive
  smoothing splines \eqref{eq:add_quad} (of cubic order), computed over 10
  repetitions from the heterogeneous smoothness simulation with $n=2500$ and
  $d=10$, described in Section \ref{sec:experiment_large_sim}, where the SNR is 
set to 4. Vertical segments denote $\pm 1$ standard deviations. The right  
panel displays the best-case MSE for each method (the minimum MSE over its  
regularization path), in a problem setup with $n=1000$
and $d=6$, as the signal-to-noise ratio (SNR) varies from 0.7 to 16, in equally 
spaced values on the log scale.} 
\label{fig:large_sim}
\end{figure}

The right panel of Figure \ref{fig:large_sim} shows the best-case
MSEs for additive trend filtering and additive smoothing splines (i.e.,
the minimum MSE over their regularization paths) as the noise level 
$\sigma^2$ is varied so that the SNR ranges from 0.7 to 1.6, in equally spaced 
values on the log scale.  The results were 
again averaged over 10 repetitions of data drawn from a simulation setup
essentially the same as the one described above, except that we considered a
smaller problem size, with $n=1000$ and $d=6$.  The plot reveals that 
additive trend filtering performs increasingly well (in comparison to
additive smoothing splines) as the SNR grows, which is not surprising,
since for high SNR levels it is able to better capture the heterogeneity in 
the component functions. 

Lastly, in Appendix \ref{app:experiment_large_sim_homo}, we present 
results from an experimental setup mimicking that in this
subsection, except with the component functions  
\smash{$f_{0j}$}, $j=1,\ldots,d$ having homogeneous smoothness
throughout.  The results show that additive trend filtering and
additive smoothing splines perform nearly exactly the same.  


\subsection{Cross-validation for tuning parameter selection}
\label{sec:cv}

Sticking to the simulation setup from the last subsection, but at the
smaller problem size, $n=1000$ and $d=6$ (used to produce the right panel of  
Figure \ref{fig:large_sim}), we study in the left panel of Figure
\ref{fig:large_sim_cv} the use of 5-fold cross-validation (CV) to select  
the tuning parameter $\lambda$ for additive trend filtering and additive  
smoothing splines.  Displayed are the resulting MSE curves as the SNR 
varies from 0.7 to 16.  Also shown on the same plot are the oracle MSE curves
(which are the same as those the right panel of Figure \ref{fig:large_sim}), in
which $\lambda$ has been chosen to minimize the MSE for each method.
We can see that the performance of each method degrades using CV, but not by 
much. 

In the right panel of the figure, we examine the use of multiple tuning
parameters for additive smoothing splines and additive trend filtering, i.e.,
replacing the penalties in \eqref{eq:add_quad} and \eqref{eq:add_tf} by those in
\eqref{eq:multi_lambda}, respectively, so we now have $d$ tuning parameters
$\lambda_j$, $j=1,\ldots,d$.  When the function we are estimating has different
amounts of smoothness along different dimensions, we have argued (and seen
through examples) that additive trend filtering---using only a single tuning
parameter $\lambda$---can accomodate these differences, at least somewhat,
thanks to its locally adaptive nature.  But, when the difference in smoothness
across dimensions is drastic enough, it may be worthwhile to move to multiple
tuning parameters $\lambda_j$, $j=1,\ldots,d$.  Certainly, moving to multiple
tuning parameters will help additive smoothing splines address heterogeneity in 
smoothness across dimensions, which will have greater difficulty in
accommodating such heterogeneity using a single tuning parameter, compared to 
additive trend filtering. 

When $d$ is moderate (even just for $d=6$), cross-validation over a
$d$-dimensional grid of values for $\lambda_j$, $j=1,\ldots,d$ can be
prohibitive.  As discussed in Section \ref{sec:multi_lambda}, there is a lot of
literature dedicated to an alternative approach based on restricted maximum
likelihood (REML), mostly surrounding additive models built from linear
smoothers.  Unfortunately, as far as we understand, REML does not easily apply
to additive trend filtering. We thus use the following simple approach for
multiple tuning parameter selection: within each backfitting loop, for each
component $j=1,\ldots,d$, we use (univariate) CV to choose $\lambda_j$.  While
this does not solve a particular convex optimization problem, and is not
guaranteed to converge in general, we have found it to work quite well in
practice.  (We have helpful to use CV to find the best choice of single tuning
parameter for additive trend filtering, and then use this estimate to initialize
the backfitting routine.) The right panel of Figure \ref{fig:large_sim_cv}
compares the performance of this so-called backfit-CV tuning to the oracle, that
chooses just a single tuning parameter.  Both additive trend filtering and
additive smoothing splines are seen to improve with $d$ tuning parameters, tuned
by backfit-CV, in comparison to the oracle choice of tuning parameter.
Interestingly, we also see that additive smoothing splines with $d$ tuning
parameters performs on par with additive trend filtering with the oracle choice
of tuning parameter.  (In this example, REML tuning for additive smoothing
splines---as implemented by the {\tt mgcv} R package---performed a bit worse
than backfit-CV tuning, and so we only show results from the latter.)

\begin{figure}[tb]
\centering
\includegraphics[width=0.45\textwidth]{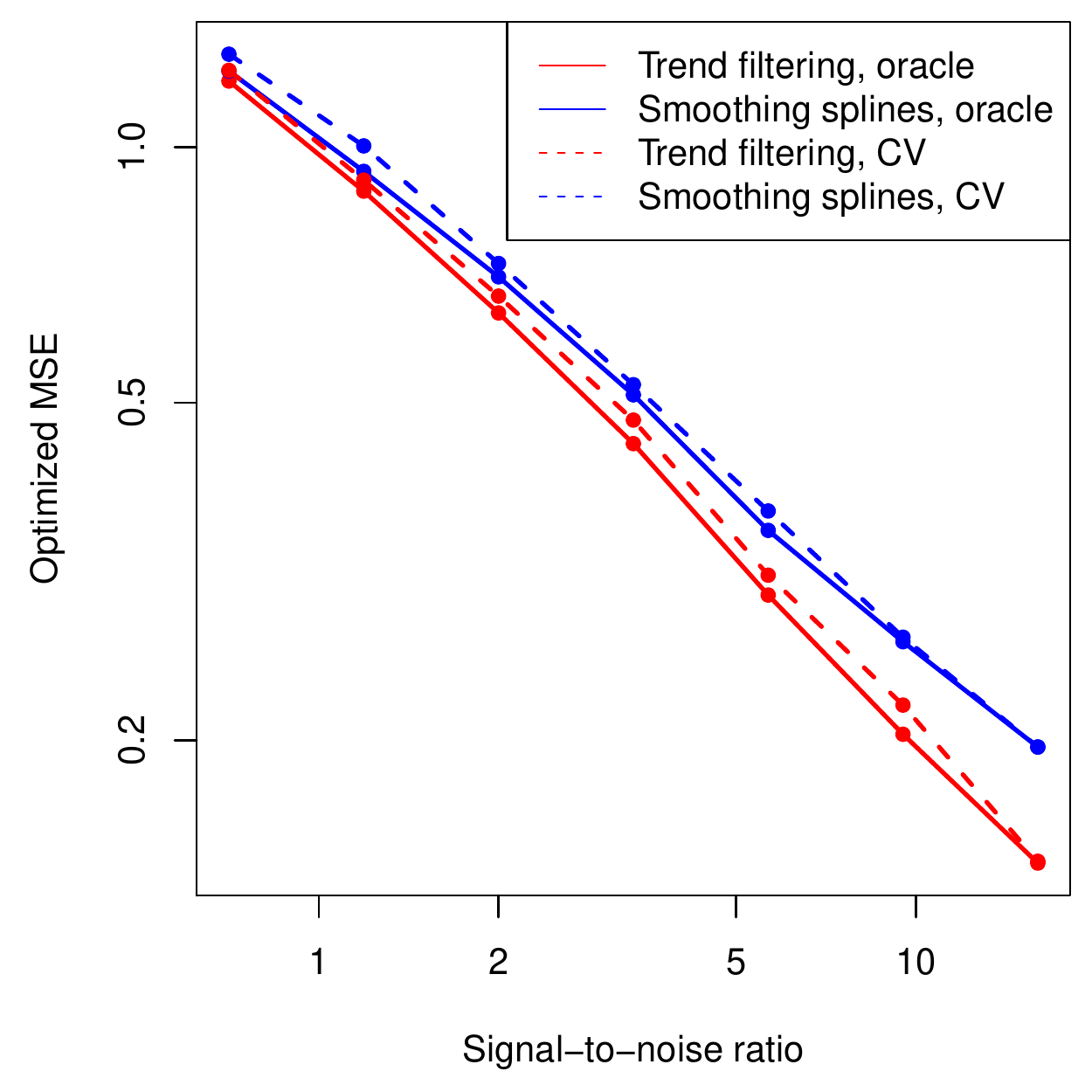} 
\includegraphics[width=0.45\textwidth]{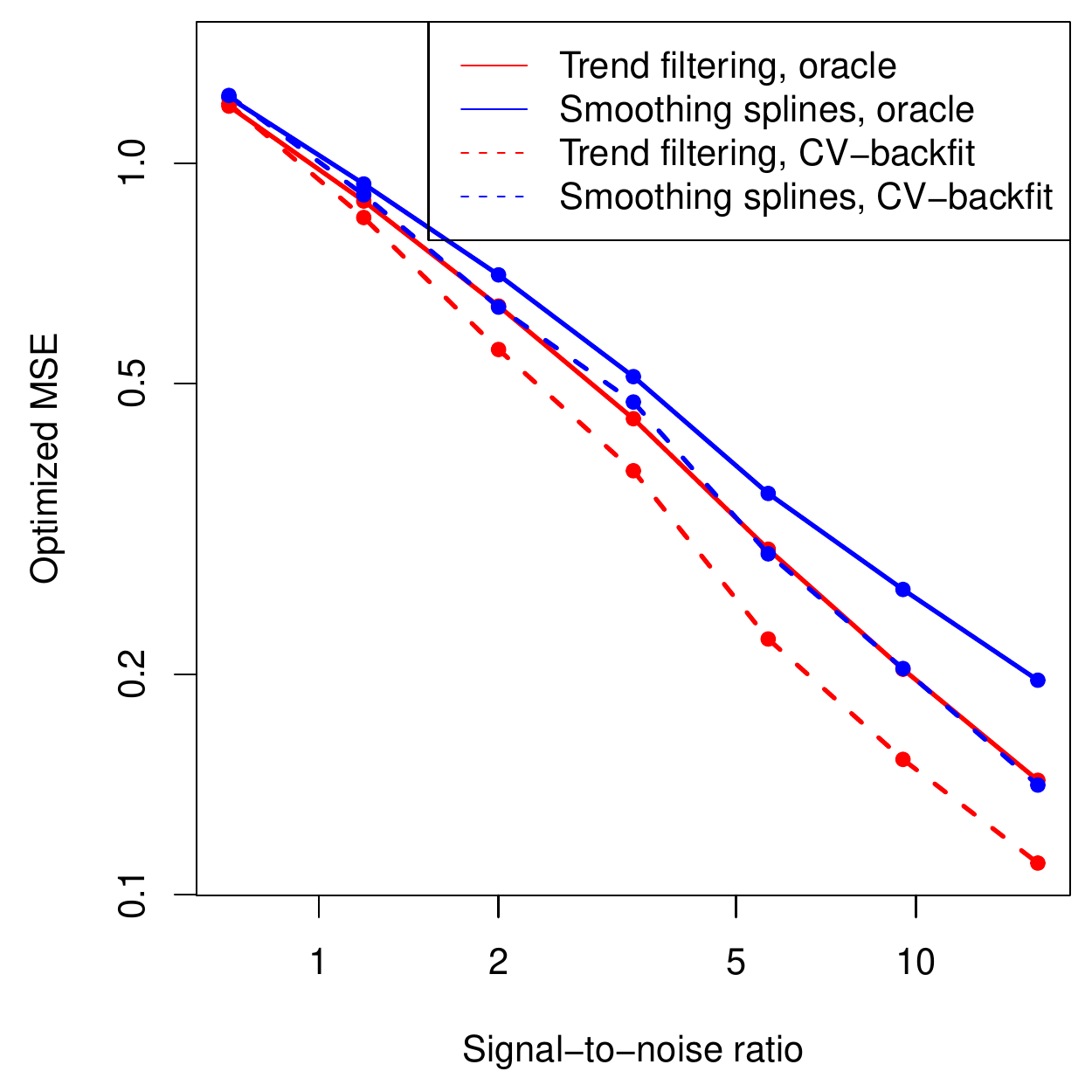} 
\caption{\it\small Both panels display results from the same simulation setup as
  that in the right panel of Figure \ref{fig:large_sim}. The left panel 
  shows MSE curves when the estimators are tuned by 5-fold cross-validation
  (CV), and also by the oracle (reflecting the minimum possible MSE). The
  right panel displays MSE curves when we allow each estimator to have $d$
  tuning  
  parameters, tuned by a hybrid backfit-CV method explained in the text,
  versus the oracle MSE curves for a single tuning parameter.}  
\label{fig:large_sim_cv}
\end{figure}

\section{Discussion} 
\label{sec:discuss}

We have studied additive models built around the univariate trend filtering
estimator, i.e., defined by penalizing according to the sum of $\ell_1$ norms of
discrete derivatives of the component functions.  We examined basic properties
of these additive models, such as extrapolation of the fitted values to a
$d$-dimensional surface, uniqueness of the component fits, and characterization
of effective degrees of freedom of the fit.  When the underlying regression
function is additive, with components whose $k$th derivatives are of bounded
variation, we derived error rates for $k$th order additive trend filtering:
\smash{$n^{-(2k+2)/(2k+3)}$} for a fixed input dimension $d$ (under weak
assumptions), and \smash{$d n^{-(2k+2)/(2k+3)}$} for a growing dimension $d$
(under stronger assumptions).  We showed these rates are sharp by
establishing matching minimax lower bounds, and showed that additive linear
smoothers (e.g., additive smoothing splines) can at best achieve a rate of
\smash{$n^{-(2k+1)/(2k+2)}$} for fixed $d$, and \smash{$dn^{-(2k+1)/(2k+2)}$}
for growing $d$, over the same function class. Finally, on the computational
side, we devised a provably convergent parallel backfitting algorithm for
additive trend filtering.

It is worth noting that our parallel backfitting method is not specific to 
additive trend filtering, but it can be embedded in a more general parallel
coordinate descent framework \citep{tibshirani2017dykstra}.  It 
would be interesting to follow up on this framework to see what practical
benefits it can provide for additive models in particular.  

A natural natural extension of our work is to consider the high-dimensional
case, where $d$ is comparable or possibly even much larger than $n$, and we fit
a {\it sparse additive model} by employing an additional sparsity penalty in
problem \eqref{eq:add_tf}. Another natural extension is to consider responses
$Y^i|X^i$, $i=1,\ldots,n$ from an exponential family distribution, and we fit a 
{\it generalized additive model} by changing the loss in \eqref{eq:add_tf}.
After we completed an initial version of this paper, both extensions have been
pursued: \citet{tan2017penalized} develop a suite of error bounds for sparse
additive models, with various form of penalties (which include total variation
on derivatives of components); and \citet{harris2018sparse} give comprehensive
theory for sparse generalized additive models, with various types of penalties 
(which again include total variation on derivatives of components).

\subsection*{Acknowledgements}

We are very thankful to Garvesh Raskutti for his generous help and insight on
various issues, and Martin Wainwright for generously sharing his unpublished
book with us. We are also grateful to an anonymous referee whose thoughtful
comments improved our paper.

\newpage
\appendix

\section{Appendix}

\subsection{Fast extrapolation}
\label{app:fastext} 

We discuss extrapolation using the fitted functions
\smash{$\hf_j$}, $j=1,\ldots,d$ from additive trend filtering 
\eqref{eq:add_tf_func_tv}, as in \eqref{eq:fit_j}.  We must compute
the coefficients \smash{$\hat\alpha_j=(\hat{a}_j,\hat{b}_j)$} whose
block form is given in \eqref{eq:a_j}, \eqref{eq:b_j}.  Clearly, the
computation 
of \smash{$\hat{b}_j$} in \eqref{eq:b_j} requires $O(n)$ operations
(owing to the bandedness of \smash{$\Dj$}, and treating $k$ as
a constant).  As for \smash{$\hat{a}_j$} in \eqref{eq:a_j}, it can be
seen from the structure of \smash{$C^{(\Xj,k+1)}$} as described in
\citet{wang2014falling} that 
\begin{gather*}
(\hat{a}_j)_1 = (S_j\hthetaj)_1, \\
(\hat{a}_j)_\ell = \frac{1}{(\ell-1)!} 
\bigg[\diag\bigg(\frac{1}{X_{j}^\ell-X_{j}^1}, \ldots,
\frac{1}{X_{j}^n-X_{j}^{n-\ell+1}}\bigg) 
D^{(\Xj,\ell-1)} S_j \hthetaj\bigg]_1, \;\;\; 
\ell=2,\ldots,k+1, \\
\end{gather*}
which takes only $O(1)$ operations (again treating $k$ as constant,
and now using the bandedness of each \smash{$D_j^{(\Xj,\ell-1)}$},
$\ell=2,\ldots,k+1$).  In total then, computing the coefficients
\smash{$\hat\alpha_j=(\hat{a}_j,\hat{b}_j)$} requires $O(n)$
operations, and computing 
\smash{$\hat\alpha=(\hat\alpha_1,\ldots,\alpha_d)$} requires $O(nd)$
operations. 

After having computed
\smash{$\hat\alpha=(\hat\alpha_1,\ldots,\alpha_d)$}, which only needs
to be done once, a prediction at a new point $x=(x_1,\ldots,x_d) \in
\R^d$ with the additive trend filtering fit \smash{$\hf$} is given by
$$
\hf(x) = \bar{Y} + \sum_{j=1}^d 
\sum_{\ell=1}^n \hat\alpha^\ell_j h_\ell ^{(\Xj)} (x_j),
$$
This requires \smash{$O(d + \sum_{j=1}^d \sum_{\ell=k+2}^n
  1\{\hat{\alpha}^\ell_j \neq 0\})$} operations, utilizing the
sparsity of the components in \smash{$\hat\alpha$} not associated with
the polynomial basis functions.

\subsection{Proof of Lemma \ref{lem:unique}}
\label{app:unique}

We begin by eliminating the constraint in the additive trend filtering 
problem \eqref{eq:add_tf}, rewriting it as
$$
\min_{\theta_1,\ldots,\theta_d \in \R^n} \; 
\half \bigg\| MY - \sum_{j=1}^d M\thetaj \bigg\|_2^2 +  
\lambda \sum_{j=1}^d \big\|\Dj S_j M \thetaj\big\|_1, 
$$
where $M=I-\one\one^T/n$.  Noting that $\Dj\one=0$ for 
$j=1,\ldots,d$, we can replace the penalty term above by
\smash{$\sum_{j=1}^d \|\Dj S_j \thetaj\|_1$}. Reparametrizing using
the falling factorial basis, as in Lemma \ref{lem:fall_fact}, yields
the problem 
$$
\min_{a \in \R^{(k+1)d}, \, b\in\R^{(n-k-1)d}} \;\;\;
\half \bigg\| MY - M \sum_{j=1}^d P_j a_j -  
M \sum_{j=1}^d K_j b_j \bigg\|_2^2 +
\lambda k! \sum_{j=1}^d \|b_j\|_1,
$$
where we have used the abbreviation \smash{$P_j=P^{(\Xj,k)}$} and
\smash{$K_j=K^{(\Xj,k)}$}, as well as the block representation 
\smash{$\alpha_j=(a_j,b_j) \in \R^{(k+1)} \times
  \R^{(n-k-1)}$}, for $j=1,\ldots,d$.  Since each $P_j$, $j=1,\ldots
d$ has $\one$ for its first column, the above 
problem is equivalent to 
$$
\min_{a \in \R^{kd}, \, b\in\R^{(n-k-1)d}} \;\;\;
\half \bigg\| MY - M \sum_{j=1}^d \tilde{P}_j a_j -   
M \sum_{j=1}^d K_j b_j \bigg\|_2^2 +
\lambda k! \sum_{j=1}^d \|b_j\|_1,
$$
where \smash{$\tilde{P}_j$} denotes $P_j$ with the first column
removed, for $j=1,\ldots,d$.  To be clear, solutions in the above
problem and the original trend filtering formulation \eqref{eq:add_tf} 
are related by 
$$
\hthetaj = \tilde{P}_j\hat{a}_j+K_j\hat{b}_j, \;\;\; 
j=1,\ldots,d.
$$
Furthermore, we can see that
\smash{$\hat{a}=(\hat{a}_1,\ldots\hat{a}_d)$} solves 
\begin{equation}
\label{eq:a_min}
\min_{a \in \R^{kd}} \; \half 
\bigg\|\bigg(MY - M \sum_{j=1}^d K_j \hat{b}_j \bigg) 
- \tilde{P} a \bigg\|_2^2,
\end{equation}
where \smash{$\tilde{P}$} is as defined in \eqref{eq:ptil}, 
and \smash{$\hat{b}=(\hat{b}_1,\ldots,\hat{b}_d)$} solves  
$$
\min_{b \in \R^{(n-k-1)d}} \; \half 
\bigg\|UU^T MY - UU^T M \sum_{j=1}^d K_j b_j \bigg\|_2^2 + 
\lambda k!  \|b\|_1,  
$$
where $UU^T$ is the projection orthogonal to the column space of 
\smash{$\tilde{P}$}, i.e., it solves 
\begin{equation}
\label{eq:b_min}
\min_{b \in \R^{(n-k-1)d}} \; \half 
\big\|U^T MY - \tilde{K} b \big\|_2^2 + 
\lambda k!  \|b\|_1,  
\end{equation}
where \smash{$\tilde{K}$} is as in \eqref{eq:ktil}.
Since problem \eqref{eq:b_min} is a standard lasso problem,  
existing results on the lasso (e.g., \citet{tibshirani2013lasso}) imply that  
the solution \smash{$\hat{b}$} is unique whenever
\smash{$\tilde{K}$} has columns in general position.  This proves the
first part of the lemma. 
For the second part of the lemma, note that the solution
\smash{$\hat{a}$} in the least squares problem \eqref{eq:a_min}
is just given by the regression of 
\smash{$MY - M \sum_{j=1}^d K_j \hat{b}_j$} onto 
\smash{$\tilde{P}$}, which is unique whenever \smash{$\tilde{P}$} has
full column rank.  This completes the proof.

\subsection{Derivation of additive trend filtering dual}
\label{app:dual}

As in the proof of Lemma~\ref{lem:unique}, we begin by rewriting 
the problem \eqref{eq:add_tf} as
$$
\min_{\theta_1,\ldots,\theta_d \in \R^n} \; 
\half \bigg\| MY - \sum_{j=1}^d M\thetaj \bigg\|_2^2 +  
\lambda \sum_{j=1}^d \|D_j S_j M \thetaj\|_1, 
$$
where $M=I-\one\one^T/n$. Then, we reparametrize
the above problem,
$$
\begin{alignedat}{2}
&\ccol{\min_{\substack{\theta_1,\ldots,\theta_d \in \R^n \\  
 w \in \R^n, \; z \in \R^{md}}}} \quad &&\half \|MY - w \|_2^2 +     
\lambda \sum_{j=1}^d \|z_j\|_1 \\ 
&\ccol{\st} \quad &&w = \sum_{j=1}^d M\thetaj, \;\;\;  
z_j=D_j S_j M \thetaj, \;\;\; j=1,\ldots,d,
\end{alignedat}
$$
and form the Lagrangian
$$
L(\theta,w,z,u,v) = \half \|MY - w \|_2^2 +   
\lambda \sum_{j=1}^d \|z_j \|_1
+ u^T \bigg(w - \sum_{j=1}^d M\thetaj\bigg)
+ \sum_{j=1}^d v_j^T (D_jS_j M \thetaj - z_j).
$$
Minimizing the Lagrangian $L$ over all $\theta,z$ yields the dual
problem 
$$
\begin{alignedat}{2}
&\ccol{\max_{\substack{u \in \R^n \\ v_1,\ldots,v_d \in \R^m}}} \quad 
&&\half \|MY\|_2^2 - \half\|MY-u\|_2^2 \\
&\ccol{\st}\quad &&u = S_jD_j^T v_j, \;\;\; \|v_j\|_\infty \leq
\lambda, \;\;\; j=1,\ldots,d.
\end{alignedat}
$$
The claimed dual problem \eqref{eq:add_tf_dual} is just the
above, rewritten in an equivalent form.

\subsection{Proof of Lemma \ref{lem:df}}
\label{app:df}

We first eliminate the equality constraint in \eqref{eq:add_tf}, 
rewriting this problem, as was done in the proof of Lemma
\ref{lem:unique}, as  
$$
\min_{\theta_1,\ldots,\theta_d \in \R^d} \; 
\half \bigg\| MY - \sum_{j=1}^d M\thetaj \bigg\|_2^2 +  
\lambda \sum_{j=1}^d \|D_j S_j \thetaj\|_1, 
$$
where $M=I-\one\one^T/n$, and \smash{$D_j=\Dj$}, $j=1,\ldots,d$.  This
is a generalized lasso problem with a design matrix 
$T \in \R^{n\times nd}$ that has $d$ copies of $M$ stacked along its
columns, and a penalty matrix $D \in \R^{nd\times nd}$ that is block 
diagonal in the blocks $D_j$, $j=1,\ldots,d$.  Applying Theorem 3 of
\citet{tibshirani2012degrees}, we see that
$$
\df(T\htheta) = \E\big[ \dim\big(T \nul(D_{-A})\big) \big],
$$
where \smash{$A = \supp(D\htheta)$}, and 
where $D_{-A}$ denotes the matrix $D$ with rows removed that
correspond to the set $A$.  The conditions for uniqueness in the lemma
now precisely imply that 
$$
\dim\big(T \nul(D_{-A})\big) = \bigg(\sum_{j=1}^d |A_j|\bigg) + kd, 
$$
where $A_j$ denotes the subset of $A$ corresponding to the block of
rows occupied by $D_j$, and $|A_j|$ its cardinality, for $j=1,\ldots,
d$.  This can be verified by transforming to the basis perspective as
utilized in the proofs of Lemmas \ref{lem:fall_fact} and
\ref{lem:unique}.  The desired result is obtained by noting that, for  
$j=1,\ldots,d$, the component \smash{$\hthetaj$} exhibits a knot for
each element in $A_j$.  

\subsection{Preliminaries for the proof of Theorem \ref{thm:error_bd_j}}
\label{app:error_bd_j_prelim}

Before the proof of Theorem \ref{thm:error_bd_j}, we collect important
preliminary results.  We start with a result on orthonormal polynomials.  We
thank Dejan Slepcev for his help with the next lemma.

\begin{lemma}
\label{lem:phi_bdry}
Given an integer $\kappa \geq 0$, and a continuous measure $\Lambda$ on
$[0,1]$, whose Radon-Nikodym derivative $\lambda$ is bounded below and above by   
constants $b_1,b_2>0$, respectively.  Denote by $\phi_0,\ldots,\phi_\kappa$
an orthonormal basis for the space of polynomials of degree $\kappa$ on $[0,1]$,  
given by running the Gram-Schmidt procedure on the
polynomials $1,t,\ldots,t^\kappa$, with respect to the $L_2(\Lambda)$ inner
product. Hence, for $\ell=0,\ldots,\kappa$, $\phi_\ell$ is an $\ell$th degree
polynomial, orthogonal (in $L_2(\Lambda)$) to all polynomials of degree less than
$\ell$, and we denote its leading coefficient by $a_\ell > 0$. Now define, for
$t\in [0,1]$: 
\begin{align*}
\Phi_{\kappa,0}(t) &= \phi_\kappa(t) \lambda(t),\\
\Phi_{\kappa,\ell+1}(t) &= \int_0^t \Phi_{\kappa,\ell} (u) \, du,
 \;\;\; \ell=0,\ldots,\kappa.
\end{align*}
Then the following two relations hold:
\begin{equation}
\label{eq:phi_bdry}
 \Phi_{\kappa,\ell}(1) = 
\begin{cases}
0 & \text{for $\ell=1,\ldots,\kappa$}, \\
\frac{(-1)^\kappa}{a_\kappa \kappa!} & \text{for $\ell=\kappa+1$},  
\end{cases}
\end{equation}
and 
\begin{equation}
\label{eq:phi_abs}
a_\kappa \kappa!
|\Phi_{\kappa,\kappa}(t)| \leq
{2\kappa \choose \kappa} \sqrt{\frac{b_2}{b_1}},  
\;\;\; t \in [0,1].
\end{equation}
\end{lemma}

\begin{proof}
First, we use induction to show that for $t \in [0,1]$,
\begin{align}
\label{eq:phi_integral}
\Phi_{\kappa,\ell}(t) = \int_0^t \phi_\kappa(u)
  \frac{(t-u)^{\ell-1}}{(\ell-1)!} \lambda(u)\, du, \;\;\;
\ell=1,\ldots,\kappa+1.
\end{align}
This statement holds for $\ell=1$ by definition of 
$\Phi_{\kappa,0},\Phi_{\kappa,1}$.  Assume it holds at some
$\ell>1$. Then 
\begin{align*}
\Phi_{\kappa,\ell+1}(t)
&= \int_0^t \int_0^u \phi_\kappa(v)
  \frac{(u-v)^{\ell-1}}{(\ell-1)!} \lambda(v)\, dv  \, du \\  
&= \int_0^t \phi_\kappa(v) \bigg( \int_v^t
  \frac{(u-v)^{\ell-1}}{(\ell-1)!} \, du\bigg) \lambda(v) \, dv \\  
&= \int_0^t \phi_\kappa(v) \frac{(t-v)^{\ell}}{\ell!} \lambda(v) \, dv,      
\end{align*}
where we used inductive hypothesis in the first line and Fubini's
theorem in the second line, which completes the inductive proof.

Now, the relation in
\eqref{eq:phi_integral} shows that $\Phi_{\kappa,\ell}(1)$ is the
$L_2(\Lambda)$ inner product of $\phi_\kappa$ and an $(\ell-1)$st  
degree polynomial, for $\ell=1,\ldots,\kappa$. As $\phi_\kappa$ is 
orthogonal to all polynomials of degree less than $\kappa$, we have
$\Phi_{\kappa,\ell}(1)=0$, $\ell=1,\ldots,\kappa$. For
$\ell=\kappa+1$, note that this same orthogonality along with 
\eqref{eq:phi_integral} also shows   
$$
\Phi_{\kappa,\kappa+1}(1) = \bigg\langle \phi_\kappa,
  \frac{(-1)^\kappa }{a_\kappa \kappa!} \phi_\kappa
\bigg\rangle_2 =  \frac{(-1)^\kappa }{a_\kappa \kappa!}.  
$$
where $\langle \cdot,\cdot \rangle_2$ is the $L_2(\Lambda)$ inner product. 
This establishes the statement in \eqref{eq:phi_bdry}.

As for \eqref{eq:phi_abs}, note that if $\kappa=0$ then the
statement holds, because the left-hand side is 1 and the right-hand side is 
always larger than 1.   Hence consider $\kappa \geq 1$. From
\eqref{eq:phi_integral}, we have, for any $t \in [0,1]$,
\begin{align}
\nonumber
|\Phi_{\kappa,\kappa}(t)| 
&\leq \int_0^t |\phi_\kappa(u)|
  \frac{(t-u)^{\kappa-1}}{(\kappa-1)!} \lambda(u)\, du \\
\nonumber
& \leq \bigg(\int_0^t  \phi_\kappa^2(u) \lambda(u) \, du \bigg)^{1/2}
\bigg(\int_0^t \frac{(t-u)^{2\kappa-2}}{(\kappa-1)!^2} \lambda(u) \, du
  \bigg)^{1/2} \\ 
\label{eq:phi_abs_bd}
&\leq \frac{\sqrt{b_2}}{(\kappa-1)!\sqrt{2\kappa -1}},
\end{align}
where in the second line we used Cauchy-Schwartz, 
and in the third line we used the fact that $\phi_\kappa$ has unit
norm, and the upper bound $b_2$ on $\lambda$. 
Next we bound $a_\kappa$. Let $p$ be the projection of $x^{\kappa}$
onto the space of polynomials of degree $\kappa-1$, with respect to the
$L_2(\Lambda)$ inner product.  Then we have
\smash{$\phi_\kappa=(x^\kappa-p)/\|x^\kappa-p\|_2$}, thus its
leading coefficient is
\smash{$a_\kappa = 1/\|x^\kappa-p\|_2$}, where
$\|\cdot\|_2$ is the $L_2(\Lambda)$ norm. Consider 
\begin{align}
\nonumber
\|x^\kappa-p\|_2 
&\geq \sqrt{b_1} \bigg(\int_0^1 (x^\kappa-p)^2(t) \, dt\bigg)^{1/2} 
\label{eq:phi_ak_bd} \\
&\geq \sqrt{b_1} \bigg(\int_0^1 P_\kappa^2(t) \, dt\bigg)^{1/2} 
= \frac{\sqrt{b_1}}{\sqrt{2\kappa+1}{2\kappa \choose \kappa}}.    
\end{align}
In the first line we used the lower bound $b_1$ on $\lambda$.
In the second we used the fact  
the the Legendre polynomial $P_\kappa$ of degree $\kappa$, shifted to
$[0,1]$ but unnormalized, is the result from projecting out
$1,t,\ldots,t^{\kappa-1}$ from $t^\kappa$, with respect to the uniform 
measure.  In the last step we used the fact that $P_\kappa$ has norm  
\smash{$1/(\sqrt{2\kappa+1}{2\kappa \choose \kappa})$}.  Combining  
\eqref{eq:phi_abs_bd} and \eqref{eq:phi_ak_bd} gives the 
result \eqref{eq:phi_abs}.
\end{proof}

\begin{remark}[\textbf{Special case: uniform measure and Rodrigues' 
    formula}] In the case of $\Lambda=U$, the uniform measure on
  $[0,1]$, we can just take $\phi_0,\ldots,\phi_\kappa$ to be the Legendre    
polynomials, shifted to $[0,1]$ and normalized appropriately.
Invoking Rodrigues' formula to express these functions,  
$$
\phi_\ell(t) = \frac{\sqrt{2\ell+1}}{\ell!}
\frac{d^\ell}{dt^\ell} 
(t^2-t)^\ell, \;\;\; \ell=0,\ldots\kappa,
$$
the results in Lemma \ref{lem:phi_bdry} can be directly verified.   
\end{remark}

We use Lemma \ref{lem:phi_bdry} to construct a sup norm bound on 
functions in $B_J(1)$ that are orthogonal (in $L_2(\Lambda)$) to all polynomials
of degree $k$.  We again Dejan Slepcev for his help with the next lemma.

\begin{lemma}
\label{lem:pi_perp_sup_bd}
Given an integer $k \geq 0$, and a continuous measure $\Lambda$ on
$[0,1]$, whose Radon-Nikodym derivative $\lambda$ is bounded below and above  
by constants $b_1,b_2>0$, respectively.  Let $J$ be a functional satisfying
Assumptions \ref{as:poly_null} and \ref{as:deriv_bd}, for a constant
$L>0$.  There is a constant $R_0 > 0$, that depends only on
$k,b_1,b_2,L$, such that   
$$
\|g\|_\infty \leq R_0, \;\;\; 
\text{for all $g \in B_J(1)$, such that
$\langle g,p\rangle_2=0$ for all polynomials $p$ of degree $k$},
$$
where $\langle \cdot,\cdot \rangle_2$ denotes the $L_2(\Lambda)$ inner product. 
\end{lemma}

\begin{proof}
Fix an arbitrary function $g \in B_J(1)$, orthogonal (in $L_2(\Lambda)$) to all
polynomials of degree $k$. Using integration by parts, and repeated application
of Lemma \ref{lem:phi_bdry}, we have  
\begin{equation}
\label{eq:gj_mean}
0 = a_\ell \ell! \cdot \langle g, \phi_\ell \rangle_2  =
\int_0^1 g^{(\ell)}(t) w_\ell(t) \, dt, \;\;\; \ell=0,\ldots,k,
\end{equation}
where \smash{$w_\ell(t) = (-1)^\ell a_\ell \ell!  
  \Phi_{\ell,\ell}(t)$}, $\ell=0,\ldots,k$, and
by properties \eqref{eq:phi_bdry}, \eqref{eq:phi_abs} of Lemma 
\ref{lem:phi_bdry}, 
\begin{equation}
\label{eq:wj_int}
\int_0^1 w_\ell(t) \, dt = 1, \;\;\;
\int_0^1 |w_\ell(t)| \, dt \leq {2\ell \choose \ell}
\sqrt{\frac{b_2}{b_1}}, \;\;\; \ell=0,\ldots,k.
\end{equation}
Now, we will prove the following by induction:
\begin{equation}
\label{eq:gj_sup}
\| g^{(\ell)} \|_\infty \leq
L \bigg(\frac{b_2}{b_1}\bigg)^{(k-\ell+1)/2}
\prod_{i=\ell}^k {2i \choose i}, \;\;\; \ell=0,\ldots,k. 
\end{equation}
Starting at $\ell=k$, the statement holds because, using
\eqref{eq:gj_mean}, for almost every $t \in [0,1]$,
\begin{align*}
|g^{(k)}(t)|
&= \bigg|g^{(k)}(t) - \int_0^1 g^{(k)}(u) w_k(u) \,du\bigg| \\  
&= \bigg|\int_0^1 \Big(g^{(k)}(t) - g^{(k)}(u)\Big) w_k(u)
  \,du\bigg| \\
&\leq L {2k \choose k} \sqrt{\frac{b_2}{b_1}}, 
\end{align*}
where in the second line we used the fact that the weight function
integrates to 1 from \eqref{eq:wj_int}, and in the third we used   
Assumption \ref{as:deriv_bd} and the upper
bound on the integrated absolute weights from \eqref{eq:wj_int}.  
Assume the statement holds at some $\ell<k$. Then again by
\eqref{eq:gj_mean}, \eqref{eq:wj_int}, for almost every $t \in 
[0,1]$, 
\begin{align*}
|g^{(\ell-1)}(t)|
&= \bigg|\int_0^1 \Big(g^{(\ell-1)}(t) - g^{(\ell-1)}(u)\Big)
  w_{\ell-1}(u) \,du\bigg| \\
&\leq \bigg(\esssup_{0 \leq u<v \leq 1}
  |g^{(\ell-1)}(v)-g^{(\ell-1)}(u)|\bigg) 
{2\ell-2 \choose \ell-1} \sqrt{\frac{b_2}{b_1}} \\ 
&=\bigg(\esssup_{0 \leq u<v \leq 1} \bigg| \int_u^v g^{(\ell)}(s) 
  \, ds \bigg|\bigg)
  {2\ell-2 \choose \ell-1} \sqrt{\frac{b_2}{b_1}} \\ 
&\leq L \bigg(\frac{b_2}{b_1}\bigg)^{(k-\ell+2)/2} 
\prod_{i=\ell-1}^k {2i \choose i},
\end{align*}
the last line using \smash{$\esssup_{0 \leq u<v \leq 1} | \int_u^v
  g^{(\ell)}(s) \, ds| \leq \|g^{(\ell)}\|_\infty$} 
and the inductive hypothesis. This verifies \eqref{eq:gj_sup}.  Taking 
$\ell=0$ in \eqref{eq:gj_sup} and defining
\smash{$R_0=L(b_2/b_1)^{(k+1)/2} \prod_{i=0}^k {2i \choose i}$} proves  
the lemma. 
\end{proof}

We study the minimum eigenvalue of the (uncentered) empirical covariance
matrix of a certain basis for additive $k$th degree polynomials in $\R^{kd}$.
We thank Mathias Drton for his help with part (a) of the next lemma.

\begin{lemma}
\label{lem:poly_mat}
Let $X^i$, $i=1,\ldots,n$ denote an i.i.d.\ sample from a continuous
distribution $Q$ on $[0,1]^d$. For an integer $k \geq 0$, let $V \in 
\R^{n \times kd}$ be a matrix whose $i$th row is given by
\begin{equation}
\label{eq:vi}
V^i =
\Big(X^i_1,(X^i_1)^2,\ldots,(X^i_1)^k,\ldots,X^i_d,(X^i_d)^2,\ldots,(X^i_d)^k\Big)
\in \R^{kd},
\end{equation}
for $i=1,\ldots,n$.  Let 
$$
\nu_n^2 = \lambda_{\min}\Big(\frac{1}{n}V^T V\Big), 
\;\;\;\text{and}\;\;\; 
\nu_0^2 = \lambda_{\min}\Big( \frac{1}{n} \E[V^T V] \Big),
$$
where $\lambda_{\min}(\cdot)$ denotes the minimum eigenvalue of its argument. 
Assuming that $n \geq kd$, the following properties hold:

\begin{itemize}
\item[(a)] $\nu_n>0$, almost surely with respect to $Q$;

\item[(b)] $\nu_0>0$;

\item[(c)] for any $0 \leq t \leq 1$, $\P(\nu_n^2 > t \nu_0^2)$ with
  probability at least \smash{$\displaystyle
    1-d\exp\Big(-\frac{(1-t)^2\nu_0 n}{2(kd)^2}\Big)$}. 
\end{itemize}
\end{lemma}

\begin{proof}
For part (a), if the claim holds for $n=kd$, then it holds for all $n > kd$,
so we may assume without a loss of generality that $n=kd$. Note that the 
determinant of $V \in \R^{n \times kd}$ is a polynomial function, call it
  $q(X)$, of the elements \smash{$X^i_j$}, 
$i=1,\ldots,n$, $j=1,\ldots,d$.  By Lemma 1 of \citet{okamoto1973distinctness},
the roots of any polynomial---that is not identically zero---form a set of
Lebesgue measure zero.  To check that the polynomial $q$ in question is not
identically zero, it suffices to show that it is nonzero at a single realization
of $X$.  To this end, consider an input matrix defined by
$$
X = \left[\begin{array}{c} \alpha_1 I \\ \vdots \\ \alpha_k
            I \end{array}\right] \in \R^{n \times kd},
$$
the rowwise concatenation of $\alpha_\ell I \in \R^{d \times
  d}$, $\ell=1,\ldots,k$.  By the blockwise Vandermonde structure of the
corresponding basis matrix $V$, we have that $q(X) \not= 0$ provided the
coefficients $\alpha_\ell$, $\ell=1,\ldots,k$ are all
distinct.  Therefore $q$ is not identically zero, and with respect to the
continuous distribution $Q$, the determinant of $V$ is nonzero, i.e., $\nu_n>0$, 
almost surely. 

For part (b), given any $a \in \R^{kd}$ with $a \not= 0$, we know
that $Va \not= 0$ almost surely, since $\nu_n>0$ almost surely, by
part (a).  Thus  
$$
a^T \E[V^T V] a = \E\|Va\|_2^2 > 0,
$$
which proves that $\nu_0>0$.

Part (c) is an application of a matrix Chernoff bound from
\citet{tropp2012user}.  In order to apply this result, we must
obtain an almost sure upper bound $R$ on $\lambda_{\max}(V^i
(V^i)^T)$, with $V^i$ as in \eqref{eq:vi} and $\lambda_{\max}(\cdot)$
denoting the maximum eigenvalue of its argument.  This follows as
$$
\lambda_{\max}\big(V^i(V^i)^T\big) \leq \sum_{j=1}^{kd} 
\sum_{\ell=1}^{kd} (V^i_j V^i_\ell)^2 \leq (kd)^2,
$$
as each component of $V^i$ has absolute magnitude at most 1
(recalling that $Q$ is supported on $[0,1]^d$).  Taking $R=(kd)^2$
and applying Corollary 5.2 of \citet{tropp2012user} (to be specific,
applying the form of the Chernoff bound given in Remark 5.3 of this
paper) gives the result. 
\end{proof}

The next lemma pertains to the additive function space
\begin{equation}
\label{eq:mn_space}
\cM_n(\delta) = \bigg\{ \sum_{i=1}^d m_j : \sum_{j=1}^d J(m_j) \leq \delta, \;\, 
\text{and}\;\, \langle m_j,1 \rangle_n=0, \; j=1,\ldots,d \bigg\}.  
\end{equation}
We give a sup norm bound on the components of functions in $\cM_n(1) \cap
B_n(\rho)$.  The proof combines Lemmas \ref{lem:pi_perp_sup_bd} and
\ref{lem:poly_mat}, and uses a general strategy that follows the arguments given
in Example 2.1(ii) of \citet{vandegeer1990estimating}.     

\begin{lemma}
\label{lem:mn_sup_bd}
Let $X^i$, $i=1,\ldots,n$ denote an i.i.d.\ sample from a continuous
distribution $Q$ on $[0,1]^d$, and let $J$ be a seminorm
satisfying Assumptions \ref{as:poly_null} and \ref{as:deriv_bd}.  
There are constants $R_1,R_2,c_0,n_0>0$, depending only on $d,k,L$, such that
for all $\rho>0$ and $n \geq n_0$,  
$$
\|m_j\|_\infty \leq R_1 \rho + R_2, \;\;\; 
\text{for all $j=1,\ldots,d$ and 
$\sum_{j=1}^d m_j  \in \cM_n(1) \cap B_n(\rho)$},    
$$
with probabilty at least $1-\exp(-c_0 n)$, where $\cM_n(1)$ is the 
function space in \eqref{eq:mn_space}.
\end{lemma}

\begin{proof}
Fix an arbitrary \smash{$m=\sum_{j=1}^d m_j \in \cM_n(1) \cap
  B_n(\rho)$}. For each $j=1,\ldots,d$, decompose   
$$
m_j = \langle m_j,1 \rangle_n + p_j + g_j,
$$
where $p_j$ is a polynomial of degree $k$ such that \smash{$\langle
  p_j,1 \rangle_n=0$}, and $g_j$ is orthogonal to all polynomials of degree 
$k$,  with respect to the $L_2(U)$ inner product, with $U$ the uniform
distribution on $[0,1]$; in fact, by definition of $\cM_n(1)$, we know that
\smash{$\langle m_j,1 \rangle_n=0$} so      
$$
m_j = p_j+g_j.
$$
By the triangle inequality and Lemma \ref{lem:pi_perp_sup_bd} applied to the
measure $\Lambda=U$ (whose density is of course lower and upper bounded with  
$b_1=b_2=1$), we have, for each $j=1,\ldots,d$, 
\begin{equation}
\label{eq:g_sup_bd}
\bigg\|\sum_{j=1}^d g_j \bigg\|_\infty \leq 
\sum_{j=1}^d \| g_j \|_\infty \leq R_0 \sum_{j=1}^d J(g_j) \leq R_0, 
\end{equation}
where $R_0 > 0$ is the constant from Lemma \ref{lem:pi_perp_sup_bd}, 
and we have used $J(m_j)=J(g_j)$, for $j=1,\ldots,d$, as
the null space of $J$ contains $k$th degree polynomials.

The triangle inequality and \eqref{eq:g_sup_bd} now imply
\begin{equation}
\label{eq:p_empirical_bd}
\|p\|_n \leq \|m\|_n + \|g\|_n \leq \rho+R_0.
\end{equation}
Write
$$
p(x) = \sum_{j=1}^d \sum_{\ell=1}^k \alpha_{j\ell} x_j^\ell, \;\;\; 
\text{for $x \in [0,1]^d$}, 
$$
for some coefficients $\alpha_{j\ell}$, $j=1,\ldots$, $\ell=1,\ldots,k$.   
For $V \in \R^{n \times kd}$ the basis matrix as in Lemma \ref{lem:poly_mat},
and $\alpha =
(\alpha_{11},\ldots,\alpha_{1k},\ldots,\alpha_{d1},\ldots,\alpha_{dk}) \in 
\R^{kd}$, we have 
$$
\|p\|_n = \frac{1}{\sqrt{n}} \|V\alpha\|_2.
$$
Furthermore, noting 
$$
\|p\|_n \geq \sqrt{\lambda_{\min}\Big(\frac{1}{n}V^T V\Big)}
\|\alpha\|_2, 
$$
we have 
$$
\|\alpha\|_2 \leq \frac{\rho+R_0}{\nu_n},
$$
where $\nu_n^2=\lambda_{\min}(V^T V/n)$, as in Lemma
\ref{lem:poly_mat}, and we have used the upper bound in
\eqref{eq:p_empirical_bd}.  Using part (c) of Lemma
\ref{lem:poly_mat}, with $t=1/2$, we have 
$$
\|\alpha\|_2 \leq \frac{2(\rho+R_0)}{\nu_0},
$$
with probability at least
\smash{$1-d\exp(-\nu_0n/(8(kd)^2))$},  
where $\nu_0^2=\lambda_{\min}(\E[V^T V]/n)$, as in Lemma
\ref{lem:poly_mat}.  Therefore, using the
triangle inequality and the fact that $Q$ is supported on $[0,1]^d$, 
we have for each $j=1,\ldots,d$, and any $x_j \in [0,1]$, 
$$
|p_j(x_j)| \leq \sum_{\ell=1}^k |\alpha_{j\ell}| \leq \|\alpha\|_1 \leq
\frac{2\sqrt{kd}(\rho+R_0)}{\nu_0},
$$
with probability at least \smash{$1-d\exp(-\nu_0n/(8(kd)^2))$}. 
Finally, for each $j=1,\ldots,d$, using the triangle inequality, and
the sup norm bound from Lemma \ref{lem:pi_perp_sup_bd} once again,  
$$
\|m_j\|_\infty \leq \|p_j\|_\infty + \|g_j\|_\infty \leq 
\frac{2\sqrt{kd}(\rho+R_0)}{\nu_0} + R_0,
$$
with probability \smash{$1-d\exp(-\nu_0n/(8(kd)^2))$}, completing the
proof.  
\end{proof}

We give a simple bound on the entropy of an arbitrary sum of sets in  
terms of the entropies of the original sets. 

\begin{lemma}
\label{lem:add_entropy_bd}
Given sets $S_1,\ldots,S_d$ and a norm $\|\cdot\|$, it holds that
$$
\log N ( \delta, \|\cdot\|, S_1+\cdots+S_d) \leq 
\sum_{j=1}^d \log N(\delta/d, \|\cdot\|, S_j).
$$
\end{lemma}

\begin{proof}
For $j=1,\ldots,d$, suppose that $S_j$ can be covered in $N_j$ balls
of radius $\delta/d$, with centers at
\smash{$s_j^1,\ldots,s_j^{N_j}$}.  Take an arbitrary $s \in S_1 +
\cdots + S_d$, and write \smash{$s=\sum_{j=1}^d s_j$}, with $s_j \in
S_j$, $j=1,\ldots,d$.  For each $j=1,\ldots,d$, there is some
\smash{$s_j^{\ell_j}$} such that \smash{$\|s_j-s_j^{\ell}\| \leq
  \delta/d$}, and so by the triangle inequality 
$$
\bigg\|\sum_{j=1}^d s_j - \sum_{j=1}^d s_j^{\ell_j} \bigg\| \leq
\delta.
$$
That is, we have shown that \smash{$\prod_{j=1}^d N_j$} balls of
radius $\delta$ with centers at 
$$
\sum_{j=1}^d s_j^{\ell_j}, \;\;\; \text{for $(\ell_1,\ldots,\ell_d) \in
  \{1,\ldots,N_1\} \times \cdots \times \{1,\ldots,N_d\}$},
$$
cover $S$.  This completes the proof.
\end{proof}

The next result represents our main tool from empirical process theory
that will be used in the proof of Theorem \ref{thm:error_bd_j}.  It is
essentially an application of Lemma 3.5 in
\citet{vandegeer1990estimating} (see also
\citet{vandegeer2000empirical}).  

\begin{lemma}
\label{lem:mn_entropy_bd}
Let $X^i$, $i=1,\ldots,n$ denote an i.i.d.\ sample from a continuous
distribution $Q$ on $[0,1]^d$. Let $\epsilon^i$, $i=1,\ldots,n$ be uniformly  
sub-Gaussian random variates that have variance proxy $\sigma^2>0$ and 
are independent of $X^i$, $i=1,\ldots,n$. Let $J$ be a seminorm satisfying
Assumptions 
\ref{as:poly_null}, \ref{as:deriv_bd}, \ref{as:entropy_bd}, and let $\rho>0$
be arbitrary. Then there are constants $c_1,c_2,c_3,n_0>0$, 
depending only on $d,\sigma,k,L,K,w,\rho$, such that for all $c \geq
c_1$ and $n \geq n_0$,    
$$
\sup_{m \in \cM_n(1) \cap B_n(\rho)} \,
\frac{\frac{1}{n} \sum_{i=1}^n \epsilon^i m(X^i)} 
{\|m\|_n^{1-w/2}} \leq \frac{c}{\sqrt{n}},
$$
with probabilty at least $1 - \exp(-c_2 c^2) - \exp(-c_3 n)$.  
\end{lemma}

\begin{proof}
Let $\Omega_1$ denote the event on which the conclusion in 
Lemma \ref{lem:mn_sup_bd} holds, which has probability at least
$1-\exp(-c_3 n)$ for $n \geq n_1$, for constants $c_3,n_1>0$.
Also let $R_0=R_1\rho+R_2$, where $R_1,R_2>0$ are the 
constants from the lemma. Denote
$$
B_\infty^d(\delta) = \bigg \{ \sum_{j=1}^d f_j : \|f_j\|_\infty \leq \delta, \;  
 j=1,\ldots,d \bigg\}.
$$ 
On $\Omega_1$, consider        
\begin{align}
\label{eq:entropy_ineq_1}
\log N\big(\delta, \|\cdot\|_n, \cM_n(1) \cap B_n(\rho) \big) 
&\leq \log N \big(\delta, \|\cdot\|_n, \cM_n(1) \cap B_\infty^d(R_0) \big) \\     
\label{eq:entropy_ineq_2}
& \leq \sum_{j=1}^d \log N \big(\delta/d, \|\cdot\|_n, B_J(1) \cap
  B_\infty(R_0) \big) \\ 
\label{eq:entropy_ineq_3}
&\leq \sum_{j=1}^d \log N \big( \delta/(R_0 d), \|\cdot\|_n, B_J(1) \cap
  B_\infty(1) \big) \\ 
\label{eq:entropy_ineq_4}
&\leq K d^{1+w} (R_0)^w \delta^{-w}. 
\end{align}
The first inequality \eqref{eq:entropy_ineq_1} uses the sup norm bound from
Lemma \ref{lem:mn_sup_bd}; the second inequality
\eqref{eq:entropy_ineq_2} uses 
$$
\cM_n(1) \cap B_\infty^d(R_0) \subseteq \bigg\{ \sum_{j=1}^d m_j :   
m_j \in B_J(1) \cap B_\infty(R_0), \; j=1,\ldots,d \bigg\},
$$
and applies Lemma \ref{lem:add_entropy_bd} to the space on the right-hand side
above. The third inequality \eqref{eq:entropy_ineq_3} just uses the fact we may
assume $R_0 \geq 1$, without a loss of generality; and the last inequality
\eqref{eq:entropy_ineq_4} uses Assumption \ref{as:entropy_bd}. 
The entropy bound established in \eqref{eq:entropy_ineq_4} allows us
to apply Lemma 3.5 \citet{vandegeer1990estimating} (see also Lemma 8.4 in
\citet{vandegeer2000empirical}), which gives constants $c_1,c_2,n_2>0$, 
depending only on $d,\sigma,k,R_0,K,w$, such that for all $c \geq c_1$ and $n 
\geq n_1$,  
$$
\sup_{m \in \cM_n(1) \cap B_n(\rho)} \,
\frac{\frac{1}{\sqrt{n}} \sum_{i=1}^n \epsilon^i m(X^i)} 
{\|m\|_n^{1-w/2}} \leq c 
$$
on an event $\Omega_2$ with probability at least $1-\exp(-c_2
c^2)$. The desired result in the lemma therefore holds for all $c \geq 
c_1$ and \smash{$n \geq n_0=\max\{n_1,n_2\}$}, on $\Omega_1 \cap
\Omega_2$.    
\end{proof}

We finish with two simple results, on shifting around exponents
in sums and products. 

\begin{lemma}
\label{lem:ab_sum_ineq}
For any $a,b \geq 0$, and any $0 < q < 1$,
$$
(a+b)^q \leq a^q+b^q.
$$
\end{lemma}

\begin{proof}
The function $f(t)=(1+t)^q-(1+t^q)$ has derivative
$f'(t)=q(1+t)^{q-1}-qt^{q-1}<0$ for all $t>0$, and so 
$f(t)<f(0)=0$ for all $t>0$.  Plugging in $t=a/b$ and rearranging
gives the claim.  
\end{proof}

\begin{lemma}
\label{lem:ab_prod_ineq}
For any $a,b \geq 0$, and any $w$,
$$
a b^{1-w/2} \leq a^{1/(1+w/2)} b + a^{2/(1+w/2)}.
$$
\end{lemma}

\begin{proof}
Note that either \smash{$a b^{1-w/2} \leq a^{1/(1+w/2)} b$} 
or \smash{$a b^{1-w/2} \geq a^{1/(1+w/2)} b$}, and in the latter case
we get \smash{$b \leq a^{1/(1+w/2)}$}, so \smash{$a b^{1-w/2} \leq  
  a^{2/(1+w/2)}$}.  
\end{proof}

\subsection{Proof of Theorem \ref{thm:error_bd_j}}
\label{app:error_bd_j}

This proof roughly follows the ideas in the proof of Theorem 9 in
\citet{mammen1997locally}, though it differs in a few key ways.  We use 
$c>0$ to denote a constant that will multiply our final estimation error bound;
it will also control the probability with which our final result holds. 
Some steps will only hold for sufficiently large $n$, but we do not always
make this explicit.  Lastly, we will occasionally abuse our notation for the
empirical norms and empirical inner products by using them with vector
arguments, to be interpreted in the appropriate sense (e.g., \smash{$\langle m,v
  \rangle_n = \frac{1}{n} \sum_{i=1}^n v^i m(X^i)$} for a function $m$ and
vector $v \in \R^n$). 

We break down the presentation of our proof into mini sections for readability.   

\bigskip
\noindent
{\bf Deriving a basic inequality.}
Denote by \smash{$\hf= \sum_{j=1}^d \hf_j$} the total additive 
fit in \eqref{eq:add_func_j}.  Let $\cS$ denote feasible set in
\eqref{eq:add_func_j}.  For any $f \in \cS$, note that by orthogonality, 
$$
\|Y-\bar{Y}\one - f\|_n^2 = 
\|(f_0+\epsilon-\bar\epsilon\one)-f\|_n^2 + (\bar\epsilon)^2 
$$
where \smash{$\bar\epsilon=\frac{1}{n}\sum_{i=1}^n \epsilon^i$}. 
Therefore \smash{$\hf$} must also be optimal for the problem
$$
\min_{f \in \cS} \; \half \|W-f\|_n^2 + \lambda_n J_d(f), 
$$
where \smash{$W^i= f_0(X^i)+\epsilon^i-\bar\epsilon$}, 
$i=1,\ldots,n$, and we denote $\lambda_n=\lambda/n$ 
and \smash{$J_d(f)=\sum_{j=1}^d J(f_j)$}.  Standard arguments (from first-order 
optimality) show that any solution \smash{$\hf$} in the above
satisfies 
$$
\langle W - \hf , \tf - \hf \rangle_n \leq 
\lambda_n\big(J_d(\tf)-\lambda_n J_d(\hf)\big),
$$
for any feasible \smash{$\tf=\sum_{j=1}^d \tf_j \in \cS$}.   
Expanding the definition of $W$ and rearranging gives 
$$
\langle \hf - f_0, \hf - \tf \rangle_n \leq  
\langle \epsilon - \bar\epsilon\one, \hf - \tf \rangle_n +  
\lambda_n\big(J_d(\tf)-\lambda_n J_d(\hf)\big).  
$$
Using the polarization identity \smash{$\langle a,b \rangle  
=\half(\|a\|^2+\|b\|^2-\|a-b\|_2^2)$} for an inner product
$\langle \cdot,\cdot \rangle$ and its corresponding norm $\|\cdot\|$, 
$$
\| \hf - f_0 \|_n^2 + \|\hf - \tf \|_n^2 \leq  
2 \langle \epsilon - \bar\epsilon\one, \hf - \tf \rangle_n + 
2\lambda_n \big(J_d(\tf)-\lambda_n J_d(\hf)\big) +
\|\tf - f_0 \|_n^2.
$$
Abbreviating \smash{$\hDelta = \hf-\tf$},
\smash{$\hJ=J_d(\hf)$}, and
\smash{$\tJ=J_d(\tf)$}, and using \smash{$\langle \bar\epsilon\one,
  \hDelta \rangle =0$}, this becomes 
\begin{equation}
\label{eq:basic_ineq}
\|\hf - f_0 \|_n^2 + \| \hDelta \|_n^2 \leq 
2 \langle \epsilon, \hDelta \rangle_n +
2 \lambda_n(\tJ-\hJ) + \|\tf-f_0\|_n^2, 
\end{equation}
which is our basic inequality.  In what follows, we will restrict our attention 
to feasible \smash{$\tf$} such that \smash{$\|\tf-f_0\|_n \leq
  \max\{C_n,\tJ\}$}. 

\bigskip
\noindent
{\bf Localizing the error vector.}
We prove that \smash{$\hDelta$} is appropriately bounded in the 
empirical norm.  By the tail bound for quadratic forms of sub-Gaussian random
variates in Theorem 2.1 of \citet{hsu2012tail}, for all $t > 0$, 
$$
\P\Bigg(\|\epsilon\|_n^2 > \sigma^2 \bigg(1 + \frac{2\sqrt{t}}{\sqrt{n}} +
\frac{2t}{n}\bigg)\Bigg) \leq e^{-t},
$$
and hence taking \smash{$t=\sqrt{n}$},
$$
\|\epsilon\|_n^2 \leq 5\sigma^2,
$$
on an event $\Omega_1$ with probability at least \smash{$1-\exp(-\sqrt{n})$}.
Thus returning to the basic inequality \eqref{eq:basic_ineq}, using the Cauchy-Schwartz
inequality, and the above bound, we have  
$$
\|\hDelta\|_n^2 \leq 2 \sqrt{5} \sigma \|\hDelta\|_n + 2 \lambda_n \tJ +
\|\tf-f_0\|_n^2,
$$
on $\Omega_1$.  This is a quadratic inequality of the form $x^2 \leq bx + c$ in
\smash{$x=\|\hDelta\|_n$}, so we can upper bound $x$ by the larger of the two
roots, \smash{$x \leq (b+\sqrt{b^2+4c})/2 \leq b+\sqrt{c}$}, i.e.,
$$
\|\hDelta\|_n \leq 2 \sqrt{5} \sigma + \sqrt{2 \lambda_n \tJ +
\|\tf-f_0\|_n^2},
$$
on $\Omega_1$.  Abbreviating \smash{$\sJ=\max\{C_n,\tJ\}$}, and using
\smash{$\sJ \geq 1$} (as $C_n \geq 1$ by assumption),
$$
\|\hDelta\|_n \leq \sJ \Big(2 \sqrt{5} \sigma + \sqrt{2 \lambda_n +
  \|\tf-f_0\|_n^2/(\sJ)^2} \,\Big), 
$$
on $\Omega_1$. Recalling \smash{$\|\tf-f_0\|_n \leq \sJ$}, 
and using the fact that \smash{$\lambda_n=o(1)$} for our eventual 
choice of $\lambda_n$, we have that for sufficiently large $n$,
\begin{equation}
\label{eq:delta_empirical_bd}
\|\hDelta\|_n \leq \sJ (2\sqrt{5} \sigma + \sqrt{2}),
\end{equation}
on $\Omega_1$.

\bigskip
\noindent
{\bf Bounding the sub-Gaussian complexity term.}
We focus on the first term on the right-hand side in \eqref{eq:basic_ineq},
i.e., the sub-Gaussian complexity term.  Let
\smash{$m=\hDelta/(\hJ+\sJ)$}.  By 
construction, we have $J(m) \leq 1$, and from \eqref{eq:delta_empirical_bd},
we have \smash{$\|m\|_n \leq 2\sqrt{5} \sigma +
  \sqrt{2}$} on $\Omega_1$.  Then, 
applying Lemma \ref{lem:mn_entropy_bd}, with the choice 
\smash{$\rho = 2\sqrt{5} \sigma + \sqrt{2}$}, we see that
there are constants $c_1,c_2,c_3>0$ such that for all $c \geq c_1$,  
$$
\frac{2\langle \epsilon, m \rangle_n}{\|m\|_n^{1-w/2}} \leq \frac{c}{\sqrt{n}},  
$$
on $\Omega_1 \cap \Omega_2$, where $\Omega_2$ is an event with
probability at least $1-\exp(-c_2 c^2)-\exp(-c_3 n)$.  Plugging this
into \eqref{eq:basic_ineq} gives 
$$
\|\hf - f_0 \|_n^2 + \| \hDelta \|_n^2 \leq 
\frac{c}{\sqrt{n}} (\hJ+\sJ)^{w/2} \|\hDelta\|_n^{1-w/2} + 
2 \lambda_n(\tJ-\hJ) + \|\tf-f_0\|_n^2, 
$$
on $\Omega_1 \cap \Omega_2$.  By the inequality in Lemma
\ref{lem:ab_prod_ineq}, applied to the first term on the right-hand side above,
with \smash{$a=n^{-1/2} (\hJ+\sJ)^{w/2}$} and \smash{$b=\|\hDelta\|_n$}, 
$$
\|\hf - f_0 \|_n^2 + \| \hDelta \|_n^2 \leq 
c r_n (\hJ+\sJ)^{w/(2+w)} \|\hDelta\|_n + c r_n^2 (\hJ+\sJ)^{2w/(2+w)} +
2 \lambda_n(\tJ-\hJ) + \|\tf-f_0\|_n^2,
$$
on $\Omega_1 \cap \Omega_2$, where we abbreviate 
\smash{$r_n=n^{-1/(2+w)}$}.  Applying the simple inequality $2ab \leq
a^2+b^2$ to the first term on the right-hand side, with
\smash{$a=c r_n (\hJ+\sJ)^{w/(2+w)}$} and \smash{$b=\|\hDelta\|_n$}, 
and subtracting \smash{$\|\hDelta\|_n^2/2$} from both sides,
\begin{equation}
\label{eq:basic_ineq_2}
\|\hf - f_0 \|_n^2 + \half\| \hDelta \|_n^2 \leq 
\frac{3}{2} c^2 r_n^2 (\hJ+\sJ)^{2w/(2+w)} + 2 \lambda_n(\tJ-\hJ) + 
\|\tf-f_0\|_n^2, 
\end{equation}
on $\Omega_1 \cap \Omega_2$ (where we have assumed without a
 loss of generality that $c \geq 1$).

\bigskip
\noindent
{\bf Controlling the effect of the penalty terms.}  Now we handle the
appearances of the achieved penalty term \smash{$\hJ$}.  First, 
set \smash{$\lambda_n \geq (3/4) c^2 r_n^2 / C_n^{(2-w)/(2+w)}$}, and denote  
$$
a = \frac{3}{2} c^2 r_n^2 (\hJ+\sJ)^{2w/(2+w)} + 2 \lambda_n(\tJ-\hJ). 
$$
Consider the case \smash{$\hJ \geq C_n$}.  Then \smash{$-1/C_n^{(2-w)/(2+w)}
  \geq -1/\hJ^{(2-w)/(2+w)}$}, and 
$$
2 \lambda_n(\tJ-\hJ) \leq 2\lambda_n\tJ - (3/2) c^2r_n^2
  \hJ^{2w/(2+w)},
$$
thus, using the simple inequality in Lemma \ref{lem:ab_sum_ineq}, we have 
\smash{$a \leq 4 \lambda_n \sJ$}.  In the case \smash{$\hJ
  < C_n$}, we have by Lemma \ref{lem:ab_sum_ineq} again,  
$$
a \leq \frac{3}{2} c^2 r_n^2 \Big(C_n^{2w/(2+w)} + (\sJ)^{2w/(2+w)}\Big) 
+ 2 \lambda_n\tJ \leq 6 \lambda_n \sJ.
$$
Therefore, altogether, we conclude that \smash{$a \leq 6 \lambda_n \sJ$}, and
plugging this into \eqref{eq:basic_ineq_2} gives 
$$
\|\hf - f_0 \|_n^2 + \half \|\hDelta\|_n^2 \leq 6\lambda_n \sJ +
\|\tf-f_0\|_n^2, 
$$
on $\Omega_1 \cap \Omega_2$.  The statement \eqref{eq:error_bd_j} as
made in the theorem follows by dropping the nonnegative term
\smash{$\|\hDelta\|_n^2/2$} from the left-hand side, and adjusting the
constants $c,c_1,c_2,c_3>0$ as needed.  

\subsection{Proof of the best additive approximation bound in 
  \eqref{eq:error_bd_j_best}}
\label{app:error_bd_j_best}

We follow the exact same arguments as in the proof of Theorem
\ref{thm:error_bd_j}, up until the last part, in which we control the achieved
penalty terms \smash{$\hJ$}.  Now 
we deviate from the previous arguments, slightly.
Set \smash{$\lambda_n \geq (3/2) c^2 r_n^2 / C_n^{(2-w)/(2+w)}$}, and denote  
$$
a = \frac{3}{2} c^2 r_n^2 (\hJ+\sJ)^{2w/(2+w)} + \lambda_n(\tJ-\hJ). 
$$
By the same logic as in the proof of Theorem \ref{thm:error_bd_j}, we have
\smash{$a \leq 3 \lambda_n \sJ$}. Plugging this into \eqref{eq:basic_ineq_2}
gives  
$$
\|\hf - f_0 \|_n^2 + \half \|\hDelta\|_n^2 \leq 3\lambda_n \sJ +
\lambda_n(\tJ-\hJ) + \|\tf-f_0\|_n^2, 
$$
on $\Omega_1 \cap \Omega_2$.  Rearranging,
$$
\half \|\hDelta\|_n^2 \leq 3\lambda_n \sJ + \Big(
\|\tf-f_0\|_n^2 + \lambda_n \tJ - \|\hf - f_0 \|_n^2 -\lambda_n \hJ\Big),  
$$
on $\Omega_1 \cap \Omega_2$.  But, setting \smash{$\tf=f^{\mathrm{best}}$},  
the bracketed term on the right-hand side above is nonpositive (by definition of 
\smash{$f^{\mathrm{best}}$} in \eqref{eq:add_func_j_best}). This leads to 
\eqref{eq:error_bd_j_best}, after adjusting $c,c_1,c_2,c_3>0$ as needed.  

\subsection{Preliminaries for the proof of Corollary \ref{cor:error_bd_tv}} 
\label{app:error_bd_tv_prelim}

The following two lemmas will be helpful for the proof of Corollary 
\ref{cor:error_bd_tv}.

\begin{lemma}
\label{lem:spline_approx}
Given \smash{$f=\sum_{j=1}^d f_j$}, whose
component functions are each $k$ times weakly differentiable, there
exists an additive spline approximant \smash{$\check{f}=\sum_{j=1}^d  
  \check{f}_j$}, where \smash{$\check{f}_j \in \cG_j$}, the set of $k$th order  
splines with knots in the set $T_j$ defined in \eqref{eq:tj}, for
$j=1,\ldots,d$, such that   
\begin{itemize}
\item[(i)] $\TV(\check{f}_j^{(k)}) \leq a_k \TV(f_j^{(k)})$, for $j=1,\ldots,d$;
  and 
\item[(ii)] $\|\check{f}_j - f_j \|_\infty 
  \leq a_k W_{\max}^k \TV(f_j^{(k)})$, for $j=1,\ldots,d$.
\end{itemize}
Above, $a_k \geq 1$ is a constant depending only on $k$, and  
we define \smash{$W_{\max} = \max_{j=1,\ldots,d} W_j$}, where 
$$
W_j = \max_{i=1,\ldots,n-1} \,|X^{(i)}_j -
  X^{(i+1)}_j|, \;\;\; j=1,\ldots,d.
$$
When the input points are drawn from a distribution $Q$
that satisfies Assumptions \ref{as:x_dist}, \ref{as:x_dens_bd}, there
are universal constants $c_0,n_0>0$ such that for $n \geq n_0$, 
we have \smash{$W_{\max} \leq (c_0/b_0) \log n/n$} with
probability at least $1-2b_0 d/n$, and so the bounds in (ii) 
become
\begin{equation}
\label{eq:spline_approx_bd}
\|\check{f}_j - f_j \|_\infty 
  \leq \frac{c_0^k a_k}{b_0^k} \bigg(\frac{\log n}{n} \bigg)^k   
  \TV(f_j^{(k)}), \;\;\; \text{for $j=1,\ldots,d$}, 
\end{equation}
with probability at least $1-2b_0 d/n$.
\end{lemma}

\begin{proof}
Parts (i) and (ii) are simply a componentwise application of 
Proposition 7 in \citet{mammen1997locally}.  In particular, from their 
result, we know that for $j=1,\ldots,d$, there is a $k$th degree spline  
function \smash{$\check{f}_j$} whose knots lie in $T_j$ in
\eqref{eq:tj}, with \smash{$\TV(\check{f}_j^{(k)}) \leq a_k  
  \TV(f_j^{(k)})$} and
$$
\|\check{f}_j - f_j \|_\infty \leq a_k W_j^k \TV(f_j^{(k)}), 
$$
where $a_k \geq 1$ depends only on $k$.
(This result follows from strong
quasi-interpolating properties of spline functions, from 
\citet{deboor1978practical}.)  This proves parts (i) and (ii).

When we consider random inputs drawn from a distribution $Q$
satisfying Assumptions \ref{as:x_dist}, \ref{as:x_dens_bd}, 
the densities of the marginals $Q_j$, $j=1,\ldots,d$ will be bounded
below by $b_0>0$, and thus there are universal constants $c_0,n_0 >
0$ such that for $n \geq n_0$, we have  
$W_j \leq (c_0/b_0) \log n/n$ with probability at least
$1-2b_0/n$ (see, e.g., Lemma 5 in \citet{wang2014falling}), for
$j=1,\ldots,d$, and hence applying a union bound gives the result for   
$W_{\max}$. 
\end{proof}

\begin{lemma}
\label{lem:fall_fact_approx}
Given \smash{$f=\sum_{j=1}^d f_j$}, whose
component functions are each $k$ times weakly differentiable,  
there is an additive falling factorial approximant 
\smash{$\check{f} = \sum_{j=1}^d \check{f}_j$}, where 
\smash{$\check{f}_j \in \cH_j$}, the set of $k$th order falling factorial
functions defined over $\Xj^1,\ldots,\Xj^n$, for each $j=1,\ldots,d$, such    
that 
\begin{itemize}
\item[(i)] $\TV(\check{f}_j^{(k)}) \leq a_k \TV(f_j^{(k)})$, for $j=1,\ldots,d$;
  and 
\item[(ii)] $\|\check{f}_j - f_j\|_\infty \leq a_k (W_{\max}^k + 2 k^2 W_{\max}) 
  \TV(f_j^{(k)})$, for $j=1,\ldots,d$.
\end{itemize}
Again, $a_k \geq 1$ is a constant depending only on $k$, and 
$W_{\max}$ is as defined in Lemma \ref{lem:spline_approx}.    
When the inputs are drawn from a distribution $Q$
satisfying Assumptions \ref{as:x_dist}, \ref{as:x_dens_bd}, the bound
in (ii) become
\begin{equation}
\label{eq:fall_fact_approx_bd}
\|\check{f}_j - f_j \|_\infty  
  \leq a_k \Bigg(\frac{c_0^k}{b_0^k} \bigg(\frac{\log n}{n} \bigg)^k      
  + 2k^2 \frac{c_0}{b_0}\frac{\log n}{n} \Bigg) \TV(f_j^{(k)}), \;\;\;
\text{for $j=1,\ldots,d$},
\end{equation}
with probability at least $1-2b_0 d/n$.
\end{lemma}

\begin{proof}
First we apply Lemma \ref{lem:spline_approx} to produce an additive
spline approximant, call it \smash{$f^*=\sum_{j=1}^d f_j^*$}, to
the given \smash{$f=\sum_{j=1}^d f_j$}.
Next, we parametrize the spline component functions in a helpful way:   
$$
f_j^* = \sum_{\ell=1}^n \alpha_j^\ell g_{j \ell}, \;\;\; j=1,\ldots,d.
$$
where $\alpha_j^1,\ldots,\alpha_j^n \in \R$ are coefficients and 
$g_{j1} \ldots,g_{jn}$ are the truncated power basis functions   
over the knot set $T_j$ defined in \eqref{eq:tj}, and we write
\smash{$g_{j \ell}(t) = t^{\ell-1}$}, $\ell=1,\ldots,k$ without a loss of
generality, for $j=1,\ldots,d$. It 
is not hard to check that \smash{$\TV((f_j^*)^{(k)}) =
  \sum_{\ell=k+2}^n |\alpha^\ell_j|$}, for $j=1,\ldots,d$. 

We now define \smash{$\check{f}=\sum_{j=1}^d \check{f}_j$}, 
our falling factorial approximant, to have component functions 
$$
\check{f}_j = \sum_{\ell=1}^{k+1} \alpha_j^\ell g_{j\ell} + 
\sum_{\ell=k+2}^n \alpha_j^\ell h_{j\ell}, \;\;\; j=1,\ldots,d.
$$
where $h_{j1},\ldots,h_{jn}$ are the falling factorial 
basis functions defined over $X^1_j,\ldots,X^n_j$, for
$j=1,\ldots,d$.  (Note that \smash{$\check{f}_j$} preserves
the polynomial part of \smash{$f^*_j$} exactly, for $j=1,\ldots,d$.) 
Again, it is straightforward to check that 
\smash{$\TV(\check{f}_j^{(k)}) = \sum_{\ell=k+2}^n 
  |\alpha^\ell_j|$}, for $j=1,\ldots,d$, i.e.,
$$
\TV (\check{f}_j^{(k)}) = \TV\big((f_j^*)^{(k)}\big) \leq a_k \TV(f_j^{(k)}), 
\;\;\;\text{for $j=1,\ldots,d$},
$$
the inequality coming from part (i) of of Lemma
\ref{lem:spline_approx}.  This verifies part (i) of the current
lemma.  As for part (ii), we note that Lemma 4 of
\citet{wang2014falling} shows that 
$$
|h_{j\ell}(\Xj^i)-g_{j\ell}(\Xj^i)| \leq k^2 W_j,
\;\;\;\text{for $\ell=k+2,\ldots,n$, $i=1,\ldots,n$, $j=1,\ldots,d$},
$$
where recall $W_j$ is the maximum gap between
sorted input points along the $j$th dimension, $j=1,\ldots,d$, as
defined in Lemma \ref{lem:spline_approx}. In fact, a straightforward
modification of their proof can be used to strengthen this result to   
$$
\|h_{j\ell}-g_{j\ell}\|_\infty \leq 2k^2 W_j,
\;\;\;\text{for $\ell=k+2,\ldots,n$, $j=1,\ldots,d$},
$$
which means that by Holder's inequality, 
$$
\|\check{f}_j-f_j^*\|_\infty \leq 2k^2 W_j  \sum_{\ell=k+2}^n  
  |\alpha^\ell_j| \leq 2k^2 a_k W_j \TV(f_j^{(k)})
\;\;\;\text{for $j=1,\ldots,d$}.
$$
Then, by the triangle inequality, 
$$
\|\check{f}_j - f_j \|_\infty \leq  
\|\check{f}_j-f_j^*\|_\infty + \|f_j^*-f_j\|_\infty
\leq a_k \Big( W_{\max}^k + 2k^2 W_{\max} \Big) 
\TV(f_j^{(k)}), \;\;\; \text{for $j=1,\ldots,d$},
$$
where we have used part (ii) of Lemma \ref{lem:spline_approx}.
 This verifies part (ii) of the current lemma.  

Lastly, for random inputs drawn from a distribution $Q$ 
satisfying Assumptions \ref{as:x_dist}, \ref{as:x_dens_bd}, the proof
of \eqref{eq:fall_fact_approx_bd} follows the same arguments as the proof  
of \eqref{eq:spline_approx_bd}.
\end{proof}

\subsection{Proof of Corollary \ref{cor:error_bd_tv}}
\label{app:error_bd_tv}

We consider first the statement in part (a).  We must check that
Assumptions \ref{as:poly_null}, \ref{as:deriv_bd}, \ref{as:entropy_bd} hold for
our choice of regularizer \smash{$J(g)=\TV(g^{(k)})$}, and then we can 
apply Theorem \ref{thm:error_bd_j}.
Assumptions \ref{as:poly_null}, \ref{as:deriv_bd} are immediate.
As for Assumption \ref{as:entropy_bd}, consider the univariate
function class  
$$
\cW_{k+1} = \bigg\{ f : \int_0^1 |f^{(k+1)}(t)| \, dt
\leq 1, \; \|f\|_\infty \leq 1 \bigg\}.
$$
The results in
\citet{birman1961piecewise} imply that for any set 
\smash{$Z_n=\{z^1,\ldots,z^n\} \subseteq [0,1]$},
$$
\log N (\delta, \|\cdot\|_{Z_n}, \cW_{k+1}) \leq K \delta^{-1/(k+1)},   
$$
for a universal constant $K>0$.  
As explained in \citet{mammen1991nonparametric,mammen1997locally},  
this confirms that Assumption \ref{as:entropy_bd} holds for our choice
of regularizer, with $w=1/(k+1)$.  Applying Theorem
\ref{thm:error_bd_j}, with \smash{$\tf=f_0$}, gives the result in
\eqref{eq:error_bd_tv}.  

For the statement in part (b), note first that we can consider $k
\geq 2$ without a loss of generality, as pointed out in Remark
\ref{rem:spline_approx} following the corollary. Using Lemma
\ref{lem:spline_approx}, can choose an additive spline approximant 
\smash{$\check{f}$} to $f_0$, with components \smash{$\check{f}_j \in \cG_j$},   
$j=1,\ldots,d$.  Define \smash{$\tf_j$} to be the centered version of
\smash{$\check{f}_j$}, with zero empirical mean, $j=1,\ldots,d$. By the fact
that centering does not change the penalty, and part (i) of the lemma, 
\smash{$\sum_{j=1}^d \TV(\tf_j^{(k)}) \leq a_k \sum_{j=1}^d 
  \TV(f_{0j}^{(k)})$}. Also, using the fact that centering cannot increase the
empirical norm, the triangle inequality, and 
\eqref{eq:spline_approx_bd}, we get that with probability least $1-2b_0 d/n$,    
\begin{align*}
\bigg\|\sum_{j=1}^d \tf_j - \sum_{j=1}^d f_{0j} \bigg\|_n &\leq
\bigg\|\sum_{j=1}^d \tf_j - \sum_{j=1}^d f_{0j} \bigg\|_n \\
&\leq \sum_{j=1}^d \|\tf_j - \sum_{j=1}^d f_{0j}\|_\infty \\
 &\leq \frac{c_0^k a_k}{b_0^k} \bigg(\frac{\log n}{n} \bigg)^k   
  \sum_{j=1}^d \TV(f_{0j}^{(k)}),
\end{align*}
When \smash{$\sum_{j=1}^d \TV(f_{0j}^{(k)}) \leq C_n$}, we see that 
\smash{$\|\tf-f_0\|_n$} is bounded by $C_n$ for large 
enough $n$. This meets required condition for Theorem \ref{thm:error_bd_j}, by
the above display, the approximation error in \eqref{eq:error_bd_j} satisfies  
$$
\bigg\|\sum_{j=1}^d \tf_j - \sum_{j=1}^d f_{0j} \bigg\|_n^2 \leq 
\bigg(\frac{c_0^k a_k}{b_0^k}\bigg)^2
  \bigg(\frac{\log n}{n} \bigg)^{2k} C_n^2.
$$
But when \smash{$n/(\log n)^{1+1/k} \geq n_0 C_n^{(2k+2)/(2k^2+2k-1)}$}, the
right-hand side above is upper bounded by 
\smash{$a_0 n^{-(2k+2)/(2k+3)} C_n^{2/(2k+3)}$}, for a constant $a_0>0$. 
This establishes the result in \eqref{eq:error_bd_tv} for restricted
additive locally adaptive splines.     

For the statement in part (c), we can again consider $k \geq
2$ without a loss of generality.  Then the same arguments as
given for part (b) apply here, but now we use Lemma
\ref{lem:fall_fact_approx} for the additive falling factorial approximant 
\smash{$\check{f}$} to $f_0$, and we require \smash{$n/(\log n)^{2k+3} \geq n_0 
  C_n^{4k+4}$} for the approximation error to be bounded by the
estimation error. 

\subsection{Preliminaries for the proof of Theorem \ref{thm:error_bd_j_hd}}
\label{app:error_bd_j_hd_prelim}

Our first lemma is similar to Lemma \ref{lem:mn_entropy_bd}, but concerns 
univariate functions.  As in Lemma \ref{lem:mn_entropy_bd}, this result relies
on Lemma 3.5 in \citet{vandegeer1990estimating} (see also
\citet{vandegeer2000empirical}).   

\begin{lemma}
\label{lem:bj_entropy_bd}
Let $\epsilon^i$, $i=1,\ldots,n$ be uniformly sub-Gaussian random
variables having variance proxy $\sigma^2>0$. Let $J$ be a seminorm
satisfying Assumption \ref{as:entropy_bd}, and let $\rho>0$ be arbitrary.
Then there exist constants $c_1,c_2,n_0>0$, depending only on
$\sigma,K,w,\rho$, such that for all $c \geq c_1$ and $n \geq n_0$,     
$$
\sup_{Z_n = \{z^1,\ldots,z^n\} \subseteq [0,1]} \,
\sup_{g \in B_J(1) \cap B_\infty(\rho)} \,
\frac{\frac{1}{n} \sum_{i=1}^n \epsilon^i g(z^i)} 
{\|g\|_{Z_n}^{1-w/2}} \leq \frac{c}{\sqrt{n}}, 
$$
with probabilty at least $1-\exp(-c_2 c^2)$, where we write
\smash{$\|\cdot\|_{Z_n}$} for the empirical norm defined over a set of
univariate points $Z_n=\{z^1,\ldots,z^n\} \subseteq [0,1]$.
\end{lemma}

\begin{proof}
Assume without a loss of generality that $\rho \geq 1$. Note that 
for any $Z_n=\{z^1,\ldots,z^n\} \subseteq [0,1]$,
$$
\log N\big(\delta, \|\cdot\|_{Z_n}, B_J(1) \cap B_\infty(\rho) \big) 
\leq K \rho^w \delta^{-w},
$$
by Assumption \ref{as:entropy_bd}. As the right-hand side in the
above entropy bound does not depend on $Z_n$, we can apply Lemma 3.5 in
\citet{vandegeer1990estimating} to get the desired uniform control over all
subsets. 
\end{proof}

We give a coupling between the empirical and $L_2$ norms over  
$B_J(1) \cap B_\infty(\rho)$, using Theorem 14.1 in 
\citet{wainwright2017high} (see also
\citet{vandegeer2000empirical,bartlett2005local,raskutti2012minimax}). 

\begin{lemma}
\label{lem:bj_L2_norm}
Let $z^i$, $i=1,\ldots,n$ denote an i.i.d.\ sample from a distribution 
$\Lambda$ on $[0,1]$.  Write $\|\cdot\|_2$ for the $L_2(\Lambda)$ norm, and
$\|\cdot\|_n$ for the $L_2(\Lambda_n)$ norm.  Let $J$ satisfy
Assumption \ref{as:entropy_bd}, and let $\rho>0$ be arbitrary.  Then there
are constants $c_1,c_2,c_3,n_0>0$, that depend only on 
$K,w,\rho$, such that for any \smash{$t \geq c_1 n^{-1/(2+w)}$} and $n \geq
n_0$,     
$$
\big|\|g\|^2_n - \|g\|_2^2\big| \leq \half\|g\|_2^2 + \frac{t^2}{2},  
\;\;\; \text{for all $g \in B_J(1) \cap B_\infty(\rho)$},
$$
with probability at least $1-c_2\exp(-c_3nt^2)$.  
\end{lemma}

\begin{proof}
Abbreviate $\cF=B_J(1) \cap B_\infty(\rho)$.  We will analyze the local Rademacher 
complexity    
$$
\cR\big(\cF \cap B_2(t)\big) = \E_{z,\sigma} \left[   
\sup_{g \in \cF \cap B_2(t)} \,
\frac{1}{n} \Bigg|\sum_{i=1}^n \sigma^i g(z^i) \Bigg|\right],
$$
the expectation being taken over i.i.d.\ draws $z^i$, $i=1,\ldots,n$ from
$\Lambda$ and i.i.d.\ Rademacher variables $\sigma^i$, $i=1,\ldots,n$, as
usual. Define the critical radius $\tau_n>0$ to the smallest solution of the
equation  
$$
\frac{\cR\big(\cF \cap B_2(t)\big)}{t} = \frac{t}{\rho}.
$$
We will prove \smash{$\tau_n \leq c_1 n^{-1/(2+w)}$} for a constant
$c_1>0$. Applying Theorem 14.1 in \citet{wainwright2017high} would
then give the result. 

In what follows, we will use $c>0$ to denote a constant whole value may  
change from line to line (but does not depend on $z^i$,
$i=1,\ldots,n$).  Consider the empirical local Rademacher complexity    
$$
\cR_n\big(\cF \cap B_2(t)\big) = \E_\sigma \left[  
\sup_{g \in \cF \cap B_2(t)} \,
\frac{1}{n} \Bigg|\sum_{i=1}^n \sigma^i g(z^i)\Bigg| \right].
$$
As we are considering $t \geq \tau_n$, Corollary 2.2 of
\citet{bartlett2005local} gives 
$$
\cF \cap B_2(t) \subseteq \cF \cap B_n(\sqrt{2} t),
$$
with probability at least $1-1/n$. Denote by $\cE$ the event that this 
occurs.  Then on $\cE$,
\begin{align*}
\cR_n\big(\cF \cap B_2(t)\big) &\leq \E_\sigma \left[ 
\sup_{g \in \cF \cap B_n(\sqrt{2} t)} \,
\frac{1}{n}\Bigg|\sum_{i=1}^n \sigma^i g(z^i) \Bigg| \right] \\ 
&\leq \frac{c}{\sqrt{n}} \int_0^{\sqrt{2}t}
\sqrt{\log N(\delta, \|\cdot\|_n, \cF)} \, d\delta \\
&\leq \frac{c\sqrt{K}\rho^{w/2}}{\sqrt{n}} 
\int_0^{\sqrt{2}t} \delta^{-w/2} \, d\delta =
\frac{c}{\sqrt{n}} t^{1-w/2},
\end{align*}
where in second line we used Dudley's entropy integral \citep{dudley1967sizes}, 
and in the third line we used Assumption \ref{as:entropy_bd}. On
$\cE^c$, note that we have the trivial bound \smash{$\cR_n(\cF \cap B_2(t)) \leq
  \rho$}.  Therefore we can upper bound the local Rademacher complexity,
splitting the expectation over $\cE$ and $\cE^c$, 
$$
\cR\big(\cF \cap B_2(t)\big) = \E_z \cR_n\big(\cF \cap B_2(t)\big)   
\leq \frac{c t^{1-w/2}}{\sqrt{n}} + \frac{\rho}{n} \leq
\frac{c t^{1-w/2}}{\sqrt{n}}, 
$$
where the second inequality holds when $n$ is large enough, as we may assume
\smash{$t \geq n^{-1/2}$} without a loss of generality.   An upper bound on the
critical radius $\tau_n$ is thus given by the solution of    
$$
\frac{c t^{-w/2}}{\sqrt{n}} = \frac{t}{\rho},
$$
which is \smash{$t = c n^{-1/(2+w)}$}. This completes the proof.        
\end{proof}

We extend Lemma \ref{lem:pi_perp_sup_bd} to give a uniform sup norm bound on the   
functions in $B_J(1) \cap B_2(\rho)$.   

\begin{lemma}
\label{lem:bj_b2_sup_bd}
Assume the conditions of Lemma \ref{lem:pi_perp_sup_bd}.  Then there are
constants $R_1,R_2>0$ that depend only on $k,b_1,b_2,L$, such that   
$$
\|m\|_\infty \leq R_1 \rho + R_2, \;\;\; 
\text{for all $m \in B_J(1) \cap B_2(\rho)$}.
$$ 
\end{lemma}

\begin{proof}
For $m \in B_J(1) \cap B_2(\rho)$, decompose $m = p + g$ where $p$ is a
polynomial of degree $k$, and $g$ is orthogonal to all polynomials of degree $k$
with respect to the $L_2(\Lambda)$ inner product. By Lemma
\ref{lem:pi_perp_sup_bd}, we have $\|g\|_\infty \leq R_0$ for a constant
$R_0>0$, and by the triangle inequality,   
$$
\|p\|_2 \leq \|m\|_2 + \|g\|_2 \leq \rho + R_0.
$$
Now write 
$$
p(x) = \sum_{\ell=1}^{k+1} \alpha_\ell \phi_\ell(x), \;\;\; 
\text{for $x \in [0,1]^d$}, 
$$
where $\phi_\ell$, $\ell=1,\ldots,k+1$ are orthonormal polynomials on $[0,1]$
with respect to the $L_2(\Lambda)$ inner product. Then 
$\|\alpha\|_2=\|p\|_2 \leq \rho + R_0$, from the second to last display, and  
\smash{$\|\alpha\|_2 \leq \sqrt{k+1} (\rho + R_0)$}, so for any $x \in [0,1]$,
$$
|p(x)| \leq \|\alpha\|_1 \max_{\ell=1,\ldots,k+1} |\phi_\ell(x)| 
\leq c_k \sqrt{k+1} (\rho + R_0),
$$
where \smash{$c_k = \max_{\ell=1,\ldots,k+1} \|\phi_\ell\|_\infty$} is a  
constant that depends only on $k,b_1$ from \cite{aptekarev2016orthogonal}.  
Therefore 
$$
\|m\|_\infty \leq \|p\|_\infty + \|g\|_\infty \leq  c_k \sqrt{k+1}
(\rho + R_0) + R_0,
$$
and defining $R_1,R_2>0$ appropriately, this is of the desired form, and
completes the proof.
\end{proof}

Our last two lemmas pertain to the function space
\begin{equation}
\label{eq:m2_space} 
\cM_2(\delta) = \bigg\{ \sum_{i=1}^d m_j : J(m_j) \leq \delta, \;\,  
\text{and}\;\, \langle m_j,1 \rangle_2=0, \; j=1,\ldots,d \bigg\}.
\end{equation}
We derive a one-sided bound on the $L_2$ norm in terms of the   
empirical norm, over $\cM_2(1)$.  Our proof uses Theorem 14.2 in
\citet{wainwright2017high}, which is a somewhat unique theorem, because it does
not require a global sup norm bound on the function class in consideration
(unlike many standard results of this type).

\begin{lemma}
\label{lem:m2_L2_norm}
Let $X^i$, $i=1,\ldots,n$ denote an i.i.d.\ sample from a distribution
$Q$ on $[0,1]^d$ satisfying Assumption \ref{as:x_dist_prod},
 and let $J$ satisfy Assumption \ref{as:entropy_bd}.
Then there are constants $c_1,c_2,c_3,n_0>0$, that depend only on
$b_1,b_2,k,L,K,w$,        
such that for any \smash{$c_1 \sqrt{d} n^{-1/(2+w)} \leq t \leq 1$} and  
$n \geq n_0$, 
$$
\|m\|_2^2 \leq 2 \|m\|_n^2 + t^2, \;\;\; 
\text{for all $m \in \cM_2(1)$},
$$
with probability at least $1-c_2\exp(-c_3 nt^2)$, where $\cM_2(1)$ 
is the space in \eqref{eq:m2_space}.
\end{lemma}

\begin{proof}
Let $m \in \cM_2(1)$ with $\|m\|_2 \leq 1$.  Then as $\|m\|_2^2 = \sum_{j=1}^d
\|m_j\|_2^2$, it follows that $\|m_j\|_2 \leq 1$, $j=1,\ldots,d$, and by Lemma
\ref{lem:bj_b2_sup_bd}, we have $\|m_j\|_\infty \leq R_1+R_2$, $j=1,\ldots,d$. 
From the calculation in Example 14.6 of \citet{wainwright2017high}, we have
the property 
$$
\|m^2\|_2^2 \leq C^2 \|m\|_2^4, \;\;\; \text{for all $m \in \cM_2(1) \cap
  B_2(1)$},
$$
where $C^2 = (R_1+R_2)^2+6$.  Abbreviating $\cF=\cM_2(1)$, we will study the
local Rademacher complexity       
$$
\cR\big(\cF \cap B_2(t)\big) = \E_{z,\sigma} \left[     
\sup_{m \in \cF \cap B_2(t)} \, 
\frac{1}{n} \Bigg|\sum_{i=1}^n \sigma^i m(z^i) \Bigg|\right],
$$
and the associated critical radius $\tau_n>0$, defined as usual to be the
smallest solution of 
$$
\frac{\cR\big(\cF \cap B_2(t)\big)}{t} = \frac{t}{C}. 
$$
We will establish \smash{$\tau_n \leq c_1 \sqrt{d} n^{-1/(2+w)}$} for a 
constant $c_1>0$. Applying Theorem 14.2 in \citet{wainwright2017high}
would then give the result.    

For the rest of the proof, we will use $c>0$ for a constant whose value may 
change from line to line; also, many statements will hold for large enough $n$, 
but this will not always be made explicit.  Fix some $0 < t \leq 1$. By $L_2$
orthogonality of the components of functions in $\cF$,   
\begin{align*}
\sup_{m \in \cF \cap B_2(t)} \, \frac{1}{n}  
\Bigg| \sum_{i=1}^n \sigma^i m(X^i) \Bigg| 
&\leq \sup_{\|\beta\|_2 \leq t} \,   
\sup_{\substack{m_j \in B_J(1) \cap B_2(|\beta_j|),  \\ 
 j=1,\ldots,d}} \,      
\Bigg| \sum_{i=1}^n \sigma^i \sum_{j=1}^d m_j(X_j^i) \Bigg| \\ 
&\leq \sup_{\|\beta\|_2 \leq t} \,  
\sum_{j=1}^d \sup_{m_j \in B_J(1) \cap B_2(|\beta_j|)} \,
\frac{1}{n} \Bigg| \sum_{i=1}^n \sigma^i m_j(X_j^i) \Bigg|.  
\end{align*}
We now bound the inner supremum above, for an arbitrary 
$j=1,\ldots,d$. Denote by $\tau_{nj}$ the critical radius 
of $B_J(1) \cap B_2(|\beta_j|)$, denote \smash{$r_n=n^{-1/(2+w)}$}, and define
the abbreviation \smash{$a \vee b = \max\{a,b\}$}. Observe  
\begin{align*}
\sup_{m_j \in B_J(1) \cap B_2(|\beta_j|)}
  \hspace{-30pt}&\hspace{30pt}\, 
\frac{1}{n} \Bigg| \sum_{i=1}^n \sigma^i m_j(X^i_j) \Bigg| \\  
&\leq c\Bigg( \cR_n\big(B_J(1) \cap
  B_2(|\beta_j|)\big)  + \sqrt{\frac{\log n}{n}} \bigg(\sup_{m_j \in
  B_J(1) \cap B_2(|\beta_j|)} \, \|m_j\|_n \bigg) \Bigg) \\     
& \leq c\bigg( \cR\big(B_J(1) \cap
  B_2(|\beta_j|)\big) +  \frac{\log n}{n} + \sqrt{\frac{\log n}{n}}
  \bigg(\sup_{m_j \in B_J(1) \cap B_2(|\beta_j|)} \,
  \|m_j\|_n \bigg) \Bigg) \\     
& \leq c\bigg( \cR\big(B_J(1) \cap
  B_2(|\beta_j|)\big) +  \frac{\log n}{n} + \sqrt{\frac{\log
  n}{n}} \sqrt{2} (|\beta_j| \vee \tau_{nj}) \bigg) \\      
& \leq c\bigg( \frac{|\beta_j|^{1-w/2}}{\sqrt{n}} +
  \frac{\log n}{n} + (|\beta_j| \vee \tau_{nj}) 
  \sqrt{\frac{\log n}{n}} \bigg) \\     
& \leq c\bigg( \frac{|\beta_j|^{1-w/2}}{\sqrt{n}} + 
  (|\beta_j| \vee r_n) \sqrt{\frac{\log n}{n}}\bigg). 
\end{align*}
The first three inequalities above
hold with probability at least $1-1/3n^2$ each. The first inequality
is by Theorem 3.6 in \citet{wainwright2017high} (see also Example 3.9
in \citet{wainwright2017high}); the second and third are by Lemma
A.4 and Lemma 3.6 in \citet{bartlett2005local}, respectively. The fourth 
upper bounds the local Rademacher complexity of $B_J(1) \cap B_2(|\beta_j|)$,
and the fifth upper bounds the critical radius $\tau_{nj}$ of this class, both 
following the proof of Lemma \ref{lem:bj_L2_norm} (recall, the functions in
$B_J(1) \cap B_2(|\beta_j|)$ have a uniform sup norm bound of $\rho=R_1+R_2$, 
by Lemma \ref{lem:bj_b2_sup_bd}).
The last step also uses \smash{$\log n/n \leq r_n \sqrt{\log n/n}$} for $n$
sufficiently large. The final result of the above display holds with probability
at least $1-1/n^2$; by a union bound, it holds with probability at 
least $1-d/n^2$ simultaneously over $j=1,\ldots,d$.  Call 
this event $\cE$. Then on $\cE$,
\begin{align}
\nonumber
\sup_{m \in \cF \cap B_2(t)} 
\frac{1}{n} \Bigg| \sum_{i=1}^n \sigma^i m(X^i) \Bigg| 
&\leq c \sum_{j=1}^d
  \bigg( \frac{|\beta_j|^{1-w/2}}{\sqrt{n}} +  
  (|\beta_j| \vee r_n) \sqrt{\frac{\log n}{n}}\bigg) \\   
\label{eq:f_rad_bd_e}
&\leq c \bigg( \frac{d^{(2+w)/4} t^{1-w/2}}{\sqrt{n}}  
+ \sqrt{\frac{d \log n}{n}} t + d r_n^2 \bigg).
\end{align}
In the second line, we use Holder's inequality \smash{$a^T b \leq \|a\|_p
  \|b\|_q$} for the first term, with $p=4/(2+w)$ and $q=4/(2-w)$;
we use $a \vee b \leq a+b$ for the second term, along the
bound \smash{$\|\beta\|_1 \leq \sqrt{d} t$}, and the fact that \smash{$r_n 
  \sqrt{\log n/n} \leq r_n^2$} for large enough $n$.

Meanwhile, on $\cE^c$, we can apply the simple bound
\smash{$\|m\|_\infty \leq \sum_{j=1}^d \|m_j\|_\infty \leq \rho d$} for
functions in $\cF \cap B_2(t)$, where $\rho=R_1+R_2$ (owing to Lemma
\ref{lem:bj_b2_sup_bd}), and thus
\begin{equation}
\label{eq:f_rad_bd_ec}
\sup_{m \in \cF \cap B_2(t)} \, \frac{1}{n}  
\Bigg| \sum_{i=1}^n \sigma^i m(X^i) \Bigg| \leq \rho d.
\end{equation}
Splitting the expectation defining the local Rademacher 
complexity over $\cE,\cE^c$, and using \eqref{eq:f_rad_bd_e}, 
\eqref{eq:f_rad_bd_ec},
\begin{align}
\nonumber
\cR\big(\cF \cap B_2(t)\big) &=   
\E_{X,\sigma} \left[\sup_{m \in \cF \cap B_2(t)} \,
  \frac{1}{n} \Bigg| \sum_{i=1}^n \sigma^i m(X^i) \Bigg|\right] \\ 
\label{eq:f_rad_bd}
&\leq c \bigg( \frac{d^{(2+w)/4} t^{1-w/2}}{\sqrt{n}}  
+ \sqrt{\frac{d \log n}{n}} t + d r_n^2\bigg) + \frac{\rho d^2}{n^2}. 
\end{align}
It can be easily verified that for \smash{$t=c \sqrt{d} r_n^2$}, the upper bound
in \eqref{eq:f_rad_bd} is at most $t^2/C$.  Therefore this is an upper bound on
the critical radius of $\cF$, which completes the proof.
\end{proof}

Lastly, we bound the gap in the empirical and $L_2$ means of functions 
in \smash{$\cM_2(1) \cap B_2(t)$}, for small enough $t$.  The proof uses
Theorem 2.1 in \citet{bartlett2005local}.

\begin{lemma}
\label{lem:m2_L2_mean}
Let $X^i$, $i=1,\ldots,n$ denote an i.i.d.\ sample from a distribution
$Q$ on $[0,1]^d$ satisfying Assumption \ref{as:x_dist_prod},
 and let $J$ satisfy Assumption \ref{as:entropy_bd}.
There are constants $c_0,n_0>0$, that depend only on $b_1,b_2,k,L,K,w$,  
such that for any $0< t \leq 1$ and $n \geq n_0$,
$$
\big|\langle m,1 \rangle_n  - \langle m,1 \rangle_2 \big| \leq 
c_0 \bigg( \frac{d^{(2+w)/4} t^{1-w/2}}{\sqrt{n}} + 
\sqrt{\frac{d \log n}{n}} t + d n^{-2/(2+w)} \bigg),
\;\;\; \text{for all $m \in \cM_2(1) \cap B_2(t)$},   
$$
with probability at least $1-1/n$, where $\cM_2(1)$ 
is the space in \eqref{eq:m2_space}.
\end{lemma} 

\begin{proof}
This follows by combining the local Rademacher bound in \eqref{eq:f_rad_bd} from 
the proof of Lemma \ref{lem:m2_L2_norm} with Theorem 2.1 in
\citet{bartlett2005local}, and simplifying by keeping the dominant terms for
large enough $n$.
\end{proof}

\subsection{Proof of Theorem \ref{thm:error_bd_j_hd}}
\label{app:error_bd_j_hd}

At a high-level, the difference between this proof and that of Theorem
\ref{thm:error_bd_j} is that here we do not try to directly control the
sub-Gaussian complexity term (as this would lead to a poor dependence
on the dimension $d$).  Instead, we reduce the 
problem to controlling univariate sub-Gaussian complexities, and then assemble
the result using ties between the empirical and $L_2$ norms, and the
decomposition property \eqref{eq:L2_decomp}. We will use the same
general notation as in the proof of Theorem \ref{thm:error_bd_j}:
$c>0$ denotes a constant that will multiply our final bound, and will control
the probability with which the final result holds; we will use the empirical
norms and inner products with vector arguments, to be interpreted appropriately;
we use the abbreviations \smash{$r_n,\hDelta$}, and so on.  Finally, in many
lines that follow, we will redefine $c$ by absorbing constant factors into it,
without explicit notice.  

The same arguments that led us to \eqref{eq:basic_ineq} yield the basic inequality 
\begin{equation}
\label{eq:basic_ineq_3}
\|\hf - f_0 \|_n^2 + \| \hDelta \|_n^2 
\leq 2 \langle \epsilon, \hDelta \rangle_n + \|\tf-f_0\|_n^2 
= 2 \sum_{j=1}^d \langle \epsilon, \hDelta_j \rangle_n + \|\tf-f_0\|_n^2, 
\end{equation}
where we write \smash{$\hDelta=\sum_{j=1}^d \hDelta_j$}. 

\bigskip
\noindent
{\bf Bounding the sub-Gaussian complexity terms.}  We now bound the univariate 
sub-Gaussian complexity terms, appearing in the sum on the right-hand side in 
\eqref{eq:basic_ineq_3}.  For $j=1,\ldots,d$, define 
\smash{$g_j = \hDelta_j/(2\delta + \|\hDelta_j\|_n)$}, and note that by
construction $J(g_j) \leq 1$ and $\|g_j\|_n \leq 1$. By Lemma
\ref{lem:mn_sup_bd}, there are constants $c_0,R>0$ such that $\|g_j\|_\infty 
\leq R$ on an event whose probability is at least $1-\exp(-c_0n)$.  Thus
by Lemma \ref{lem:bj_entropy_bd}, there are constants $c_1,c_2>0$ such
that for all $c \geq c_1$, 
$$
\frac{2\langle \epsilon, g_j \rangle_n}{\|g_j\|_n^{1-w/2}} \leq
  \frac{c}{\sqrt{n}}, \;\;\; \text{for all $j=1,\ldots,d$},
$$
on an event $\Omega_1$ with probability at least $1-\exp(c_0n)-\exp(-c_2c^2)$.  
Plugging this into \eqref{eq:basic_ineq_3} gives
\begin{align}
\nonumber
\|\hf - f_0 \|_n^2 + \| \hDelta \|_n^2 
&\leq \frac{c}{\sqrt{n}} \sum_{j=1}^d (2\delta + \|\hDelta_j\|_n)^{w/2}  
\|\hDelta_j\|_n^{1-w/2} + \|\tf-f_0\|_n^2, \\
\label{eq:basic_ineq_4}
&\leq \frac{c\delta^{w/2}}{\sqrt{n}} \sum_{j=1}^d \|\hDelta_j\|_n^{1-w/2} +
\frac{c}{\sqrt{n}} \sum_{j=1}^d \|\hDelta_j\|_n + \|\tf-f_0\|_n^2, 
\end{align}
on $\Omega_1$, where we used Lemma \ref{lem:ab_sum_ineq} in the second
inequality.  

\bigskip
\noindent
{\bf Converting empirical norms into $L_2$ norms.}  For each $j=1,\ldots,d$, let  
\smash{$\bDelta_j=\langle \hDelta_j,1 \rangle_2$} be the $L_2$ mean of
\smash{$\hDelta_j$}, and \smash{$\tDelta_j = \hDelta_j-\bDelta_j$} 
the $L_2$ centered version of \smash{$\hDelta_j$}.  Note that, for each
$j=1,\ldots,d$, we have by empirical orthogonality \smash{$\|\tDelta_j\|_n^2 =   
  \|\hDelta_j\|_n^2 + |\bDelta_j|^2$}, which implies
\smash{$\|\hDelta_j\|_n \leq \|\tDelta_j\|_n$}.  Applying this to upper bound
the right-hand side in \eqref{eq:basic_ineq_4} gives  
\begin{equation}
\label{eq:basic_ineq_5}
\|\hf - f_0 \|_n^2 + \| \hDelta \|_n^2 
\leq \frac{c\delta^{w/2}}{\sqrt{n}} \sum_{j=1}^d \|\tDelta_j\|_n^{1-w/2} +
\frac{c}{\sqrt{n}} \sum_{j=1}^d \|\tDelta_j\|_n + \|\tf-f_0\|_n^2, 
\end{equation}
on $\Omega_1$. We bound each empirical norm in the sum on the right-hand side in
\eqref{eq:basic_ineq_5} by its $L_2$ norm counterpart.  Now, for each
$j=1,\ldots,d$, define \smash{$g_j = \tDelta_j/(2\delta + \|\tDelta_j\|_2)$}.  
Since $J(g_j) \leq 1$ and $\|g_j\|_2 \leq 1$, by Lemma \ref{lem:bj_b2_sup_bd},
there is a constant $R>0$ such that $\|g_j\|_\infty \leq R$. We can hence apply
Lemma \ref{lem:bj_L2_norm} to the measure $\Lambda=Q_j$, which gives constants
$c_3,c_4,c_5>0$ such that  
$$
\|g_j\|_n \leq \sqrt{\frac{3}{2}} \|g_j\|_2 + c_3 r_n, 
\;\;\; \text{for all $j=1,\ldots,d$},
$$
on an event $\Omega_2$ with probability at least $1-c_4d\exp(-c_5nr_n^2)$,
where recall \smash{$r_n=n^{-1/(2+w)}$}, i.e.,
$$
\|\tDelta_j\|_n \leq 2 \sqrt{\frac{3}{2}} \|\tDelta_j\|_2 + 2 c_3 r_n \delta,   
\;\;\; \text{for all $j=1,\ldots,d$}, 
$$
on $\Omega_2$, where we assume $n$ is large enough so that 
\smash{$c_3 r_n \leq \sqrt{3/2}$}. 
Returning to \eqref{eq:basic_ineq_5}, and using the simple inequality in 
Lemma \ref{lem:ab_sum_ineq}, we have    
\begin{equation} 
\label{eq:basic_ineq_6}
\|\hf - f_0 \|_n^2 + \| \hDelta \|_n^2 
\leq \frac{c\delta^{w/2}}{\sqrt{n}} \sum_{j=1}^d \|\tDelta_j\|_2^{1-w/2} +
\frac{c}{\sqrt{n}} \sum_{j=1}^d \|\tDelta_j\|_2 + cdr_n^2 \delta +
\|\tf-f_0\|_n^2,
\end{equation}
on $\Omega_1 \cap \Omega_2$. 

\bigskip
\noindent
{\bf Invoking $L_2$ decomposability.}  We recall the key $L_2$ decomposition 
property \eqref{eq:L2_decomp}, of additive functions with $L_2$ mean
zero components.  Using Holder's inequality \smash{$a^T b \leq \|a\|_p \|b\|_q$}
to bound the first sum on the right-hand side in \eqref{eq:basic_ineq_6}, with 
$p=4/(2+w)$ and $q=4/(2-w)$, and Cauchy-Schwartz to bound the second sum in
\eqref{eq:basic_ineq_6}, we get 
\begin{equation}
\label{eq:basic_ineq_7}
\|\hf - f_0 \|_n^2 + \| \hDelta \|_n^2 
\leq \frac{c d^{(2+w)/4} \delta^{w/2}}{\sqrt{n}} \|\tDelta\|_2^{1-w/2} + 
c \sqrt{\frac{d}{n}} \|\tDelta\|_2 + cdr_n^2 \delta + 
\|\tf-f_0\|_n^2,
\end{equation}
on $\Omega_1 \cap \Omega_2$, where we denote \smash{$\tDelta = \sum_{j=1}^d
  \tDelta_j$}. 

\bigskip
\noindent
{\bf Converting back to empirical norm.}  We bound the $L_2$ norm of the
centered error vector on the right-hand side in \eqref{eq:basic_ineq_7} with its
empirical norm counterpart.  By Lemma \ref{lem:m2_L2_norm} applied to
\smash{$m=\tDelta/(2\delta)$}, 
provided $n$ is large enough so that \smash{$c_6 \sqrt{d} r_n \leq 1$},
there are constants $c_6,c_7,c_8>0$ such that  
\begin{equation}
\label{eq:delta_L2_bd}
\|\tDelta\|_2 \leq \sqrt{2} \|\tDelta\|_n + 2 c_6 \sqrt{d} r_n \delta, 
\end{equation}
on an event $\Omega_3$ with probability at least $1-c_7\exp(-c_8 d n r_n^2)$.
Plugging this into the right-hand side in \eqref{eq:basic_ineq_7}, and using  
Lemma \ref{lem:ab_sum_ineq}, we have     
$$
\|\hf - f_0 \|_n^2 + \| \hDelta \|_n^2 
\leq \frac{c d^{(2+w)/4} \delta^{w/2}}{\sqrt{n}} \|\tDelta\|_n^{1-w/2} + 
c \sqrt{\frac{d}{n}} \|\tDelta\|_n + cdr_n^2 \delta + 
\|\tf-f_0\|_n^2,
$$
on $\Omega_1 \cap \Omega_2 \cap \Omega_3$.  Using Lemma \ref{lem:ab_prod_ineq}
on the first term above, with \smash{$a=d^{(2+w)/4} \delta^{w/2} / \sqrt{n}$}
and $b=\|\tDelta\|_n$, and simplifying, gives
\begin{equation}
\label{eq:basic_ineq_8}
\|\hf - f_0 \|_n^2 + \| \hDelta \|_n^2 
\leq c \sqrt{d} r_n \sqrt\delta \|\tDelta\|_n + 
cdr_n^2 \delta + \|\tf-f_0\|_n^2,
\end{equation}
on $\Omega_1 \cap \Omega_2 \cap \Omega_3$.

\bigskip
\noindent
{\bf Deriving an empirical norm error bound.} Note that in
\eqref{eq:basic_ineq_8}, we have \smash{$\|\hDelta\|_n$} on the
left-hand side and \smash{$\|\tDelta\|_n$} on the right-hand side, where  
\smash{$\tDelta=\hDelta-\bDelta$} is the centered error vector, and we are
abbreviating \smash{$\bDelta=\sum_{j=1}^d \bDelta_j$}.  We seek to bound
\smash{$|\bDelta|$}. Define \smash{$t = c_6 \sqrt{d} r_n$}, where
$c_6$ is the constant in \eqref{eq:delta_L2_bd}, and define  
$$
m = \frac{t \tDelta/(2\delta)}
{\sqrt{2} \|\tDelta\|_n / (2\delta) + t}.  
$$
Note that \smash{$J(m_j) \leq J(\tDelta_j)/(2\delta) \leq 1$}, for
$j=1,\ldots,d$, by construction, and also  
$$
\|m\|_2 = \frac{t \|\tDelta\|_2/(2\delta)}
{\sqrt{2} \|\tDelta\|_n/(2\delta) + t} \leq t, 
$$
on $\Omega_1 \cap \Omega_2 \cap \Omega_3$, recalling 
\eqref{eq:delta_L2_bd}.  By Lemma \ref{lem:m2_L2_mean}
applied to $m$, provided $n$ is large enough such that 
\smash{$t=c_6 \sqrt{d} r_n \leq 1$}, 
there is a constant $c_9>0$ such that 
\smash{$|\langle m,1 \rangle_n| \leq c_9 t^2$} on 
$\Omega_1 \cap \Omega_2 \cap \Omega_3 \cap \Omega_4$, 
where $\Omega_4$ is an event with probability at least  
$1-1/n$, i.e.,  
$$
|\langle 1, \tDelta \rangle_n |/ (2\delta) \leq 
c_9 t \big(\sqrt{2} \|\tDelta\|_n/(2\delta) + t\big) ,
$$
on $\Omega_1 \cap \Omega_2 \cap \Omega_3 \cap \Omega_4$,
i.e., 
$$
|\bDelta| \leq \sqrt{2} c_9 t \|\tDelta\|_n + 2 c_9 t^2 \delta,
$$
on $\Omega_1 \cap \Omega_2 \cap \Omega_3 \cap \Omega_4$.
Thus, by empirical orthogonality,
$$
\|\tDelta\|_n^2 = \|\hDelta\|_n^2 + |\bDelta|^2 \leq 
\|\hDelta\|_n^2 + 2 (\sqrt{2} c_9 t)^2 \|\tDelta\|_n^2 + 
2 (2 c_9 t^2 \delta)^2,   
$$
on $\Omega_1 \cap \Omega_2 \cap \Omega_3 \cap \Omega_4$,
and assuming $n$ is large enough so that \smash{$2(\sqrt{2} c_9 t)^2 \leq 1/2$} 
and \smash{$2(2c_9)^2 t^2 \delta \leq 1$}, this becomes 
\begin{equation}
\label{eq:tdelta_bd}
\frac{1}{2} \|\tDelta\|_n^2 \leq \|\hDelta\|_n^2 + t^2 \delta, 
\end{equation}
on $\Omega_1 \cap \Omega_2 \cap \Omega_3 \cap \Omega_4$.
Using this on the right-hand side in \eqref{eq:basic_ineq_8} gives 
$$
\|\hf - f_0 \|_n^2 + \| \hDelta \|_n^2 
\leq c \sqrt{d} r_n \sqrt\delta \|\hDelta\|_n + 
cdr_n^2 \delta + \|\tf-f_0\|_n^2,
$$
on $\Omega_1 \cap \Omega_2 \cap \Omega_3 \cap \Omega_4$.
Using the simple inequality $2ab \leq a^2 + b^2$ on the first term on the
right-hand side above, with \smash{$a=c \sqrt{d} r_n \sqrt\delta$} and 
\smash{$b=\|\hDelta\|_n$}, gives 
\begin{equation}
\label{eq:basic_ineq_9}
\|\hf - f_0 \|_n^2 + \frac{1}{2}\| \hDelta \|_n^2 \leq 
\|\tf-f_0\|_n^2 + c^2dr_n^2 \delta,
\end{equation}
on $\Omega_1 \cap \Omega_2 \cap \Omega_3 \cap \Omega_4$.
The empirical norm result in \eqref{eq:error_bd_j_hd_empirical} in the 
theorem follows by dropping the nonnegative term
\smash{$\|\hDelta\|_n^2/2$} from the left-hand side, and adjusting the
constants $c,c_1,c_2,c_3>0$ as needed.    

\bigskip
\noindent
{\bf Deriving an $L_2$ norm error bound.} Note that \eqref{eq:basic_ineq_9} also
implies 
$$
\frac{1}{2}\| \hDelta \|_n^2 \leq \|\tf-f_0\|_n^2 + c^2dr_n^2 \delta,
$$
on $\Omega_1 \cap \Omega_2 \cap \Omega_3 \cap \Omega_4$.  Recalling
\eqref{eq:tdelta_bd}, this gives 
\begin{equation}
\label{eq:basic_ineq_10}
\| \tDelta \|_n^2 \leq 4 \|\tf-f_0\|_n^2 + c^2dr_n^2 \delta,
\end{equation}
on $\Omega_1 \cap \Omega_2 \cap \Omega_3 \cap \Omega_4$. 
By $L_2$ orthogonality,
\begin{align*}
\|\hDelta\|_2^2 &= \|\tDelta\|_2^2 + |\bDelta|^2 \\
&\leq 3 \|\tDelta\|_n^2 + t^2 \delta^2 \\
&\leq 12 \|\tf-f_0\|_n^2 + c^2 d r_n^2 \delta^2,
\end{align*}
on $\Omega_1 \cap \Omega_2 \cap \Omega_3 \cap \Omega_4$,
where in the second line we used \eqref{eq:delta_L2_bd} and
\smash{$|\bDelta| \leq \|\tDelta\|_n$}, and in the third line we used
\eqref{eq:basic_ineq_10}.  Finally, 
$$
\|\hf-f_0\|_2^2 \leq 2\|\hf-\tf\|_2^2 + 2\|\tf-f_0\|_2^2 \leq  
24 \|\tf-f_0\|_n^2 +  2\|\tf-f_0\|_2^2 + c^2 d r_n^2 \delta^2, 
$$
on $\Omega_1 \cap \Omega_2 \cap \Omega_3 \cap \Omega_4$. 
The $L_2$ norm result in \eqref{eq:error_bd_j_hd_L2} in the theorem
follows by simply adjusting the constants $c,c_1,c_2,c_3>0$ as needed.      

\subsection{Proof of Corollary \ref{cor:error_bd_tv_hd}}
\label{app:error_bd_tv_hd}

The proof of the statement in part (a) is exactly as in the proof of part (a) in 
Corollary \ref{cor:error_bd_tv}.

For part (b), we can consider $k \geq 2$ without a loss of generality, and 
start with an additive spline approximant \smash{$\check{f}$}
to $f_0$ from Lemma \ref{lem:spline_approx}. Let \smash{$\tf$} denote the result
of centering each component of \smash{$\check{f}$} to have zero empirical mean.  
Then \smash{$\TV(\tf_j^{(k)}) \leq a_k c_n = \delta$}, $j=1,\ldots,d$, and just
as in the proof of part (b) in Corollary \ref{cor:error_bd_tv}, letting
$\|\cdot\|$ denote either the empirical or $L_2$ norm, we have
$$
\bigg\| \sum_{j=1}^d \tf_j - \sum_{j=1}^d f_{0j} \bigg\|^2
\leq \bigg(\frac{c_0^k a_k}{b_0^k}\bigg)^2
  \bigg(\frac{\log n}{n} \bigg)^{2k} d^2 c_n^2.
$$
But when $n \geq n_0 (dc_n)^{(2k+3)/(2k+2)}$, the right-hand side above is
bounded by $a_0 d n^{-(2k+2)/(2k+3)} c_n$ for a constant $a_0>0$, which shows
the approximation error terms in \eqref{eq:error_bd_j_hd_empirical},
\eqref{eq:error_bd_j_hd_L2} are of the desired order.  This proves the desired
result for restricted locally adaptive splines.

For part (c), we follow the same arguments, the only difference being that we
construct a falling factorial approximant \smash{$\check{f}$} to $f_0$ from
Lemma \ref{lem:fall_fact_approx}.

\subsection{Preliminaries for the proof of Theorem \ref{thm:lower_bd_j}}   
\label{app:lower_bd_j_prelim}

The next two results in this subsection are helper lemmas for the
last lemma.

\begin{lemma}
\label{lem:pi_perp_packing_bd}
Let $J$ be a functional that satisfies Assumptions \ref{as:poly_null},
\ref{as:deriv_bd}, \ref{as:packing_bd}. Then there are constants
\smash{$\tilde{K}_1,\tilde\delta_1>0$},
that depend only on $k,L,K_1,w$, such that for all \smash{$0<\delta \leq  
  \tilde\delta_1$},  
$$
\log M\Big( \delta, \|\cdot\|_2, \Pi^\perp_k\big(B_J(1)\big) \Big) \geq  
\tilde{K}_1 \delta^{-w},
$$
where $\|\cdot\|_2$ is the $L_2(U)$ norm, with $U$ the
uniform distribution on $[0,1]$, and \smash{$\Pi_k^\perp$} is defined
by 
$$
\Pi_k^\perp(g)=g-\Pi_k(g), \;\;\; \text{where} \;\;\;
\Pi_k(g)=\argmin_{p \in \cP_k} \; \|g-p\|_2,
$$
with $\cP_k$ denoting the space of polynomials of degree $k$.  In
other words, \smash{$\Pi_k^\perp$} is the projection operator onto the
space orthogonal (in $L_2(U)$) to the polynomials of degree $k$.  
\end{lemma}

\begin{proof}
Let $R_0 > 0$ be the constant from Lemma \ref{lem:pi_perp_sup_bd},
when we take $\Lambda=U$. Note that
\begin{equation}
\label{eq:bj_decomp}
B_J(1) \cap B_\infty(R_0) = \Pi_k^\perp\big(B_J(1)\big) +
\big(\cP_k \cap B_\infty(R_0)\big).  
\end{equation}
In general, for $S=S_1+S_2$ and a norm $\|\cdot\|$, observe that,
from basic relationships between covering and packing numbers,  
$$
M(4\delta, \|\cdot\|, S) \leq N(2\delta, \|\cdot\|, S) \leq 
N(\delta, \|\cdot\|, S_1) N(\delta, \|\cdot\|, S_2) \leq
M(\delta, \|\cdot\|, S_1) N(\delta, \|\cdot\|, S_2),
$$
so that
$$
\log M (\delta, \|\cdot\|, S_1) \geq \log
\frac{M(4\delta, \|\cdot\|, S)}{N(\delta, \|\cdot\|, S_2)}.  
$$
Applying this to our decomposition in \eqref{eq:bj_decomp}, 
\begin{align*}
\log M\Big( \delta, \|\cdot\|_2, \Pi_k^\perp\big(B_J(1)\big) \Big) 
&\geq \log \frac{M\big(4\delta, \|\cdot\|_2, B_J(1) \cap B_\infty(R_0)\big)}   
{N\big( \delta, \|\cdot\|_2, \cP_k \cap B_\infty(R_0)\big)} \\
&\geq K_1 R_0^w 4^{-w} \delta^{-w} - A (k+1) \log(1/\delta),   
\end{align*}
where in the second inequality we used Assumption \ref{as:packing_bd} 
(assuming without a loss of generality that $R_0 \geq 1$),
and a well-known entropy bound for a finite-dimensional 
ball (e.g., \citet{mammen1991nonparametric}), with $A>0$ being a constant
that depends only on $R_0$. For small enough $\delta$, the right-hand side above
is of the desired order, and this completes the proof.  
\end{proof}

\begin{lemma}
\label{lem:varshamov_gilbert}
Let $d, M>0$ be integers, and $I = \{ 1, \ldots, M\}$.  
Denote by
\smash{$H(u,v) = \sum_{j=1}^d 1\{u_j \neq v_j\}$} the Hamming 
distance between $u,v \in I^d$. Then there is a subset $S 
\subseteq I^d$ with \smash{$|S| \geq (M/4)^{d/2}$} such
that $H(u,v) \geq d/2$ for any $u,v \in S$.
\end{lemma}

\begin{proof}
Let $\Omega_0 = I^d$, $u_0 = (1,\ldots,1) \in \Omega_0$.
For $j=0,1,\ldots$, recursively define
$$
\Omega_{j+1} = \{ u \in \Omega_j : H(u, u_j) > 
a = \lceil d/2 \rceil \}, 
$$
where $u_{j+1}$ is arbitrarily chosen from $\Omega_{j+1}.$ The
procedure is stopped when $\Omega_{j+1}$ is empty; denote the
last set defined in this procedure by $\Omega_E$, and 
denote $S=\{u_0,\ldots,u_E\}$. For $0 \leq i, j \leq E$, by
construction, $H(u_i,u_j) > a$.  For $j=0,\ldots,E$, 
\begin{align*}
n_j = | \Omega_j - \Omega_{j+1}| &=
| \{ u \in \Omega_j : H(u,u_j) \leq a \} | \\
&\leq | \{ u \in I^d : H(u,u_j) \leq a \}| \\
&= {d \choose d-a} M^a
\end{align*}
The last step is true because we can choose $d-a$ positions in
which $u$ matches $u_j$ in \smash{${d \choose d-a}$} ways, and 
the rest of the $a$ positions can be filled arbitrarily in $M$ ways. 
Also note $M^d = n_0 + \cdots + n_E$. Therefore
$$
M^d \leq (E+1) {d \choose d-a} M^a,
$$
which implies
$$
E+1 \geq \frac{M^{d-a}}{{d \choose d-a}} 
\geq \frac{M^{d-a}}{2^d} \geq (M/4)^{d/2}. 
$$
\end{proof}

The lemma below gives a key technical result used in the proof of
Theorem \ref{thm:lower_bd_j}.     

\begin{lemma}
\label{lem:bj_d_packing_bd}
Let $J$ be a functional that satisfies Assumptions \ref{as:poly_null},
\ref{as:deriv_bd}, \ref{as:packing_bd}. Then there are constants
\smash{$\bar{K}_1,\bar\delta_1>0$}, that depend 
only on \smash{$w,\tilde{K}_1,\tilde\delta_1$}, where
\smash{$\tilde{K}_1,\tilde\delta_1>0$} are the 
constants from Lemma \ref{lem:pi_perp_packing_bd},  
such that for all \smash{$0<\delta \leq \bar\delta_1$},    
$$
\log M\Big(\delta, \|\cdot\|_2, \Pi_{k,d}^\perp\big(B_J^d(1)\big) \Big)   
\geq \bar{K}_1 d^{1+w/2} \delta^{-w},
$$
where $\|\cdot\|_2$ is the $L_2(U)$ norm, with $U$ the
uniform distribution on $[0,1]^d$, and \smash{$\Pi_{k,d}^\perp$} is 
defined 
by 
$$
\Pi_{k,d}^\perp(g)=g-\Pi_{k,d}(g), \;\;\; \text{where} \;\;\;
\Pi_{k,d}(g)=\argmin_{p \in \cP_{k,d}} \; \|g-p\|_2,
$$
and \smash{$\cP_{k,d}$} contains all functions of the form
\smash{$p(x)=\sum_{j=1}^d p_j(x_j)$}, for polynomials $p_j$,
$j=1,\ldots,d$ of degree $k$.  In other words,
\smash{$\Pi_{k,d}^\perp$} is the projection operator 
onto the space orthogonal (in $L_2(U)$) to the space
\smash{$\cP_{k,d}$} of additive polynomials of degree $k$. 
\end{lemma}

\begin{proof}
It is easy to check that the decomposability property of the
$L_2(U)$ norm, in \eqref{eq:L2_decomp}, implies a certain
decomposability of the $L_2(U)$ projection operators
\smash{$\Pi_{k,d},\Pi_{k,d}^\perp$} over additive functions:
$$
\Pi_{k,d}\bigg(\sum_{j=1}^d m_j \bigg) = 
\sum_{j=1}^d \Pi_k(m_j), \;\;\;  
\Pi_{k,d}^\perp\bigg( \sum_{j=1}^d m_j \bigg) = 
\sum_{j=1}^d \Pi_k^\perp(m_j),  
$$
where \smash{$\Pi_k,\Pi_k^\perp$} are projection operators onto
$\cP_k$ and its orthocomplement, respectively, as defined in Lemma
\ref{lem:pi_perp_packing_bd}.  The decomposability result for
\smash{$\Pi_{k,d}^\perp$} in particular implies that 
\begin{equation}
\label{eq:bj_d_decomp}
\Pi_{k,d}^\perp\big(B_J^d(1)\big) = \bigg \{ \sum_{j=1}^d f_j :  
f_j \in \Pi_k^\perp\big(B_J(1)\big), \; j=1,\ldots,d \bigg\}.  
\end{equation}
Abbreviate \smash{$M = M(\delta/\sqrt{d/2}, \|\cdot\|_2,
  \Pi_k^\perp(B_J(1)))$}.  
By Lemma \ref{lem:pi_perp_packing_bd}, we have for small enough $\delta$, 
$$
\log M \geq \tilde{K}_1 2^{-w/2} d^{w/2} \delta^{-w}.  
$$
Now let \smash{$g_1, \ldots, g_M$} denote a
\smash{$(\delta/\sqrt{d/2})$}-packing of
\smash{$\Pi_k^\perp(B_J(1))$}.
Let $I = \{ 1,\ldots, M\}$, and for $u\in I^d$, 
define \smash{$f_u \in \Pi_k^\perp(B_J^d(1))$} by 
$$
f_u = \sum_{j=1}^d g_{u_j},
$$
i.e., $f_u$ is an additive function with components \smash{$g_{u_j}$},
$j=1,\ldots,d$. If the Hamming distance between indices $u,v$
satisfies $H(u,v) \geq d/2$, then
$$
\| f_u - f_v \|_2^2 = \sum_{j=1}^d \| g_{u_j} - g_{v_j} \|_2^2 
\geq H(u,v) \frac{\delta^2}{d/2} \geq \delta^2, 
$$
where we have again used the $L_2(U)$ decomposability property in
\eqref{eq:L2_decomp}. Thus, it is sufficient to find a subset $S$ of
$I^d$ such that $u,v \in S \Rightarrow H(u,v) \geq d/2$. By Lemma 
\ref{lem:varshamov_gilbert}, we can choose such an $S$ with 
\smash{$|S| \geq (M/4)^{d/2}$}. For small enough $\delta$,
such that $M \geq 16$, this gives the desired result because 
$$
\log |S| \geq \frac{d}{2} \log \frac{M}{4} \geq 
\frac{d}{4} \log M \geq \tilde{K}_12^{-w/2-2} 
d^{1+w/2} \delta^{-w}.
$$
\end{proof}

\subsection{Proof of Theorem \ref{thm:lower_bd_j}}
\label{app:lower_bd_j}

Clearly, by orthogonality, for any functions $\hf,f_0$,
$$
\|\hf - f_0\|_2^2 = \|\Pi_{k,d}(\hf) - \Pi_{k,d}(f_0)\|_2^2 +
\|\Pi_{k,d}^\perp(\hf) - \Pi_{k,d}^\perp(f_0)\|_2^2 \geq
\|\Pi_{k,d}^\perp(\hf) - \Pi_{k,d}^\perp(f_0)\|_2^2,
$$
where \smash{$\Pi_{k,d},\Pi_{k,d}^\perp$} are projection operators
onto $\cP_{k,d}$ and its orthocomplement, respectively, defined in
Lemma \ref{lem:pi_perp_packing_bd}.  Thus it suffices to consider the
minimax error over \smash{$\Pi_{k,d}^\perp(B_J^d(c_n))$}.

First, we lower bound the packing number and upper bound the
covering number of the class \smash{$\Pi_{k,d}^\perp(B_J^d(c_n))$}.  
The upper bound is more straightforward: 
\begin{align}
\nonumber
\log N\Big(\epsilon, \|\cdot\|_2, \Pi_{k,d}^\perp\big(B_J^d(c_n)\big) \Big) 
&= \log N\Big(\epsilon/c_n, \|\cdot\|_2, \Pi_{k,d}^\perp\big(B_J^d(1)\big) 
  \Big) \\
\nonumber
&\leq \sum_{j=1}^d \log N\Big(\epsilon/(c_n \sqrt{d}), \| \cdot\|_2,   
\Pi_k^\perp\big(B_J(1)\big) \Big) \\
\label{eq:add_covering_bd}
&\leq K_2 c_n^w d^{1+w/2} \epsilon^{-w}.
\end{align}
The second inequality follows from property
\eqref{eq:bj_d_decomp} in the proof of Lemma \ref{lem:bj_d_packing_bd} 
and similar arguments to those in the proof of Lemma
\ref{lem:add_entropy_bd}---except that we leverage the decomposability of  
the $L_2$ norm, as in \eqref{eq:L2_decomp}, instead of using the
triangle inequality. The third inequality follows from Assumption 
\ref{as:packing_bd}.  

The lower bound is less straightforward, and is given by 
Lemma \ref{lem:bj_d_packing_bd}:
\begin{align}
\nonumber
\log M\Big(\delta, \|\cdot\|_2, \Pi_{k,d}^\perp\big(B_J^d(c_n)\big) \Big) 
&= \log M\Big(\delta/c_n, \|\cdot\|_2, \Pi_{k,d}^\perp\big(B_J^d(1)\big)\Big) \\ 
\label{eq:bj_d_packing_bd}
&\geq \bar{K}_1 c_n^w d^{1+w/2} \delta^{-w}.
\end{align}
We note that \eqref{eq:bj_d_packing_bd} holds for \smash{$0<\delta \leq  
  \bar\delta_1$}, where \smash{$\bar\delta_1>0$} is the constant from
Lemma \ref{lem:bj_d_packing_bd}.

Now, following the strategy in \citet{yang1999information},  
we use these bounds on the packing and covering numbers, 
along with Fano's inequality, to establish the desired result.   
Let $f_1,f_2,\ldots,f_M$ be a $\delta_n$-packing of
\smash{$\Pi_{k,d}^\perp(B_J(c_n))$}, for $\delta_n>0$ to be specified 
later. Fix an arbitrary estimator \smash{$\hf$}, and let 
$$
\hat{Z} = \argmin_{j\in \{1,\ldots,M\}} \, \| \hf - f_j \|_2.
$$
We will use \smash{$P_{X,f}$} and \smash{$\E_{X,f}$} to denote the  
probability and expectation operators, respectively, over
i.i.d.\ draws $X^i \sim U$, $i=1,\ldots,n$ (where $U$ is the uniform
distribution on $[0,1]^d$), and i.i.d.\ draws $Y^i|X^i \sim N(f(X^i),\sigma^2)$,  
$i=1,\ldots,n$.  Then
\begin{align}
\nonumber
&\sup_{f_0 \in \Pi_{k,d}^\perp(B_J^d(c_n))} \, 
\E_{X, f_0} \| \hf - f_0 \|_2^2 
\geq \sup_{f_0 \in \{f_1,\ldots,f_M\}} \, 
\E_{X, f_0} \| \hf - f_0 \|_2^2 \\
\nonumber
&\qquad\geq \frac{1}{M} \E_X \sum_{j=1}^M 
\E_{f_j} \|\hf - f_j \|_2^2 \\
\nonumber
&\qquad= \frac{1}{M} \E_X \sum_{j=1}^M \bigg(
\P_{f_j}(\hat{Z} \neq j) \E_{f_j}\big( \|\hf - f_j \|_2^2 
\,\big|\, \hat{Z} \neq j \big) +  
\P_{f_j}(\hat{Z} = j) \E_{f_j} \big( \|\hf - f_j \|_2^2
\,\big|\, \hat{Z} = j \big)\bigg) \\
\nonumber
&\qquad\geq \frac{1}{M} 
\E_X \sum_{j=1}^M \P_{f_j}(\hat{Z} \neq j) 
\E_{f_j}\big( \|\hf - f_j \|_2^2 \,\big|\, \hat{Z} \neq j \big) \\ 
\label{eq:fano_setup}
&\qquad\geq \frac{1}{M} \E_X \sum_{j=1}^M 
\P_{f_j} (\hat{Z} \neq j) \frac{\delta^2_n}{4},
\end{align}
where in the last inequality we have used the fact that if
\smash{$\hat{Z} \neq j$}, then \smash{$\hf$} must be at least 
$\delta_n/2$ away from $f_j$, for each $j=1,\ldots,M$.  

Abbreviate $q_j$ for the distribution \smash{$P_{f_j}$},
$j=1,\ldots,M$, and 
define the mixture \smash{$\bar{q} = \frac{1}{M} \sum_{j=1}^M
  q_j$}. By Fano's inequality, 
\begin{equation}
\label{eq:fano_ineq}
\frac{1}{M}\E_X\sum_{j=1}^M \P_{f_j} (\hat{Z} \neq j) \geq   
1 - \frac{\frac{1}{M} \sum_{j=1}^M
  \E_X \KL(q_j \,\|\, \bar{q}) + \log{2}}{\log{M}},
\end{equation}
where $\KL(P_1 \,\|\, P_2)$ denotes the Kullback-Leibler 
(KL) divergence between distributions $P_1,P_2$.
Let $g_1, g_2,\ldots,g_N$ be an $\epsilon_n$-covering of
\smash{$\Pi_{k,d}^\perp(B_J^d(c_n))$}, for $\epsilon_n>0$ to be
determined shortly.  Abbreviate $s_\ell$ for the
distribution \smash{$P_{g_\ell}$}, $\ell=1,\ldots,N$, and
\smash{$\bar{s} = \frac{1}{N} \sum_{\ell=1}^N s_\ell$}.  
Also, write \smash{$p(N(f(X),\sigma^2 I))$} for the density 
of a $N(f(X),\sigma^2 I)$ random variable, where
$f(X)=(f(X^1),\ldots,f(X^n)) \in \R^n$. Then
\begin{align}
\nonumber
\frac{1}{M} 
\sum_{j=1}^M \E_X \KL(q_j \,\|\, \bar{q}) 
&\leq \frac{1}{M} 
\sum_{j=1}^M \E_X \KL(q_j \,\|\, \bar{s}) \\
\nonumber
&= \frac{1}{M} \sum_{j=1}^M 
\E_{X, f_j} \log \frac{p\big(N(f_j(X),\sigma^2 I)\big)}
{\frac{1}{N} \sum_{\ell=1}^N p\big(N(g_\ell(X), \sigma^2 I)\big)} \\ 
\nonumber
&\leq \frac{1}{M} \sum_{j=1}^M 
\bigg(\log N + \E_X \min_{\ell=1,\ldots,N} \,
\KL( q_j \,\|\, s_\ell) \bigg) \\ 
\nonumber
&\leq \frac{1}{M} \sum_{j=1}^M 
\bigg(\log N + \frac{n\epsilon_n^2}{2\sigma^2} \bigg) \\ 
\label{eq:kl_bd}
&\leq K_2 c_n^w d^{1+w/2} \epsilon_n^{-w} + 
\frac{n\epsilon_n^2}{2\sigma^2}.
\end{align}
In the first line above, we used the fact that
\smash{$\sum_{j=1}^M \KL(q_j \,\|\, \bar{q})
\leq \sum_{j=1}^M \KL(q_j \,\|\, s)$} for any other distribution $s$; 
in the second and third, we explicitly expressed and manipulated the
definition of KL divergence; in the fourth, we used 
\smash{$\KL(q_j \,\|\,
  s_\ell)=\|f_j(X)-g_\ell(X)\|_2^2/(2\sigma^2)$}, and for each
$j$, there is at least one $\ell$ such that 
\smash{$\E_X \|f_j(X)-g_\ell(X)\|_2^2 = \|f_j-g_\ell\|_2^2 \leq 
\epsilon_n^2$}; in the fifth line, we used the entropy bound from
\eqref{eq:add_covering_bd}.  Minimizing \eqref{eq:kl_bd} over
$\epsilon_n>0$ gives
$$
\frac{1}{M} 
\sum_{j=1}^M \E_X \KL(q_j \,\|\, \bar{q}) \leq 
\bar{K}_2 d n^{w/(2+w)} c_n^{2w/(2+w)}, 
$$
for a constant \smash{$\bar{K}_2>0$}.  Returning to Fano's
inequality \eqref{eq:fano_setup}, \eqref{eq:fano_ineq}, 
we see that a lower bound on the minimax error is
$$
\frac{\delta_n^2}{4} \bigg(
1 - \frac{\bar{K}_2 d n^{w/(2+w)} c_n^{2w/(2+w)} +  
\log{2}}{\log{M}} \bigg),  
$$
Therefore, a lower bound on the minimax error is $\delta_n^2/8$, for 
any $\delta_n>0$ such that  
$$
\log{M} \geq  2\bar{K}_2 d n^{w/(2+w)} c_n^{2w/(2+w)} + 
2\log{2},
$$
and for large enough $n$, the first term on the right-hand side above
will be larger than $2\log{2}$, so it suffices to have
\begin{equation}
\label{eq:logm_ineq}
\log{M} \geq  4\bar{K}_2 d n^{w/(2+w)} c_n^{2w/(2+w)}.
\end{equation}
Set \smash{$\delta_n=(\bar{K}_1/4\bar{K}_2)^{1/w}
  \sqrt{d} n^{-1/(2+w)} c_n^{w/(2+w)}$}.  Provided that
\smash{$\delta_n \leq \bar\delta_1$}, our log packing bound  
\eqref{eq:bj_d_packing_bd} is applicable, and ensures that
\eqref{eq:logm_ineq} will be satisfied. This completes the proof.  

\subsection{Proof of Corollary \ref{cor:lower_bd_tv}}
\label{app:lower_bd_tv}

We only need to check Assumption \ref{as:packing_bd} for
\smash{$J(g)=\TV(g^{(k)})$}, $w=1/(k+1)$, and then we can apply 
Theorem \ref{thm:lower_bd_j}.   As before, the entropy bound upper
bound is implied by results in  
\citet{birman1961piecewise} (see \citet{mammen1991nonparametric} for
an explanation and discussion).
The packing number lower bound is verified as follows.  For $f$ a
$(k+1)$ times weakly differentiable function on $[0,1]$,
$$
\TV(f^{(k)}) = \int_0^1 | f^{(k+1)}(t) | \,dt 
\leq \bigg(\int_0^1 |f^{(k+1)}(t)|^2 \,dt\bigg)^{1/2}.
$$
Hence 
$$
\Big\{ f : \TV(f^{(k)}) \leq 1, \; \|f\|_\infty \leq 1 \Big\}
\supseteq \bigg\{ f : \int_0^1 |f^{(k+1)}(t)|^2 \, dt 
\leq 1, \; \|f\|_\infty \leq 1 \bigg\} .
$$
Results in \citet{kolmogorov1959entropy} imply that the space on
the right-hand side satisfies the desired log packing number lower
bound.  This proves the result.

\subsection{Proof of the linear smoother lower bound in
  \eqref{eq:lower_bd_tv_linear}} 
\label{app:lower_bd_tv_linear}

We may assume without a loss of generality that each
$f_{0j}$, $j=1,\ldots,d$ has $L_2$ mean zero (since $f_0$ does).  
By the decomposability property of the $L_2$ norm over additive functions   
with $L_2$ mean zero components, as in \eqref{eq:L2_decomp},
we have for any additive linear smoother \smash{$\hf=\sum_{j=1}^d \hf_j$}, 
$$
\|\hf-f_0\|_2^2 = \bigg(\sum_{j=1}^d \bar{f}_j \bigg)^2 
+ \sum_{j=1}^d \|(\hf_j-\bar{f}_j)  - f_{0j}\|_2^2 
$$
where $\bar{f}_j$ denotes the $L_2$ mean of 
$\hf_j$, $j=1,\ldots,d$.   Note that the estimator
\smash{$\hf_j-\bar{f}_j$} is itself a linear smoother, for each
$j=1,\ldots,d$, since if we write 
\smash{$\hf_j(x_j)=w_j(x_j)^T Y$} for a weight function  
$w_j$ over $x_j \in [0,1]$, then 
\smash{$\hf_j(x_j)-\bar{f}_j= \tilde{w}_j(x_j)^T Y$} for
a weight function 
\smash{$\tilde{w}_j(x_j)=w_j(x_j) - \int_0^1 w_j(t)\,dt$}. 
This, and the last display, imply that 
\begin{equation}
\label{eq:minimax_linear_equiv}
\inf_{\hf \, \text{\rm additive linear}} \, \sup_{f_0 \in
  \cF_k^d(c_n)} \, \E \| \hf - f_0 \|_2^2 \;=\;
\sum_{j=1}^d \inf_{\hf_j \, \text{\rm linear}} \, \sup_{f_0 \in
  \cF_k^d(c_n)} \, \E \| \hf_j(Y) - f_{0j} \|_2^2. 
\end{equation}

Now fix an arbitrary $j=1,\ldots,d$, and consider the $j$th term in the sum on
the right-hand side above.  Here we are looking at a linear smoother
\smash{$\hf_j$} fit to data  
\begin{equation}
\label{eq:y_dist_j}
Y^i=\mu+f_{0j}(X^i_j)+\sum_{\ell\not=j}f_{0\ell}(X^i_\ell)+\epsilon^i,
\;\;\; i=1,\ldots,n.
\end{equation}
which depends on the components $f_{0\ell}$, for
$\ell\not=j$. This is why the supremum in the $j$th term of the sum 
on the right-hand side in \eqref{eq:minimax_linear_equiv} must be taken over 
\smash{$f_0 \in \cF_k^d(c_n)$}, rather than \smash{$f_{0j} \in \cF_k(c_n)$}. Our
notation \smash{$\hf_j(Y)$} is used as a reminder to emphasize the dependence 
on the full data vector in \eqref{eq:y_dist_j}.

A simple reformulation, by appropriate averaging over the
lattice, helps untangle this supremum.  
Write \smash{$\hf_j(x_j)= w_j(x_j)^T Y$} for a weight function   
$w_j$ over $x_j \in [0,1]$, and for each
$v=1,\ldots,N$, let \smash{$I_j^v$} be the set of indices  
$i$ such that \smash{$X^i_j=v/N$}.  Also let 
$$
\bar{Y}^v_j=\frac{1}{N^{d-1}} \sum_{i \in I^v_j} Y^i, \;\;\;
v=1,\ldots,N,
$$
and \smash{$\bar{Y}_j=(\bar{Y}^1_j,\ldots,\bar{Y}^N_j) \in \R^N$}.
Then note that we can also write \smash{$\hf_j(x_j)=\bar{w}_j(x_j)^T 
  \bar{Y}_j$}  for a suitably defined weight function
\smash{$\bar{w}_j$}, i.e., 
note that we can think of \smash{$\tf_j$} as a linear smoother fit to
data \smash{$\bar{Y}_j$}, whose components follow the distribution  
\begin{equation}
\label{eq:ybar_dist_j}
\bar{Y}^v_j=\mu_j+f_{0j}(v/N)+\bar\epsilon^v_j,
\;\;\; v=1,\ldots,N,
\end{equation}
where we let 
\smash{$\mu_j = \mu+\frac{1}{N}\sum_{\ell\not=j}\sum_{u=1}^n
  f_{0\ell}(u/N)$}, and \smash{$\bar\epsilon^v_j$}, $v=1,\ldots,n$
are i.i.d.\ \smash{$N(0,\sigma^2/N^{d-1})$}.  Recalling that $f_{0j}
\in \cF_k(c_n)$, we are in a position to invoke 
univariate minimax results from \citet{donoho1998minimax}. 
As shown in Section 5.1 of \citet{tibshirani2014adaptive}, the space
$\cF_k(c_n)$ contains the Besov space \smash{$B_{1,1}^{k+1}(c_n')$}, 
for a radius $c_n'$ that differs from $c_n$ only by a constant
factor. Therefore, by Theorem 1 of \citet{donoho1998minimax} on
the minimax risk of linear smoothers fit to data from the model
\eqref{eq:ybar_dist_j}, we see that for large enough $N$ and a constant $c_0>0$,    
\begin{align}
\nonumber
\inf_{\hf_j \, \text{\rm linear}} \, \sup_{f_{0j} \in
  \cF_k(c_n)} \, \E \| \hf_j(\bar{Y}_j) - f_{0j} \|_2^2  
&\geq c_0 (c_n N^{(d-1)/2})^{2/(2k+2)}
\frac{N^{-(2k+1)/(2k+2)}}{N^{d-1}} \\
\nonumber
&= c_0 N^{-d(2k+1)/(2k+2)} c_n^{2/(2k+2)} \\
\label{eq:minimax_linear_j}
&= c_0 n^{-(2k+1)/(2k+2)} c_n^{2/(2k+2)}.
\end{align}
As we have reduced the lower bound to the minimax risk of linear
smoothers over a Besov ball, we can see that the same result
\eqref{eq:minimax_linear_j} indeed holds simultaneously over all 
$j=1,\ldots,d$.  Combining this with \eqref{eq:minimax_linear_equiv}
gives the desired result \eqref{eq:lower_bd_tv_linear}. 

\subsection{Proof of Theorem \ref{thm:backfit_par} and derivation
  details for Algorithm \ref{alg:backfit_par}}
\label{app:backfit_par}

We show that the dual of \eqref{eq:add_tf_dual_re} is
equivalent to the additive trend filtering problem
\eqref{eq:add_tf}, and further, the Lagrange multipliers
corresponding to the constraints $u_0=u_j$, for $j=1,\ldots,d$, 
are equivalent to the primal variables $\thetaj$,
$j=1,\ldots,d$.   Let $M=I-\one\one^T/n$, and rewrite problem
\eqref{eq:add_tf_dual_re} as 
$$
\begin{alignedat}{2}
&&\min_{u_0,u_1,\ldots,u_d \in \R^n} \quad &\half \|MY - Mu_0\|_2^2 +
\sum_{j=1}^d I_{U_j}(u_j) \\  
&&\st \quad & Mu_0= Mu_1, \; Mu_0= Mu_2, \;\ldots, \;  Mu_0= Mu_d,       
\end{alignedat}
$$
We can write the Lagrangian of this problem as
$$
L(u_0,u_1,\ldots,u_d, \theta_1,\ldots,\theta_d) =  
\half \|MY - Mu_0\|_2^2 + \sum_{j=1}^d I_{U_j}(u_j)
+ \sum_{j=1}^d \thetaj^T M (u_0 - u_j). 
$$ 
and we want to minimize this over $u_0,\ldots,u_d$ to form the dual of 
\eqref{eq:add_tf_dual_re}.  This gives 
\begin{equation}
\label{eq:add_tf_dual_dual}
\max_{\theta_1, \ldots,\theta_d \in \R^n} \;
\half\|MY\|_2^2 - \half \bigg\|MY - \sum_{j=1}^d M\thetaj
\bigg\|_2^2  - \sum_{j=1}^d \Big(\max_{u_j \in U_j} \; u_j^T M
\thetaj\Big).  
\end{equation}
We use the fact that the support function of $U_j$ is just
$\ell_1$ penalty composed with $S_j D_j$ (invoking  
the duality between $\ell_\infty$ and $\ell_1$ norms),
$$
\max_{u_j \in U_j} \, u_j^T M \thetaj = 
\max_{\|v_j\|_\leq \lambda} \, v_j^T D_j S_j M \thetaj =  
\lambda \|D_j S_j M \thetaj \|_1,
$$
where recall we abbreviate \smash{$D_j=\Dj$}, 
for $j=1,\ldots,d$, and this allows us to rewrite the above
problem \eqref{eq:add_tf_dual_dual} as
$$
\min_{\theta_1, \ldots,\theta_d \in \R^n} \;
\half \bigg\|MY - \sum_{j=1}^d M \thetaj \bigg\|_2^2 
+ \lambda \sum_{j=1}^d \|D_j S_j M \thetaj \|_1, 
$$
which is precisely the same as the original additive trend filtering
problem \eqref{eq:add_tf}.

This realization has important consequences.  In the ADMM
iterations \eqref{eq:admm_par}, the scaled parameters $\rho \gamma_j$,
$j=1,\ldots,d$ correspond to dual variables $\thetaj$, $j=1,\ldots,d$
in problem \eqref{eq:add_tf_dual_re}, which from the above
calculation, are precisely primal variables in \eqref{eq:add_tf}. 
Under weak conditions, ADMM is known to produce convergent dual
iterates, e.g., Section 3.2 of \citet{boyd2011distributed} shows that
if (i) the criterion is a sum of closed, convex
functions and (ii) strong duality holds, then the dual iterates from
ADMM converge to optimal dual solutions.   (Convergence of primal
iterates requires stronger assumptions.) Our problem
\eqref{eq:add_tf_dual_re} satisfies these two conditions, and so
for the ADMM algorithm outlined in \eqref{eq:admm_par}, the scaled
iterates \smash{$\rho\gamma_j^{(t)}$, $j=1,\ldots,d$} converge to
optimal solutions in the dual of \eqref{eq:add_tf_dual_re}, i.e.,
optimal solutions in the additive trend filtering problem
\eqref{eq:add_tf}. This proves the first part of the theorem. 

As for the second part of the theorem, it remains to show that
Algorithm \ref{alg:backfit_par} is equivalent to the ADMM iterations
\eqref{eq:admm_seq}. This follows by notationally swapping
$\gamma_j$, $j=1,\ldots,d$ for $\thetaj/\rho$, $j=1,\ldots,d$, 
rewriting the updates
$$
\thetaj^{(t)}/\rho=u_0^{(t)}+\thetaj^{(t-1)}/\rho-u_j^{(t)}, 
\;\;\; j=1,\ldots,d, 
$$ 
as
$$
\thetaj^{(t)} = \rho \cdot \TF_\lambda 
\big( u_0^{(t)}+\thetaj^{(t-1)}/\rho, \Xj \big), \;\;\; j=1,\ldots,d,
$$
using \eqref{eq:tf_pd}, and lastly, eliminating $u_j$, $j=1,\ldots,d$
from the $u_0$ update by solving for these variables in terms of
terms of $\thetaj$, $j=1,\ldots,d$, i.e., by using
$$
u_j^{(t-1)} = u_0^{(t-1)}+\thetaj^{(t-2)}/\rho -
\thetaj^{(t-1)}/\rho, \;\;\; j=1,\ldots,d.
$$

\subsection{Cyclic versus parallel backfitting}
\label{app:experiment_backfitting}

We compare the performances of the usual cyclic backfitting
method in Algorithm \ref{alg:backfit} to the parallel version 
in Algorithm \ref{alg:backfit_par}, on a simulated data set
generated as in Section \ref{sec:experiment_large_sim}, except with 
$n=2000$ and $d=24$. We computed the additive trend filtering
estimate 
\eqref{eq:add_tf} (of quadratic order), at a fixed value of $\lambda$
lying somewhere near the middle of the regularization path, by
running the cyclic and parallel backfitting algorithms until each
obtained a suboptimality of $10^{-8}$ in terms of the achieved
criterion value (the optimal criterion value here was determined by 
running Algorithm \ref{alg:backfit} for a very large number of
iterations).  We used simply $\rho=1$ in Algorithm \ref{alg:backfit_par}. 

Figure \ref{fig:parallel} shows the progress of the two algorithms,
plotting the suboptimality of the criterion value across the
iterations. The two panels, left and right, differ in how iterations
are counted for the parallel method.  On the left, one full cycle
of $d$ component updates is counted as one iteration for the parallel
method---this corresponds to running the parallel algorithm in 
``naive'' serial mode, where each component update is actually
performed in sequence.  On the right, $d$ full cycles of $d$ component
updates is counted as one iteration for the parallel method---this
corresponds to running the parallel algorithm in an ``ideal'' parallel
mode with $d$ parallel processors. In both panels, one full cycle of
$d$ component updates is counted as one iteration for the cyclic
method. We see that, if parallelization is fully utilized, the
parallel method cuts down the iteration cost by about a factor of 2,
compared to the cyclic method.  We should expect these computational
gains to be even larger as the number of components $d$ grows.

\begin{figure}[tb]
\centering
\includegraphics[width=0.45\textwidth]{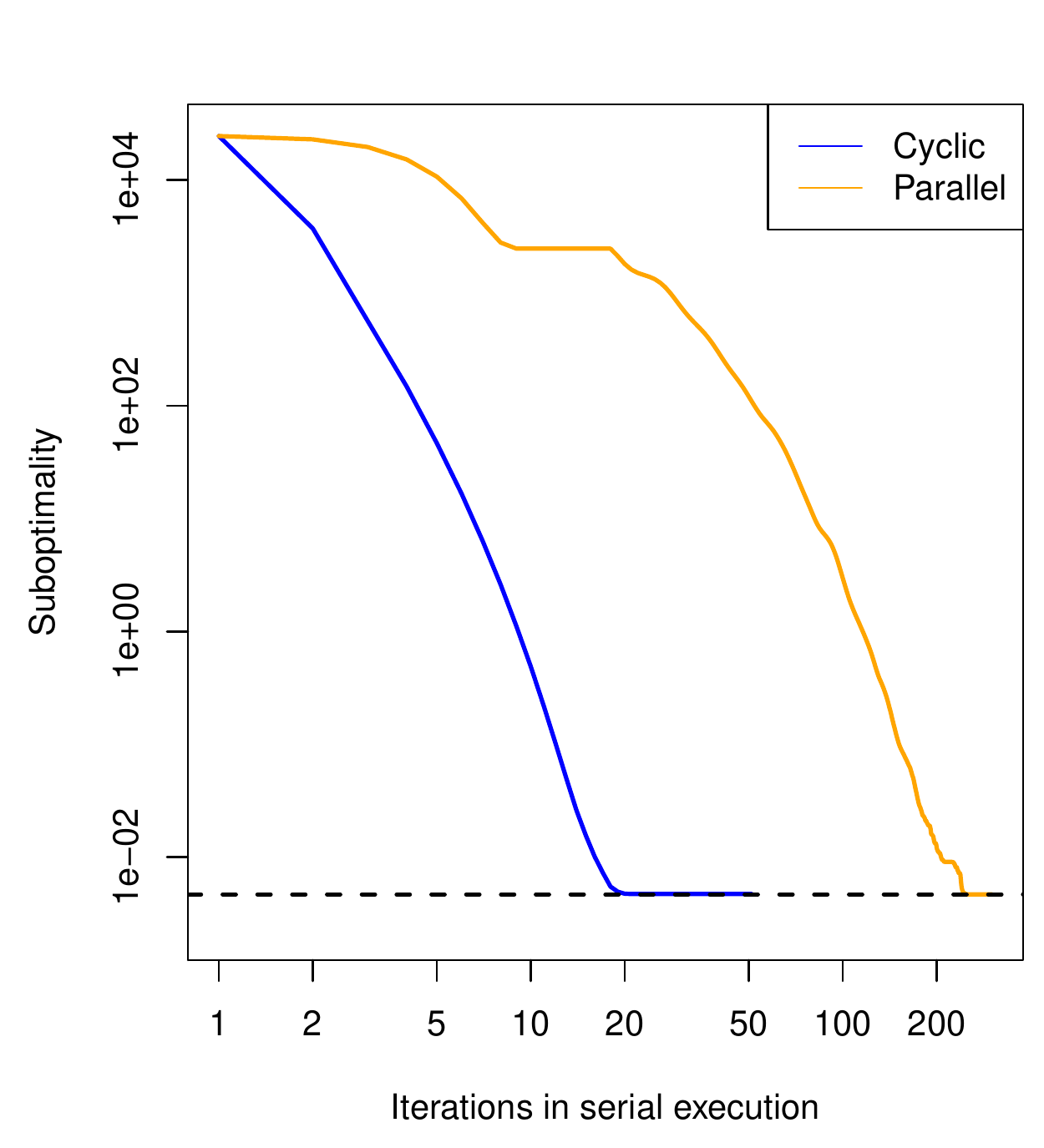}
\includegraphics[width=0.45\textwidth]{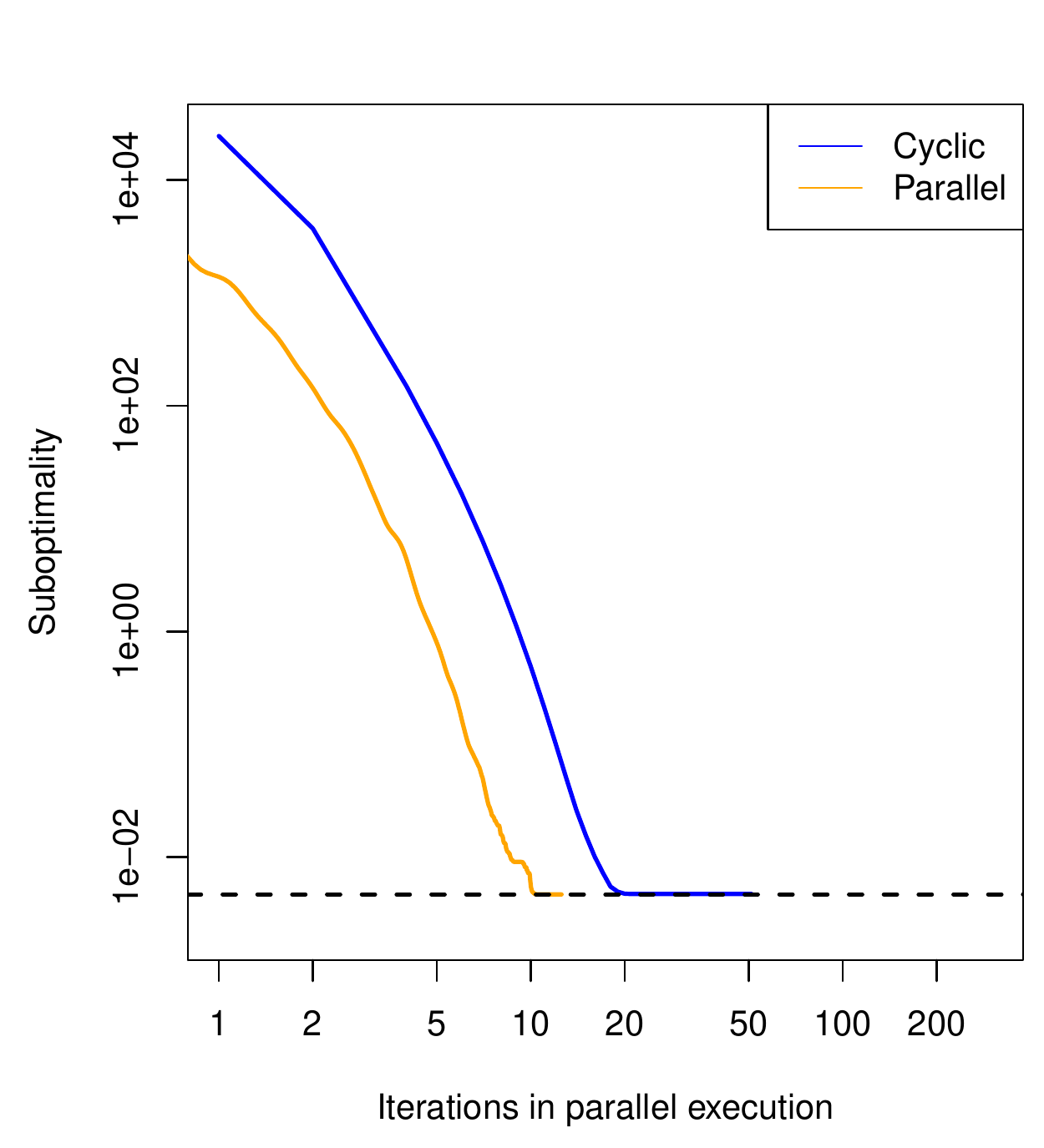}
\caption{\it\small Suboptimality in criterion value versus
iteration number for the cyclic (Algorithm \ref{alg:backfit}) and
parallel (Algorithm \ref{alg:backfit_par}) backfitting methods, on a
synthetic data set with $n=2000$ and $d=24$.  On the left, iterations
for the parallel method are counted as if ``ideal'' parallelization is
used, where the $d$ component updates are performed by $d$ processors,
at the total cost of one update, and on the right, iterations for the
parallel method are counted as if ``naive'' serialization is used,
where the component updates are performed in sequence.  To avoid
zeros on the y-axis (log scale), we added a small value to all the
suboptimalities (dotted line).} 
\label{fig:parallel}
\end{figure}

\subsection{Simulated homogeneously-smooth data}  
\label{app:experiment_large_sim_homo}

Figure \ref{fig:large_sim_homo} shows the results of a homogeneous  
simulation, as in Section
\ref{sec:experiment_large_sim} and Figure \ref{fig:large_sim}, except
that for the base component trends
we used sinusoids of equal (and spatially-constant) frequency:
$$
g_{0j}(x_j) = \sin(10\pi x_i), \;\;\; j=1,\ldots,10,
$$
and we defined the component functions as
\smash{$f_{0j}=a_jg_{0j}-b_j$}, $j=1,\ldots,d$, where $a_j,b_j$ were
chosen to standardize \smash{$f_{0j}$} (give it zero empirical mean
and unit empirical norm), for $j=1,\ldots,d$. 

\begin{figure}[tb]
\centering
\includegraphics[width=0.45\textwidth]{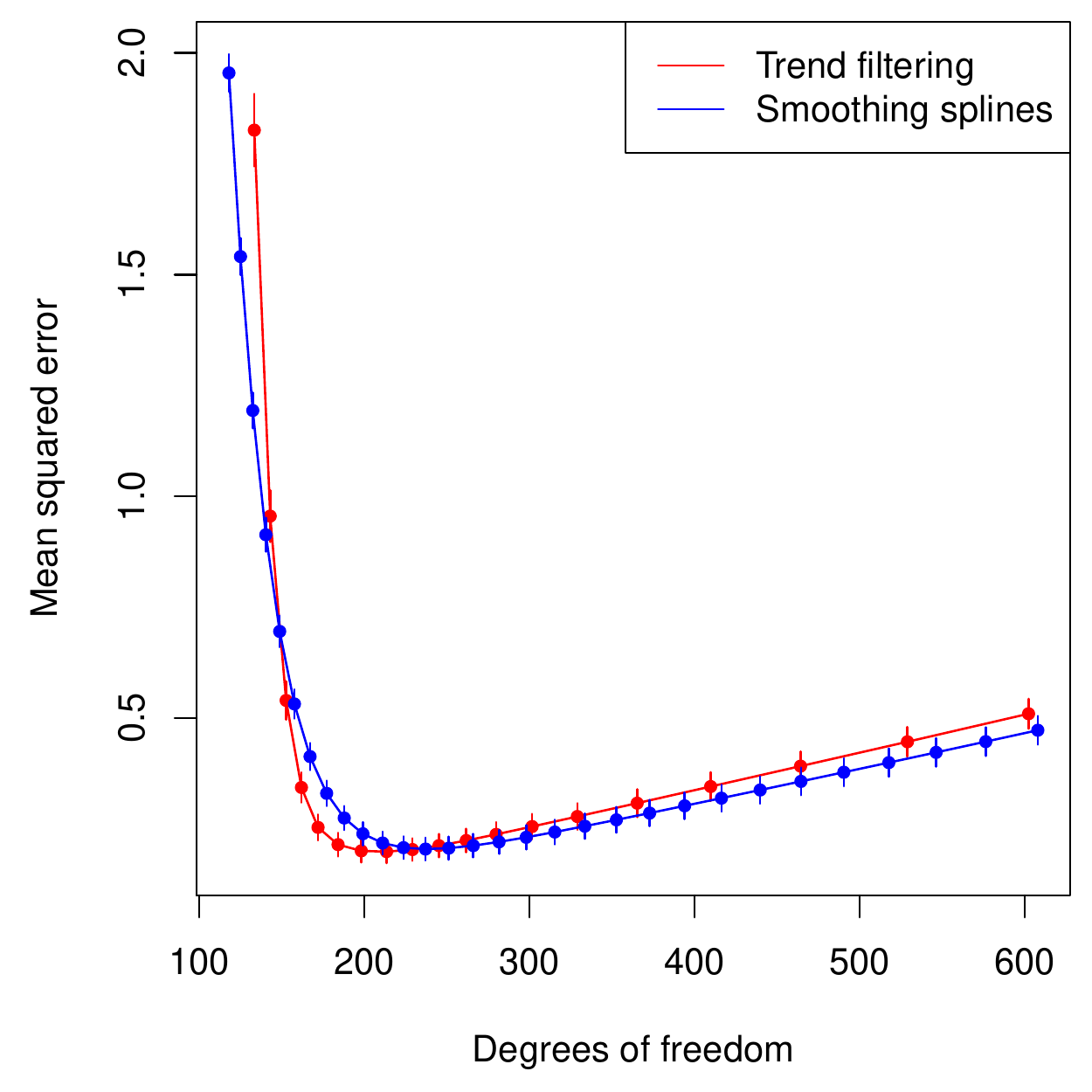}
\includegraphics[width=0.45\textwidth]{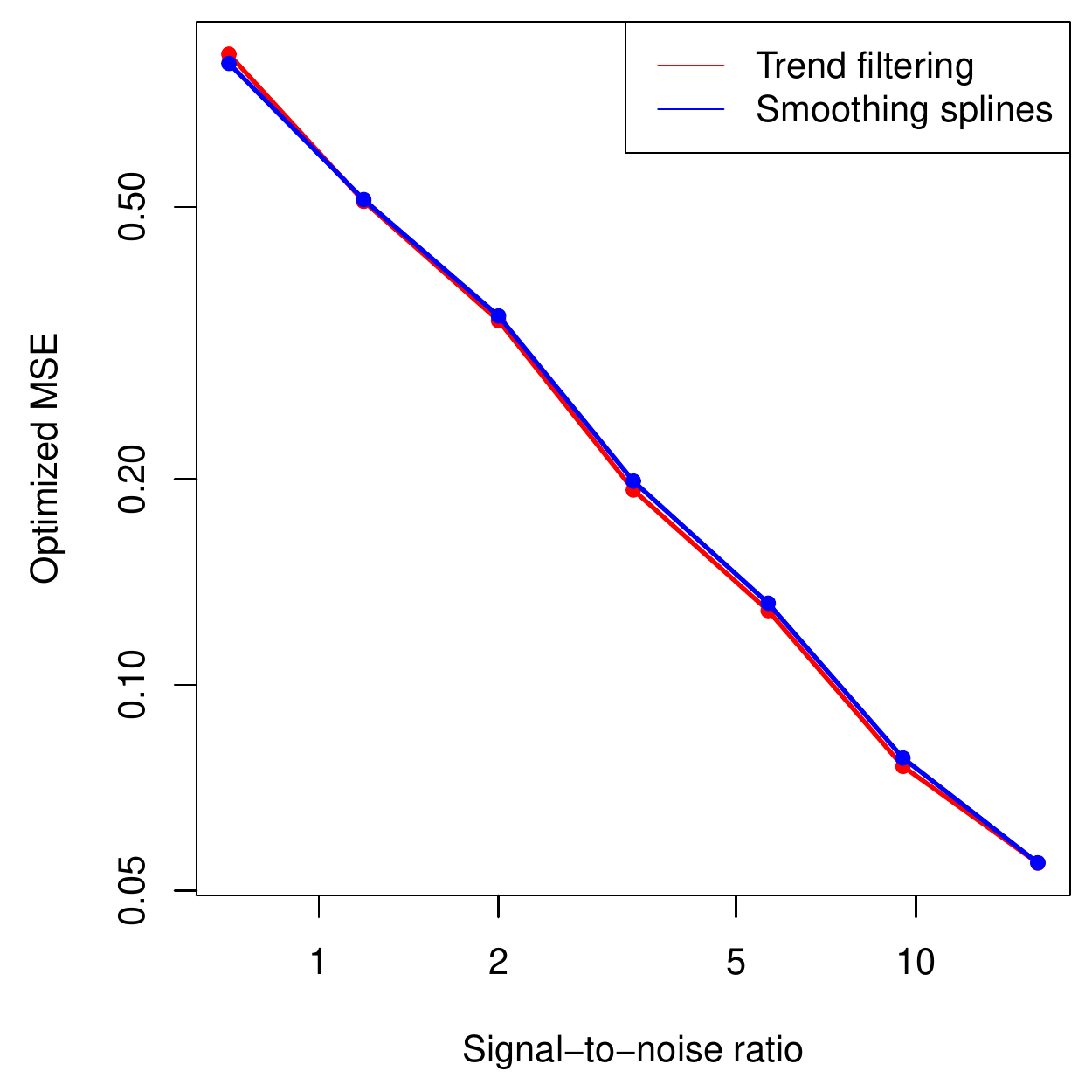}
\caption{\it\small Results from a simulation setup identical to that described 
  in Section \ref{sec:experiment_large_sim}, i.e., identical to that used to
  produce Figure \ref{fig:large_sim}, except with homogeneous smoothness in the 
  underlying component functions.}  
\label{fig:large_sim_homo}
\end{figure}

\newpage
\bibliographystyle{plainnat}
\bibliography{ryantibs}
\end{document}